\documentclass{article}

\usepackage[preprint,nonatbib]{neurips_2020}

\usepackage[utf8]{inputenc} % allow utf-8 input
\usepackage[T1]{fontenc}    % use 8-bit T1 fonts
\usepackage{url}            % simple URL typesetting
\usepackage{booktabs}       % professional-quality tables
\usepackage{amsfonts}       % blackboard math symbols
\usepackage{nicefrac}       % compact symbols for 1/2, etc.
\usepackage{microtype}      % microtypography

\usepackage{color}
\usepackage{tikz}
\usepackage{tikz-3dplot}
\usepackage{algorithm,algorithmic}
\usepackage[algo2e,ruled,vlined]{algorithm2e}
\usepackage{amssymb}
\usepackage{graphicx}
\usepackage{url}
\usepackage{setspace}
\usepackage{subfigure}
\usepackage{wrapfig}
\usepackage{enumitem}
\usepackage{hyperref}
\hypersetup{bookmarksnumbered,bookmarksopen,
colorlinks,citecolor=blue,linkcolor=blue,urlcolor=blue}

\usepackage{mathrsfs}
\usepackage{amsmath,amsfonts,amsthm}
\usepackage{bm}
\usepackage{multicol}
\usepackage{multirow}
\usepackage{tcolorbox}
\usepackage[page]{appendix}

\usepackage[normalem]{ulem}
\newcommand\redout{\bgroup\markoverwith
{\textcolor{red}{\rule[.5ex]{2pt}{0.4pt}}}\ULon}

\newcommand\blfootnote[1]{%
	\begingroup
	\renewcommand\thefootnote{}\footnote{#1}%
	\addtocounter{footnote}{-1}%
	\endgroup
}

\newcommand{\Appendix}[1]{the full version for}

\newtheorem{theorem}{Theorem}[section]
\newtheorem{lemma}[theorem]{Lemma}

\renewcommand{\u}{\bm{u}}
\renewcommand{\v}{\bm{v}}

\newcommand{\x}{\bm{x}}
\newcommand{\y}{\bm{y}}
\newcommand{\z}{\bm{z}}

\newcommand{\bfSigma}{\boldsymbol{\Sigma}}
\newcommand{\bfsigma}{\boldsymbol{\sigma}}

\newcommand{\I}{\bm{I}}

\newcommand{\Q}{\bm{Q}}
\newcommand{\R}{\mathbb{R}}
\renewcommand{\Re}{\mathbb{R}}
\renewcommand{\S}{\bm{S}}

\newcommand{\U}{\bm{U}}
\newcommand{\V}{\bm{V}}

\newcommand{\X}{\bm{X}}
\newcommand{\Y}{\bm{Y}}
\newcommand{\Z}{\bm{Z}}
\newcommand{\rank}{\textup{\textsf{rank}}}

\newcommand{\blambda}{\boldsymbol{\lambda}}

\newcommand{\0}{\mathbf{0}}

\newcommand{\tr}{\textup{\textsf{tr}}}

\newcommand{\cL}{\mathcal{L}}

\newcommand{\cT}{\mathcal{T}}

\newcommand{\bbS}{\mathbb{S}}

\DeclareMathOperator*{\argmax}{arg\,max}
\DeclareMathOperator*{\argmin}{arg\,min}
\usepackage{pifont}
\newcommand{\cmark}{\ding{51}}%
\newcommand{\xmark}{\ding{55}}%

\title{Learning Diverse and Discriminative Representations via the Principle of Maximal Coding Rate Reduction}

\author{
Yaodong Yu$^\dagger$ \;\;
Kwan Ho Ryan Chan$^\dagger$ \;\;
Chong You$^\dagger$ \;\;
Chaobing Song$^\ddagger$ \;\;
Yi Ma$^\dagger$
\\
\vspace*{-0.05in} 
\\
$^\dagger$Department of EECS, University of California, Berkeley\\
$^\ddagger$Tsinghua-Berkeley Shenzhen Institute, Tsinghua University\\
}

\begin{document}

\maketitle
\vspace*{-0.35in} 
\maketitle
\blfootnote{\hspace*{-1.4mm}$^*$The first two authors contributed equally to this work.}

\begin{abstract}
{To learn intrinsic low-dimensional structures from high-dimensional data that most discriminate between classes, we propose the principle of {\em Maximal Coding Rate Reduction} ($\text{MCR}^2$), an information-theoretic measure that maximizes the coding rate difference between the whole dataset and the sum of each individual class. We clarify its relationships with most existing frameworks such as cross-entropy, information bottleneck, information gain, contractive and contrastive learning, and provide theoretical guarantees for learning diverse and discriminative features. The coding rate can be accurately computed from finite samples of degenerate subspace-like distributions and can learn intrinsic representations in supervised, self-supervised, and unsupervised settings in a unified manner. Empirically, the representations learned using this principle alone are significantly more robust to label corruptions in classification than those using cross-entropy, and can lead to state-of-the-art results in clustering mixed data from self-learned invariant features.}
\end{abstract}

% \vskip -0.15in
\section{Context and Motivation}\label{sec:motivation}
% \vspace{-3mm}
Given a random vector $\x \in \Re^D$ which is drawn from a mixture of, say $k$, distributions $\mathcal{D} = \{\mathcal{D}_j\}_{j=1}^k$, one of the most fundamental problems in machine learning is how to effectively and efficiently {\em learn the distribution} from a finite set of i.i.d samples, say $\X = [\x_1, \x_2, \ldots, \x_m] \in \Re^{D\times m}$. To this end, we {\em seek a good representation}  through a continuous mapping, $f(\x, \theta): \Re^D \rightarrow \Re^d$, that captures intrinsic structures of $\x$ and best facilitates subsequent tasks such as classification or clustering.

\textbf{Supervised learning of discriminative representations.} To ease the task of learning $\mathcal{D}$, in the popular supervised setting, a true class label, represented as a one-hot vector $\y_i \in \Re^k$, is given for each sample $\x_i$. Extensive studies have shown that for many practical datasets (images, audios, and natural languages, etc.), the mapping from the data $\bm x$ to its class label $\bm y$ can be effectively modeled by training a deep network~\cite{goodfellow2016deep}, here denoted as $f(\x, \theta):\x \mapsto \y$ with network parameters $\theta \in \Theta$. This is typically done by minimizing the {\em cross-entropy loss} over a training set $\{(\x_i, \y_i)\}_{i=1}^m$, 
through backpropagation over the network parameters $\theta$: 
\begin{equation}
   \min_{\theta \in \Theta} \; \mbox{CE}(\theta, \x, \y) \doteq - \mathbb{E}[\langle \y, \log[f(\x, \theta)] \rangle] \, \approx - \frac{1}{m}\sum_{i=1}^m \langle \y_i, \log[f(\x_i, \theta)] \rangle.
   \label{eqn:cross-entropy}
\end{equation}
Despite its effectiveness and enormous popularity, there are two serious limitations with this approach:  1) It aims only to predict the labels $\y$ even if they might be mislabeled. Empirical studies show that deep networks, used as a ``black box,'' can even fit random labels \cite{zhang2017understanding}. 2) With such an end-to-end data fitting, it is not clear to what extent the intermediate features learned by the network capture the intrinsic structures of the data that make meaningful classification possible in the first place.\footnote{despite plenty of empirical efforts in trying to illustrate or interpreting the so-learned features \cite{Zeiler-ECCV2014}.} The precise geometric and statistical properties of the learned features are also often obscured, which leads to the lack of interpretability and subsequent performance guarantees (e.g., generalizability, transferability, and robustness, etc.) in deep learning. Therefore, the goal of this paper is to address such limitations of current learning frameworks by reformulating the objective towards learning {\em explicitly meaningful} representations for the data $\x$.

\textbf{Minimal discriminative features via information bottleneck.} One popular approach to interpret the role of deep networks is to view outputs of intermediate layers of the network as selecting certain latent features $\z = f(\x, \theta) \in \Re^d$ of the data that are discriminative among multiple classes. Learned representations $\z$ then  facilitate the subsequent classification task for predicting the class label $\y$ by optimizing a classifier $g(\z)$: 
\begin{equation*}
    \x   \xrightarrow{\hspace{2mm} f(\x, \theta)\hspace{2mm}} \z(\theta)  \xrightarrow{\hspace{2mm} g(\z) \hspace{2mm}} \y.
\label{eqn:discriminative}
\end{equation*}
The {\em information bottleneck} (IB) formulation \cite{Tishby-ITW2015} further hypothesizes that the role of the network is to learn $\z$ as the minimal sufficient statistics for predicting $\y$. Formally, it seeks to maximize the mutual information $I(\z, \y)$\footnote{Mutual information is defined to be $I(\z, \y) \doteq H(\bm z) - H(\bm z \mid \y)$ where $H(\z)$ is the entropy of $\z$ \cite{Thomas-Cover}.} between $\z$ and $\y$ while minimizing $I(\x, \z)$ between $\x$ and $\z$:
\begin{equation}
    \max_{\theta\in \Theta}\; \mbox{IB}(\x, \y, \z(\theta)) \doteq I(\z(\theta), \y) - \beta I(\x, \z(\theta)), \quad \beta >0. 
\label{eqn:information-bottleneck}
\end{equation}
This framework has been successful in describing certain behaviors of deep networks.\footnote{given one can overcome some caveats associated with this framework \cite{kolchinsky2018caveats-ICLR2018} and practical difficulties such as how to accurately evaluate mutual information with finitely samples of degenerate distributions.} But by being task-dependent (depending on the label $\y$) and seeking a {\em minimal} set of most informative features for the task at hand (for predicting the label $\y$ only), the network sacrifices generalizability, robustness, or transferability.\footnote{in case the labels can be corrupted or the learned features be tackled.} To address this, our framework uses label $\y$ only as side information to assist learning discriminative features, hence making learned features more robust to mislabeled data.

\textbf{Contractive learning of generative representations.} 
{Complementary to the above supervised discriminative approach, {\em  auto-encoding} \cite{Baldi89,Kramer1991NonlinearPC} is another popular {\em unsupervised} (label-free) framework used to learn good latent representations.} The idea is to learn a compact latent representation $\z \in \Re^d$ that adequately regenerates  the original data $\x$ to certain extent, say through optimizing some decoder or generator $g(\z, \eta)$\footnote{hence the auto-encoding \cite{Baldi89,Kramer1991NonlinearPC} can be viewed as a nonlinear extension to the classical PCA \cite{Jolliffe2002}.}:
\begin{equation}
     \x \xrightarrow{\hspace{2mm} f(\x, \theta)\hspace{2mm}} \z(\theta)  \xrightarrow{\hspace{2mm} g(\z,\eta) \hspace{2mm}} \widehat{\x}(\theta, \eta).
     \label{eqn:generative}
\end{equation}
Typically, such representations are learned in an end-to-end fashion by imposing certain heuristics on  geometric or statistical ``compactness'' of $\z$, such as its dimension, energy, or volume. For example, the {\em contractive} autoencoder  \cite{contractive-ICML11} penalizes local volume expansion of learned features approximated by the Jacobian $\|\frac{\partial \z}{\partial \theta}\|$. Another key design factor of this approach is the choice of a proper, but often elusive, metric that can measure the desired {\em similarity} between $\x$ and the decoded $\widehat{\x}$, either between sample pairs $\x_i$ and  $\widehat{\x}_i$\footnote{for tasks such as denoising, in which the metric can be chosen to the $\ell^p$-norm between samples of $\x$ and $\hat{\x}$: $\min_{\theta,\eta} \mathbb{E}[\|\x -\widehat{ \x}\|_p]$, where typically $p = 1$ or 2, for tasks such as image denoising.} or
between the two distributions $\mathcal{D}_{\x}$ and $\mathcal{D}_{\widehat \x}$.\footnote{the distance between distributions of $\x$ and $\widehat{\x}$, say the KL divergence $\mbox{KL}(\mathcal{D}_{\x}|| \mathcal{D}_{\widehat{\x}})$, is very difficult to evaluate when the data distributions are discrete and degenerate. In practice, it can only be approximated with the help of an additional disriminative network, known as GAN~\cite{goodfellow2014generative,arjovsky2017wasserstein}.}

Representations learned through this framework can be arguably rich enough to regenerate the data to a certain extent. But depending on the choice of the regularizing heuristics on $\z$ and similarity metrics on $\x$ (or $\mathcal{D}_{\x}$), the objective is typically task-dependent and often grossly approximated \cite{contractive-ICML11,goodfellow2014generative}. When the data contain complicated {\em multi-modal} structures, naive heuristics or inaccurate metrics may fail to capture all internal subclass structures\footnote{One consequence of this is the phenomenon of {\em mode collapsing} in learning generative models for data that have mixed multi-modal structures; see \cite{li2020multimodal-IJCV} and references therein.} or to explicitly discriminate among them for classification or clustering purposes. To address this, we propose a principled measure (on $\z$) to learn representations that promotes multi-class discriminative property from data of mixed structures, which works in both supervised and unsupervised settings.

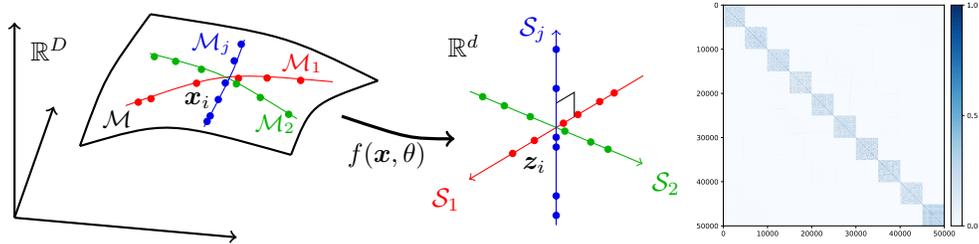
\begin{figure*}[t]
\begin{center}
    \subfigure{
    \tdplotsetmaincoords{60}{110}
\begin{tikzpicture}[scale=1.5]
  \coordinate (o1) at (0,0);
  \coordinate (o2) at (4.8,0.8);
  \coordinate (o3) at (6.1,1.2);
  
  % figure transitions
  \draw[very thick,->] (2.9,0.9) .. controls (3.5,0.7) .. (3.9,0.7);
  \draw (3.3,0.75) node[anchor=north]{$f(\x,{\theta})$};
  % left figure
  \tdplotsetrotatedcoords{100}{0}{0}
  \tdplotsetrotatedcoordsorigin{(o1)}
  \begin{scope}[tdplot_rotated_coords]
  \draw[thick,->] (0,0,0) -- (2,0,0);%  node[anchor=north east]{$x$};
  \draw[thick,->] (0,0,0) -- (0,2.0,0);%  node[anchor=north west]{$y$};
  \draw[thick,->] (0,0,0) -- (0,0,2);%  node[anchor=south]{$z$}; 
  \draw (0.3,0.1,1.7) node{$\Re^D$};
  \draw (3.9,0.8,1.7) node{$\Re^d$};

  % data manifold
  \draw[thick] (0.5,0.5,0.5) .. controls (1.2,1,0.55) .. (2.4,0.5,0.6);
  \draw[thick] (2.4,0.5,0.6) .. controls (2.4,1.5,0.55) .. (2.9,2.1,0.5);
  \draw[thick] (2.9,2.1,0.5) .. controls (2.0,2.5,0.65) .. (0.8,2.4,0.8);
  \draw[thick] (0.8,2.4,0.8) .. controls (0.55,1.5,0.65) .. (0.5,0.5,0.5);
  \draw (0.6,0.6,0.55) node[anchor=south west]{\small $\mathcal{M}$};
  
  % data subspaces
  \draw[red] (0.85,0.85,0.75) .. controls (1.6,1.5,0.8) .. (2.6,1.5,0.8);
  \draw[black!30!green] (0.9,1.8,0.75) .. controls (1.6,1.5,0.8) .. (2.4,0.8,0.8);
  \draw[black!10!blue] (1.6,0.6,0.75) .. controls (1.7,1.5,0.8) .. (1.7,2.1,0.8);
  \draw (2.2,2.2,0.55) node[red]{\small{$ \mathcal{M}_1$}};
  \draw (2.15,1.1,0.55) node[black!30!green]{\small $\mathcal{M}_2$};
  \draw (1.35,2.4,0.55) node[black!10!blue]{\small $\mathcal{M}_j$}; 
  \draw (1.4,1.4,0.55) node{$\x_i$}; 
  
  % data points on subspaces
  \def\points{(0.95,0.93,0.75), (1.05,1.01,0.75), (1.4,1.35,0.75), (1.75,1.5,0.75), (2.0,1.55,0.75), (2.3,1.56,0.75)}
  \foreach \p in \points {
    \draw plot [mark=*, mark size=0.8, mark options={draw=red, fill=red}] coordinates{\p}; 
  }
  \def\points{(0.95,1.75,0.75), (1.15,1.7,0.75), (1.4,1.6,0.75), (1.75,1.4,0.75), (2.0,1.2,0.75), (2.3,0.95,0.75)}
  \foreach \p in \points {
    \draw plot [mark=*, mark size=0.8, mark options={draw=black!30!green, fill=black!30!green}] coordinates{\p};
  }
  \def\points{(1.61,0.7,0.75), (1.62,0.8,0.75), (1.64,1.1,0.75), (1.65,1.4,0.75), (1.66,1.8,0.75), (1.67,2.1,0.75)}
  \foreach \p in \points {
    \draw plot [mark=*, mark size=0.8, mark options={draw=black!10!blue, fill=black!10!blue}] coordinates{\p};
  }
  \end{scope}
  
  % bottom right figure
  \tdplotsetrotatedcoords{-30}{0}{10}
  \tdplotsetrotatedcoordsorigin{(o2)}
  \begin{scope}[tdplot_rotated_coords]
    \draw [red, ->] (-1.2,0,0) -- (1.2,0,0) node[anchor=north east]{$\mathcal{S}_1$};
    \draw [black!30!green, ->] (0,-1,0) -- (0,1,0) node[anchor=north west]{$\mathcal{S}_2$};
    \draw [black!10!blue, ->] (0,0,-1) -- (0,0,1) node[anchor=east]{$\mathcal{S}_j$};
    \draw (0,0,0.25) -- (-0.25,0,0.25) -- (-0.25,0,0);
    \def\points{(-0.8,0,0), (-0.6,0,0), (-0.3,0,0), (-0.1,0,0), (0.3,0,0), (0.6,0,0)}
    \foreach \p in \points {
      \draw plot [mark=*, mark size=0.8, mark options={draw=red, fill=red}] coordinates{\p}; 
    }
    \def\points{(0,-0.85,0), (0,-0.6,0), (0,-0.3,0), (0,0.1,0), (0,0.4,0), (0,0.6,0)}
    \foreach \p in \points {
      \draw plot [mark=*, mark size=0.8, mark options={draw=black!30!green, fill=black!30!green}] coordinates{\p};
    }
    \def\points{(0,0,-0.9), (0,0,-0.7), (0,0,-0.2), (0,0,-0.1), (0,0,0.4), (0,0,0.8)}
    \foreach \p in \points {
    \draw plot [mark=*, mark size=0.8, mark options={draw=black!10!blue, fill=black!10!blue}] coordinates{\p};
  }
  \draw (0.3,0.0,-0.25) node{$\z_i$}; 
  \end{scope}
\end{tikzpicture}}
    % \hspace{5mm} 
    \subfigure{\includegraphics[width=0.27\textwidth]{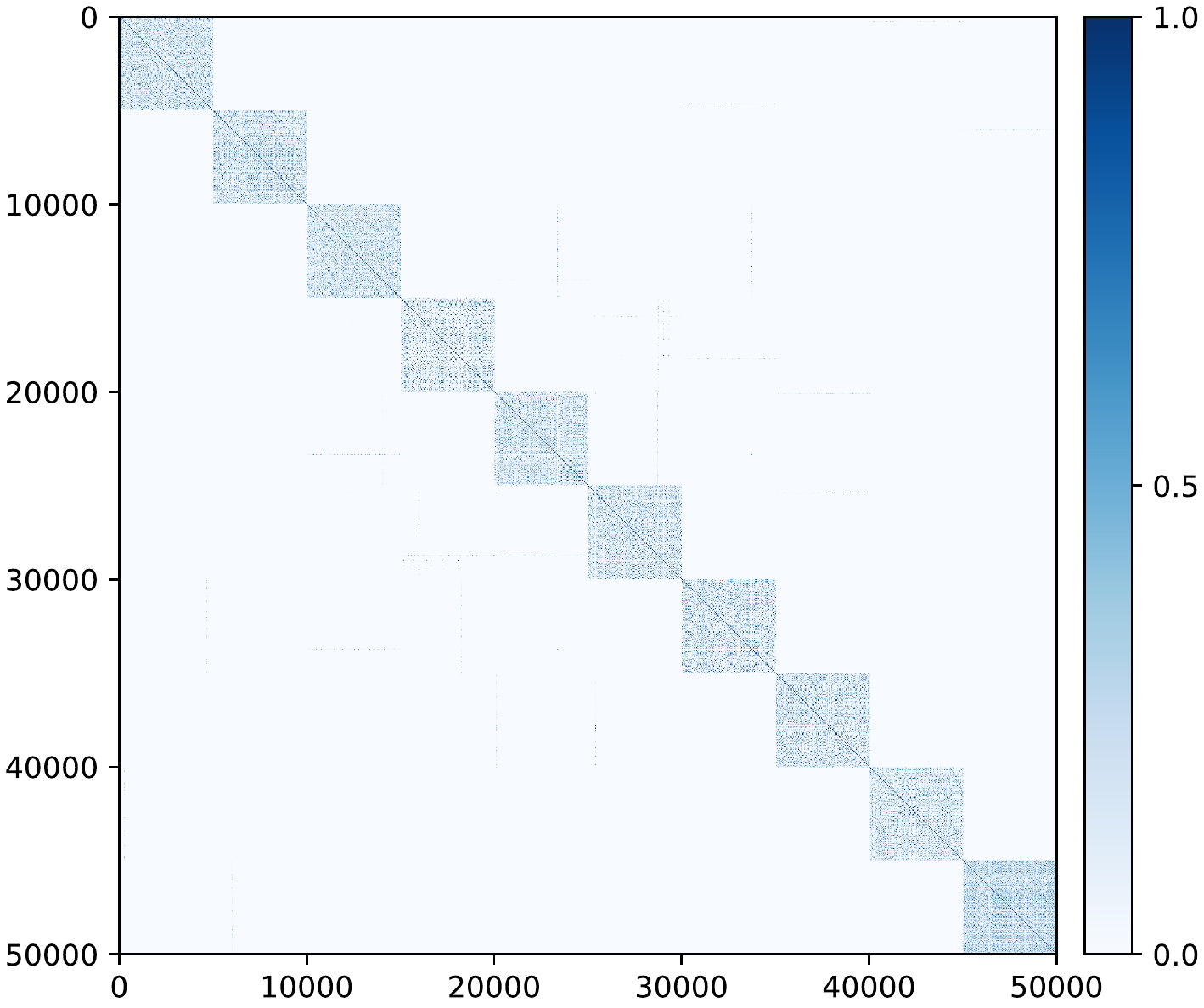}}
      \vskip -0.05in
\caption{\small \textbf{Left and Middle:} The distribution $\mathcal D$ of high-dim data $\x\in \Re^D$ is supported on  a manifold $\mathcal{M}$ and its classes on low-dim submanifolds $\mathcal{M}_j$, we learn a map $f(\x, \theta)$ such that $\z_i = f(\x_i, \theta)$ are on a union of maximally uncorrelated subspaces $\{\mathcal{S}_j\}$. \textbf{Right:} Cosine similarity between learned features by our method for the CIFAR10 training dataset. Each class has 5,000 samples and their features span a subspace of over 10 dimensions (see Figure~\ref{fig:train-test-loss-pca-3}).}
\label{fig:low-dim}
\end{center}
\vskip -0.25in
\end{figure*}

{\textbf{This work: Learning diverse and discriminative representations.} Whether the given data $\X$ of a mixed distribution $\mathcal{D}$ can be effectively classified depends on how separable (or discriminative) the component distributions $\mathcal{D}_j$ are (or can be made).}  
One popular working assumption is that the distribution of each class has relatively {\em low-dimensional} intrinsic structures.\footnote{There are many reasons why this assumption is plausible: 1. high dimensional data are highly redundant; 2. data that belong to the same class should be similar and correlated to each other; 3. typically we only care about equivalent structures of $\x$ that are invariant to certain classes of deformation and augmentations.} Hence we may assume the distribution $\mathcal{D}_j$ of each class has a support on a low-dimensional submanifold, say $\mathcal{M}_j$ with dimension $d_j \ll D$, and the distribution $\mathcal D$ of $\x$ is supported on the mixture of those submanifolds, $\mathcal M = \cup_{j=1}^k \mathcal{M}_j$,  in the high-dimensional ambient space $\Re^D$, as illustrated in Figure~\ref{fig:low-dim} left. 

%\cs{To be consistent with the following pages, use Figure~\ref{fig:low-dim} to replace Fig.~\ref{fig:low-dim}. }

With the manifold assumption in mind, we want to learn a mapping {$\z = f(\x, \theta)$} that maps each of the submanifolds $\mathcal{M}_j \subset \Re^D$  to a {\em linear} subspace $\mathcal{S}_j \subset \Re^d$ (see Figure~\ref{fig:low-dim} middle). To do so, we require our learned representation to have the following properties: 
\begin{enumerate}
    \item {\em Between-Class Discriminative:} Features of samples from different classes/clusters should be highly {\em uncorrelated} and belong to different low-dimensional linear subspaces.
    \item {\em Within-Class Compressible:} Features of samples from the same class/cluster should be relatively {\em correlated} in a sense that they belong to a low-dimensional linear subspace.
    \item {\em Maximally Diverse Representation:} Dimension (or variance) of features for each class/cluster should be {\em as large as possible} as long as they stay uncorrelated from the other classes.
\end{enumerate}
Notice that, although the intrinsic structures of each class/cluster may be low-dimensional, they are by no means simply linear in their original representation $\x$. Here the subspaces $\{\mathcal{S}_j\}$ can be viewed as nonlinear {\em generalized principal components} for $\bm x$ \cite{GPCA}.  Furthermore, for many clustering or classification tasks (such as object recognition), we consider two samples as {\em equivalent} if they differ by certain class of domain deformations or augmentations $\cT = \{\tau \}$. Hence, we are only interested in low-dimensional structures that are {\em invariant} to such deformations,\footnote{So $\x \in \mathcal{M}$ iff $\tau(\x) \in \mathcal{M}$ for all $\tau \in \cT$.} which are known to have sophisticated geometric and topological structures \cite{Wakin-2005} and can be difficult to learn in a principled manner even with CNNs \cite{Cohen-ICML-2016, cohen2019general}. There are previous attempts to directly enforce subspace structures on features learned by a deep network for supervised \cite{lezama2018ole} or unsupervised learning \cite{Ji-NIPS2017,zhang2018scalable,peng2017deep,zhou2018deep,zhang2019neural,zhang2019self,lezama2018ole}. However, the {\em self-expressive} property of subspaces exploited by \cite{Ji-NIPS2017} does not enforce all the desired properties listed above; \cite{lezama2018ole} uses a nuclear norm based geometric loss to enforce orthogonality between classes, but does not promote diversity in the learned representations, as we will soon see. Figure~\ref{fig:low-dim} right
illustrates a representation learned by our method on the CIFAR10 dataset. More details can be found in the experimental Section \ref{sec:experiments}.

% \vspace{-2mm}
\section{Technical Approach and Method}
% \vskip -0.1in
\subsection{Measure of Compactness for a Representation}\label{sec:lossy-coding}
Although the above  properties are all highly desirable for the latent representation $\z$, they are by no means easy to obtain: Are these properties compatible so that we can expect to achieve them all at once? If so, is there  a {\em simple but principled} objective that can measure the goodness of the resulting representations in terms of all these properties? The key to  these questions {is to find} a principled ``measure of compactness'' for the distribution of a random variable $\z$ or from its finite samples $\Z$. Such a measure should directly and accurately characterize intrinsic geometric or statistical properties of the distribution, in terms of its intrinsic dimension or {volume}. Unlike cross-entropy \eqref{eqn:cross-entropy} or information bottleneck \eqref{eqn:information-bottleneck}, such a measure should not depend explicitly on class labels so that it can work in all supervised, self-supervised, semi-supervised, and unsupervised settings.

\textbf{Low-dimensional degenerate distributions.} In information theory \cite{Thomas-Cover}, the notion of entropy $H(\z)$ is designed to be such a measure.\footnote{given the probability density $p(\z)$ of a random variable, $H(\z) \doteq - \int p(\z) \log p(\z) \, d\z.$} However, entropy is not  well-defined for continuous random variables with degenerate distributions.\footnote{The same difficulty resides with evaluating mutual information $I(\x, \z)$ for degenerate distributions.} This is unfortunately the case here. To alleviate this difficulty, another related concept in information theory, more specifically in lossy data compression, that measures the ``compactness'' of a random distribution is the so-called {\em rate distortion} \cite{Thomas-Cover}: {Given a random variable $\z$ and a prescribed precision $\epsilon >0$, the rate distortion $R(\z, \epsilon)$ is the minimal number of binary bits needed to encode $\z$ such that the expected decoding error\footnote{Say in terms of the $\ell^2$-norm, we have $\mathbb E[\|\z - \widehat \z \|_2] \le \epsilon$ for the decoded $\widehat \z$.} is less than $\epsilon$. Although this framework has been successful in explaining feature selection in deep networks \cite{rate-distortion}, the rate distortion of a random variable is difficult, if not impossible to compute, except for simple distributions such as discrete and Gaussian. }

\textbf{Nonasymptotic rate distortion for finite samples.} When evaluating the lossy coding rate $R$, one practical difficulty is that we normally do not know the distribution of $\z$.  Instead, we have a finite number of samples as learned representations where $\z_{i} = f(\x_i, \theta) \in \R^{d}, i = 1,\ldots, m$, for the given data samples $\X = [\x_1, \ldots, \x_m]$. Fortunately, \cite{ma2007segmentation} provides a precise estimate on the number of binary bits needed to encoded finite samples from a subspace-like distribution. In order to encode the learned representation $\Z = [\z_1, \dots, \z_m]$ up to a precision $\epsilon$, the total number of bits needed is given by the following expression\footnote{This formula can be derived either by packing $\epsilon$-balls into the space spanned by $\Z$ or by computing the number of bits needed to quantize the SVD of $\Z$ subject to the precision, see \cite{ma2007segmentation} for proofs.}:  $\cL(\Z, \epsilon) \doteq \left(\frac{m + d}{2}\right)\log \det\left(\I + \frac{d}{m\epsilon^{2}}\Z\Z^{\top}\right)$. Therefore, the compactness of learned features {\em as a whole} can be measured in terms of the average coding length per sample (as the sample size $m$ is large),  a.k.a. the {\em coding rate} subject to the distortion $\epsilon$:
\begin{equation}
R(\Z,\epsilon) \doteq \frac{1}{2}\log\det\left(\I + \frac{d}{m\epsilon^{2}}\Z\Z^{\top}\right).
\label{eqn:coding-length-eval}
\end{equation}

\textbf{Rate distortion of data with a mixed distribution.} {In general,}  the features $\Z$ of multi-class data may belong to multiple low-dimensional subspaces. To evaluate the rate distortion of such mixed data {\em more accurately}, we may partition the data $\Z$ into multiple subsets: $\Z = \Z_1 \cup \cdots \cup \Z_k$, {with} each in one low-dim subspace. So the above coding rate \eqref{eqn:coding-length-eval} is accurate for each subset. For convenience, let $\bm{\Pi} = \{\bm{\Pi}_j \in \Re^{m \times m}\}_{j=1}^{k}$ be a set of diagonal matrices whose diagonal entries encode the membership of the $m$ samples in the $k$ classes.\footnote{That is, the diagonal entry $\bm \Pi_j(i,i)$ of $\bm \Pi_j$ indicates the probability of sample $i$ belonging to subset $j$. Therefore $\bm{\Pi}$ lies in a simplex: ${\Omega} \doteq \{\bm{\Pi} \mid \bm{\Pi}_j \ge \mathbf{0}, \; \bm{\Pi}_1 + \cdots + \bm{\Pi}_k = \I\}.$} Then, according to \cite{ma2007segmentation}, with respect to this partition, the average number of bits per sample (the coding rate) is
\begin{equation}
R^c(\Z,  \epsilon \mid \bm{\Pi}) \doteq \sum_{j=1}^{k}
\frac{\tr(\bm{\Pi}_j)}{2m}\log\det\left(\I + \frac{d}{\tr(\bm{\Pi}_j)\epsilon^{2}}\Z\bm{\Pi}_j\Z^{\top}\right).
\label{eqn:compress-loss-eval}
\end{equation}
Notice that when $\Z$ is given, $R^c(\Z, \epsilon \mid \bm{\Pi})$ is a concave function of $\bm{\Pi}$.
The function $\log\det(\cdot)$ in the above expressions has been long known as an effective heuristic for rank minimization problems, with guaranteed convergence to local minimum \cite{fazel2003log-det}. {As it nicely characterizes the rate distortion of Gaussian or subspace-like distributions, $\log\det(\cdot)$ can be very effective in clustering or classification of mixed data \cite{ma2007segmentation,wright2008classification,kang2015logdet}.} We will soon reveal more desired properties of this function.

\subsection{Principle of Maximal Coding Rate Reduction}\label{sec:principle-mcr2}
On one hand, for learned features to be discriminative, features of different classes/clusters are preferred to be {\em maximally incoherent} to each other. Hence they together should span a space of the largest possible volume (or dimension) and the coding rate of the whole set $\Z$ should be as large as possible. On the other hand, learned features of the same class/cluster should be highly correlated and coherent. Hence, each class/cluster should only span a space (or subspace) of a very small volume and the coding rate should be as small as possible. Therefore, a good representation $\Z$ of $\X$ is one such that, given a partition $\bm{\Pi}$ of $\Z$, achieves a large difference between the coding rate for the whole and that for all the subsets:
\begin{equation}
\Delta R(\Z, \bm{\Pi}, \epsilon) \doteq R(\Z, \epsilon) - R^c(\Z, \epsilon \mid  \bm{\Pi}).
\label{eqn:coding-length-reduction}
\end{equation}
If we choose our feature mapping $\z = f(\x,\theta)$ to be a deep neural network, the overall process of the feature representation and the resulting rate reduction w.r.t. certain partition $\bm{\Pi}$ can be illustrated by the following diagram:
\begin{equation}
    \X 
    \xrightarrow{\hspace{2mm} f(\x, \theta)\hspace{2mm}} \Z(\theta) \xrightarrow{\hspace{2mm} \bm{\Pi},\epsilon \hspace{2mm}} \Delta R(\Z(\theta), \bm{\Pi}, \epsilon).
    \label{eqn:flow}
\end{equation}

Note that $\Delta R$ is {\em monotonic} in the scale of the features $\Z$. So to make the amount of reduction comparable between different representations,\footnote{Here different representations can be either representations associated with different network parameters or representations learned after different layers of the same deep network.} we need to {\em normalize the scale} of the learned features, either by imposing the Frobenius norm of each class $\Z_j$ to scale with the number of features in $\Z_j \in \mathbb R^{d \times m_j}$: $\|\Z_j\|_F^2 = m_j$ or by normalizing each feature to be on the unit sphere: $\z_i \in \mathbb{S}^{d-1}$. This formulation offers a natural justification for the need of ``batch normalization'' in the practice of 
training deep neural networks \cite{ioffe2015batch}. An alternative, arguably simpler, way to normalize the scale of learned representations is to ensure that the mapping of each layer of the network is approximately {\em isometric} \cite{ISOnet}.

Once the representations are comparable, our goal becomes to learn a set of features $\Z(\theta) = f(\X, \theta)$ and their partition $\bm \Pi$ (if not given in advance) such that they maximize the reduction between the coding rate of all features and that of the sum of features w.r.t. their classes:
\vspace*{0.05in} 
\begin{equation}
 \max_{\theta, \bm{\Pi}} \;  \Delta R\big(\Z(\theta), \bm{\Pi}, \epsilon\big) = R(\Z(\theta), \epsilon) - R^c(\Z(\theta),  \epsilon \mid \bm{\Pi}), \quad \mbox{s.t.} \ \ \,  \|\Z_j(\theta)\|_F^2 = m_j, \, \bm{\Pi} \in {\Omega}.
\label{eqn:maximal-rate-reduction}
\end{equation}
We refer to this as the principle of {\em maximal coding rate reduction} (MCR$^2$), 
{an embodiment of Aristotle's famous quote:  ``{\em the whole is greater than the sum of the  parts.}''}  Note that for the clustering purpose alone, one may only care about the sign of $\Delta R$ for deciding whether to partition the data or not, which leads to the greedy algorithm in \cite{ma2007segmentation}.\footnote{Strictly speaking, in the context of clustering {\em finite} samples, one needs to use the more precise measure of the coding length mentioned earlier, see \cite{ma2007segmentation} for more details.} Here to seek or learn the best representation, we further desire the whole is {\em maximally} greater than its parts.

\textbf{Relationship to  information gain.} 
The maximal coding rate reduction can be viewed as a generalization to {\em Information Gain} (IG), which aims to maximize the reduction of entropy of a random variable, say $\bm z$, with respect to an observed attribute, say $\bm \pi$:
$
\max_{\bm \pi} \; \mbox{IG}(\bm z, \bm \pi) \doteq H(\bm z) - H(\bm z \mid \bm \pi),
$
i.e., the {\em mutual information} between $\z$ and $\bm \pi$ \cite{Thomas-Cover}. Maximal information gain has been widely used in areas such as decision trees \cite{decision-trees}. However, MCR$^2$ is used differently in several ways: 1) One typical setting of MCR$^2$ is when the data class labels are given, i.e. $\bm \Pi$ is known, MCR$^2$ focuses on learning representations $\bm z(\theta)$ rather than fitting labels. 2) In traditional settings of IG, the number of attributes in $\bm z$ cannot be so large and their values are discrete (typically binary). Here the ``attributes'' $\bm \Pi$  represent the probability of a multi-class partition for all samples and their values can even be continuous. 3) As mentioned before, entropy $H(\bm z)$ or mutual information $I(\bm z, \bm \pi)$ \cite{hjelm2018learning} is not well-defined for degenerate continuous distributions whereas the rate distortion $R(\bm z, \epsilon)$ is and can be accurately and efficiently computed for (mixed)  subspaces, at least.

%\begin{wrapfigure}{r}{0.5\textwidth}
\begin{figure}
\begin{center}
% \vskip -0.3in
\includegraphics[width=0.65\textwidth]{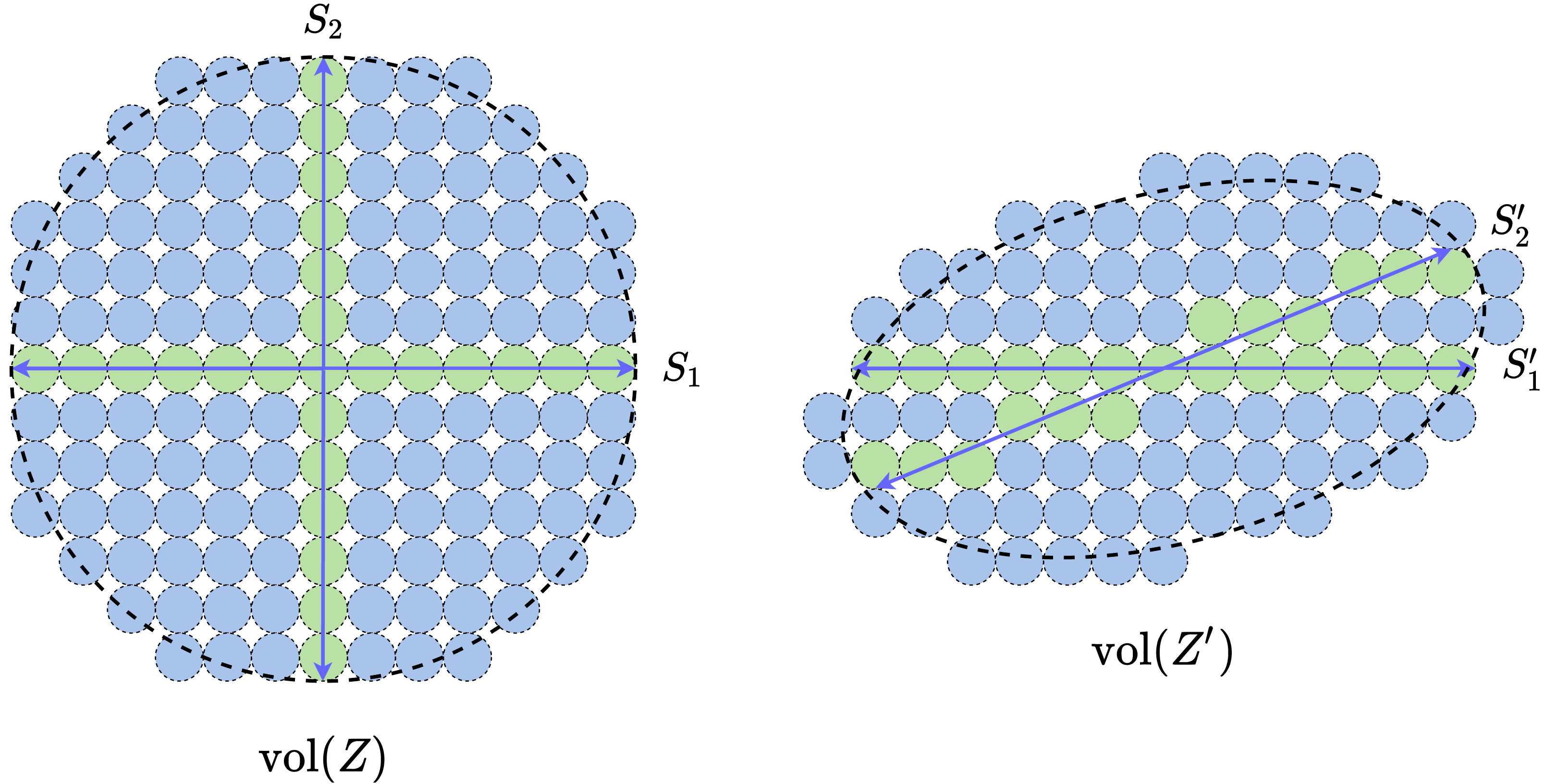}
% \hspace*{12.0mm}
% \includegraphics[width=0.65\textwidth]{Figures/pack.eps}
\end{center}
\vskip -0.05in
\caption{\small Comparison of two learned representations $\Z$ and $\Z'$ via reduced rates: $R$ is the number of $\epsilon$-balls packed in the joint distribution and $R^c$ is the sum of the numbers for all the subspaces (the green balls). $\Delta R$ is their difference (the number of blue balls). The MCR$^2$ principle prefers $\Z$ (the left one).}\label{fig:sphere-packing}
\vskip -0.15in
\end{figure}
%\end{wrapfigure}

\subsection{Properties of the Rate Reduction Function}  
In theory, the MCR$^2$ principle \eqref{eqn:maximal-rate-reduction} benefits from great generalizability and can be applied to representations $\Z$ of {\em any} distributions with {\em any} attributes $\bm \Pi$ as long as the rates $R$ and $R^c$ for the distributions can be accurately and efficiently evaluated. The optimal representation $\Z^*$ and partition $\bm \Pi^*$ should have some interesting geometric and statistical properties. We here reveal nice properties of the optimal representation with the special case of subspaces, which have many important use  cases in machine learning. When the desired representation for $\Z$ is multiple subspaces, the rates $R$ and $R^c$ in \eqref{eqn:maximal-rate-reduction} are given by \eqref{eqn:coding-length-eval}
and \eqref{eqn:compress-loss-eval}, respectively. At the maximal rate reduction, MCR$^2$ achieves its optimal representations, denoted as $\Z^* = \Z_1^*\cup \cdots \cup \Z_k^* \subset \Re^d$ with $\rank{(\Z^*_j)}\le d_j$. One can show that $\Z^*$ has the following desired properties (see Appendix \ref{ap:rate-reduction} for a formal statement and detailed proofs).

\begin{theorem}[Informal Statement]
Suppose $\Z^* = \Z_1^*\cup \cdots \cup \Z_k^*$ is the optimal solution that maximizes the rate reduction~\eqref{eqn:maximal-rate-reduction}. We have:
% \vspace{-2mm}
\begin{itemize}
\item {\em Between-class Discriminative}: As long as the ambient space is adequately large ($d \ge \sum_{j=1}^{k} d_j$), the subspaces are all orthogonal to each other, {\em i.e.} $(\Z_i^*)^{\top} \Z_{j}^* = \bm 0$ for $i \not= j$.
% \vspace{-1mm}
\item {\em Maximally Diverse Representation}: 
As long as the coding precision is adequately high, i.e., $\epsilon ^4 < \min_{j}\left\{ \frac{m_j}{m}\frac{d^2}{d_j^2}\right\}$, each subspace achieves its maximal dimension, i.e. $\rank{(\Z^{*}_{j})}= d_j$. In addition, the largest $d_j - 1$ singular values of $\Z^{*}_{j}$ are equal. 
\label{thm:MCR2-properties}
\end{itemize}
% \vspace{-2mm}
\end{theorem}

In other words, in the case of subspaces, the MCR$^2$ principle promotes embedding of data into multiple independent subspaces, with features distributed {\em isotropically}  in each subspace (except for possibly one dimension). In addition, among all such discriminative representations, it prefers the one with the highest dimensions in the ambient space. This is substantially different from the objective of information bottleneck~\eqref{eqn:information-bottleneck}.

\textbf{Comparison to the geometric OLE loss.} To encourage the learned features to be uncorrelated between classes, the work of \cite{lezama2018ole} has proposed to maximize the difference between the nuclear norm of the whole $\Z$ and its subsets $\Z_j$, called the {\em orthogonal low-rank embedding} (OLE) loss:
$
    \max_{\theta}\,
    \mbox{OLE}(\Z(\theta), \bm \Pi) \doteq  \|\Z(\theta)\|_* - \sum_{j=1}^k \|\Z_j(\theta)\|_*,
$
added as a regularizer to the cross-entropy loss \eqref{eqn:cross-entropy}. The nuclear norm $\|\cdot \|_*$ is a {\em nonsmooth convex}\footnote{Nonsmoothness poses additional difficulties in using this loss to learn features via gradient descent.} surrogate for low-rankness, whereas $\log\det(\cdot)$ is {\em smooth  concave} instead. Unlike the rate reduction $\Delta R$, OLE is always {\em negative} and achieves the maximal value $0$ when the subspaces are orthogonal, regardless of their dimensions. So in contrast to $\Delta R$, this loss serves as a geometric heuristic and does not promote diverse representations.  In fact, OLE typically promotes learning one-dim representations per class, whereas MCR$^2$ encourages learning subspaces with maximal dimensions (Figure~7 of \cite{lezama2018ole} versus our Figure~\ref{fig:pca-plot}).

\textbf{Relation to contrastive learning.}
If samples are {\em evenly} drawn from $k$ classes, a randomly chosen pair $(\x_i, \x_j)$ is of high probability belonging to difference classes if $k$ is large.\footnote{For example, when $k \ge 100$, a random pair is of probability 99\% belonging to different classes.} We may view the learned features of two samples together with their their augmentations $\Z_i$ and $\Z_j$ as two classes. Then the rate reduction $\Delta R_{ij} = R(\Z_i\cup \Z_j, \epsilon) - \frac{1}{2}(R(\Z_i, \epsilon) + R(\Z_j, \epsilon))$ gives a ``distance'' measure for how far the two sample sets are. We may try to further ``expand'' pairs that likely belong to different classes. From Theorem \ref{thm:MCR2-properties},  the (averaged) rate reduction $\Delta R_{ij}$ is maximized  when features from different samples are uncorrelated $\Z_i^\top \Z_j = \bm 0$  (see Figure~\ref{fig:sphere-packing}) and features $\Z_i$ from the same sample are highly correlated. Hence, when applied to sample pairs, MCR$^2$ naturally conducts the so-called {\em contrastive learning} \cite{hadsell2006dimensionality,oord2018representation,he2019momentum}. But  MCR$^2$ is {\em not} limited to expand (or compress) pairs of samples and can uniformly conduct ``contrastive learning'' for a subset with {\em any number} of samples as long as we know they likely belong to different (or the same) classes, say by randomly sampling subsets from a large number of classes or with a good clustering method.

% \vspace{-2mm}
\section{Experiments with Instantiations of MCR$^2$}\label{sec:experiments}
% \vskip -0.05in
Our theoretical analysis above shows how the {\em maximal coding rate reduction} (MCR$^2$) is a principled measure for learning discriminative and diverse representations for mixed data. In this section, we demonstrate experimentally how this principle alone, {\em without any other heuristics,} is adequate to learning good representations in the supervised, self-supervised, and unsupervised learning settings in a unified fashion. Due to limited space and time, instead of trying to exhaust all its potential and practical implications with extensive engineering, our goal here is only to validate effectiveness of this principle through its most basic usage and fair comparison with existing frameworks. More implementation details and experiments are given in Appendix~\ref{ap:additional-exp}. The code can be found in \url{https://github.com/ryanchankh/mcr2}.

\subsection{Supervised Learning of Robust Discriminative Features}\label{sec:supervised-experiments}
\textbf{Supervised learning via rate reduction.} When class labels are provided during training, we assign the membership (diagonal) matrix $\bm{\Pi} = \{\bm{\Pi}_j\}_{j=1}^{k}$ as follows: for each sample $\bm{x}_i$ with label $j$, set $\bm{\Pi}_{j}(i,i) = 1$ and  $\bm{\Pi}_l(i,i)=0, \forall l \not= j$. Then the mapping $f(\cdot, \theta)$ can be learned by optimizing \eqref{eqn:maximal-rate-reduction}, where $\bm{\Pi}$ 
remains constant. We apply stochastic gradient descent to optimize MCR$^2$, and for each iteration we use mini-batch data $\{(\bm{x}_i, \bm{y}_{i})\}_{i=1}^{m}$ to approximate the MCR$^2$ loss.

\textbf{Evaluation via classification.} As we will see, in the supervised setting, the learned representation has very clear subspace structures. So to evaluate the learned representations, we consider a natural nearest subspace classifier. For each class of learned features $\Z_j$, let $\bm{\mu}_j \in \Re^p$ be its mean and $\U_j \in \Re^{p \times r_j}$ be the first $r_j$ principal components for $\Z_j$, where $r_j$ is the estimated  dimension of class $j$. The predicted label of a test data $\x'$ is given by
$
j' = \argmin_{j \in \{1, \ldots, k\}}\|(\I - \U_j \U_j^{\top}) (f(\x', \theta) - \bm{\mu}_j)\|_2^2.
$

\begin{figure*}[t]
\subcapcentertrue
  \begin{center}
    \subfigure[\label{fig:train-test-loss-pca-1} Evolution of $R, R^c, \Delta R$ during the training process.]{\includegraphics[width=0.32\textwidth]{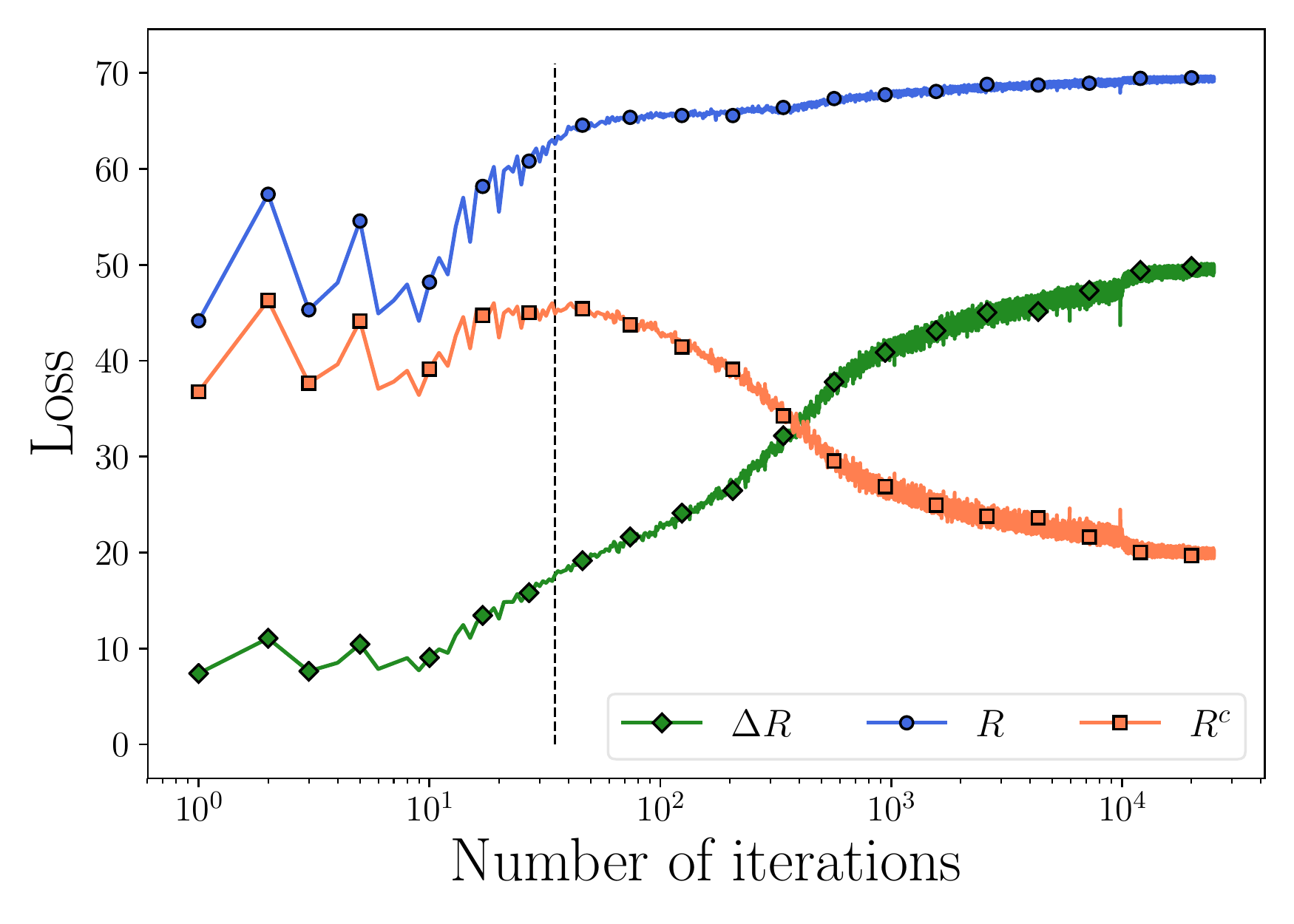}}
    \subfigure[\label{fig:train-test-loss-pca-2}Training loss versus testing loss.]{\includegraphics[width=0.32\textwidth]{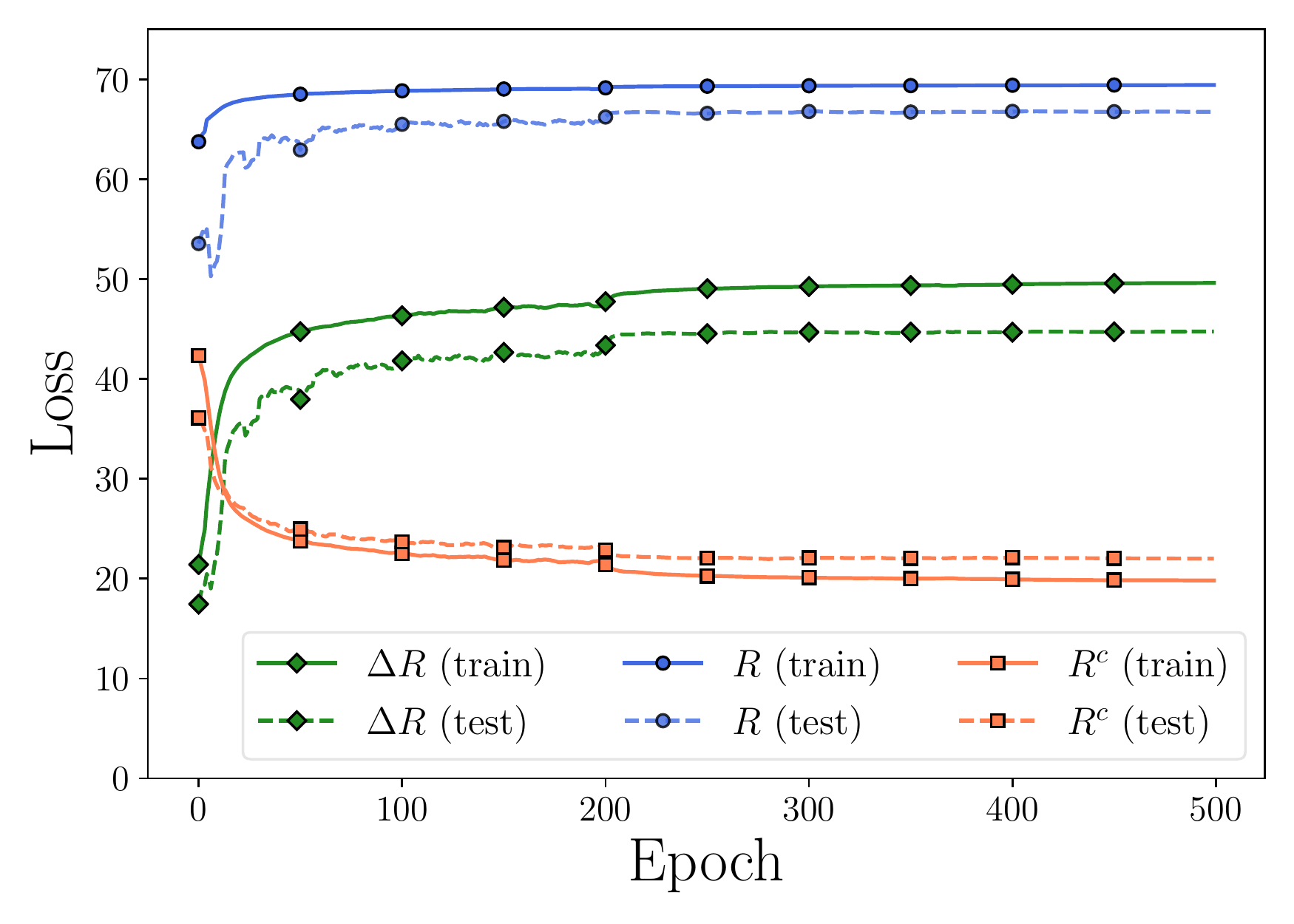}}
    \subfigure[\label{fig:train-test-loss-pca-3}PCA: {\small (\textbf{red}) overall data; (\textbf{blue}) individual classes}.]
    {\includegraphics[width=0.32\textwidth]{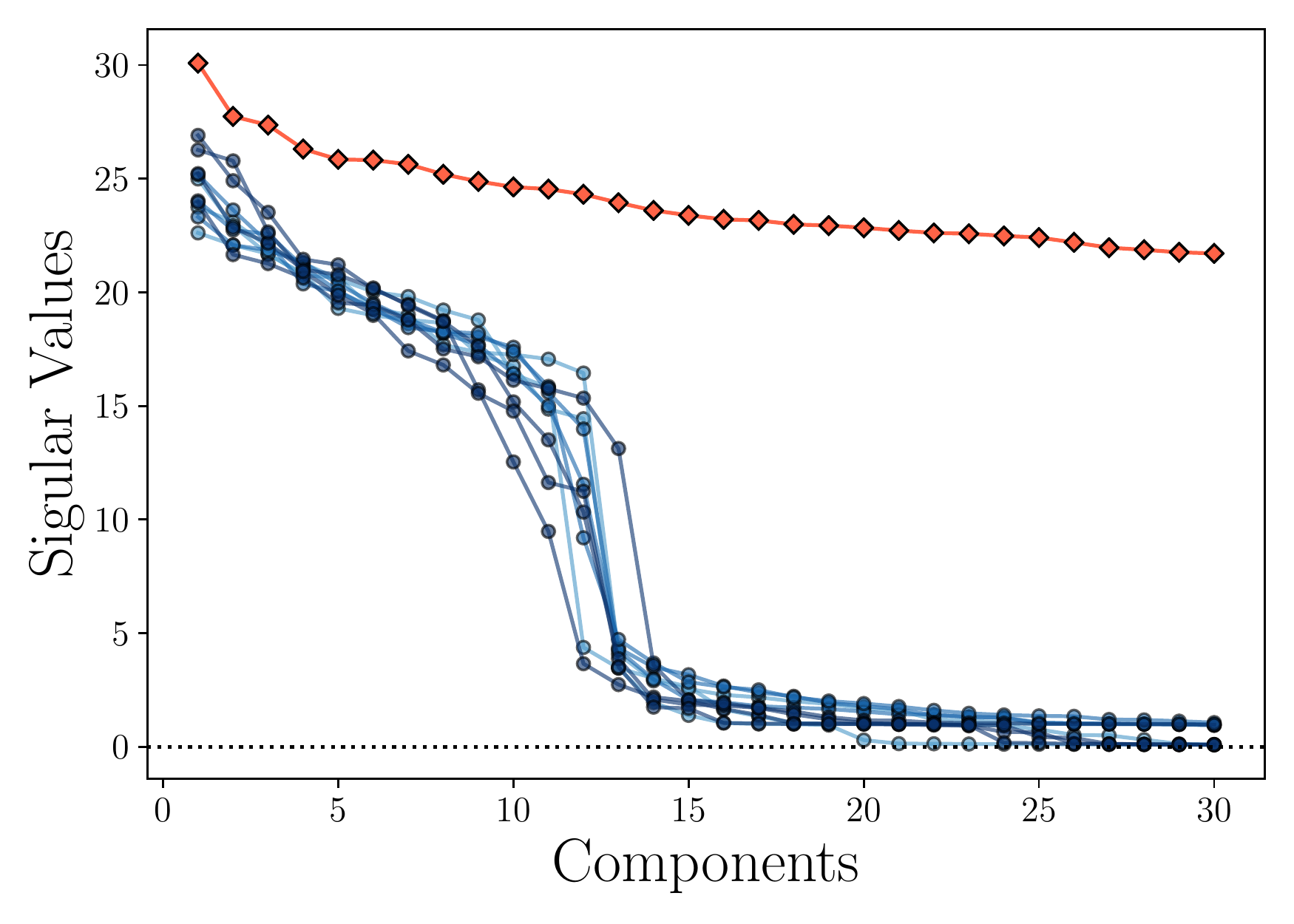}}
    \vskip -0.05in
    \caption{\small Evolution of the rates of MCR$^2$ in the training process and principal components of learned features.}
    \label{fig:train-test-loss-pca}
  \end{center}
  \vskip -0.1in
\end{figure*} 

\begin{figure*}[t]
  \begin{center}
    \subfigure[\label{fig:train-label-noise-1} $\Delta R\big(\Z(\theta), \bm{\Pi}, \epsilon\big)$.]{\includegraphics[width=0.32\textwidth]{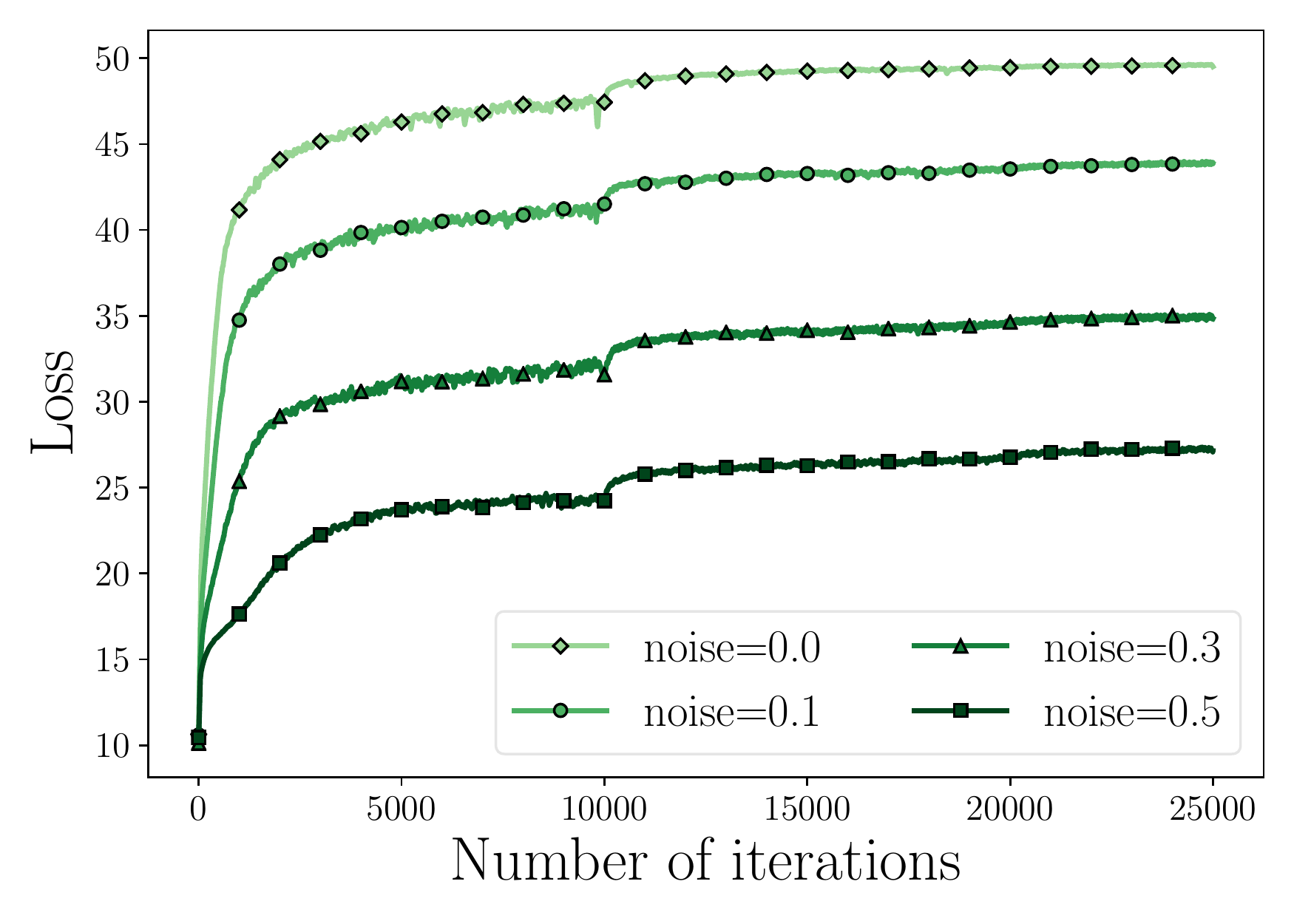}}
    \subfigure[\label{fig:train-label-noise-2} $R(\Z(\theta), \epsilon)$.]{\includegraphics[width=0.32\textwidth]{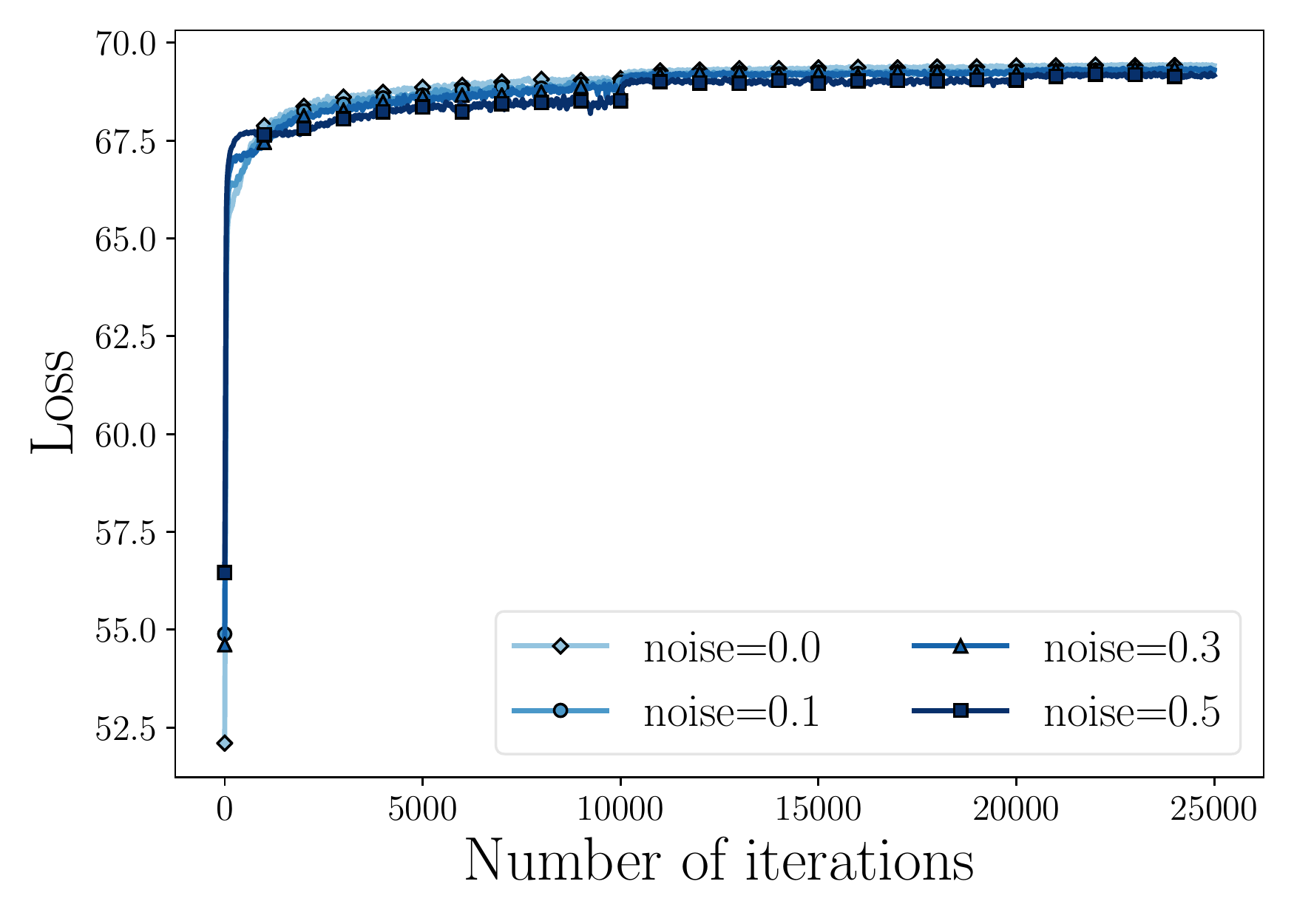}}
    \subfigure[\label{fig:train-label-noise-3}  $R^c(\Z(\theta),  \epsilon \mid \bm{\Pi})$.]{\includegraphics[width=0.32\textwidth]{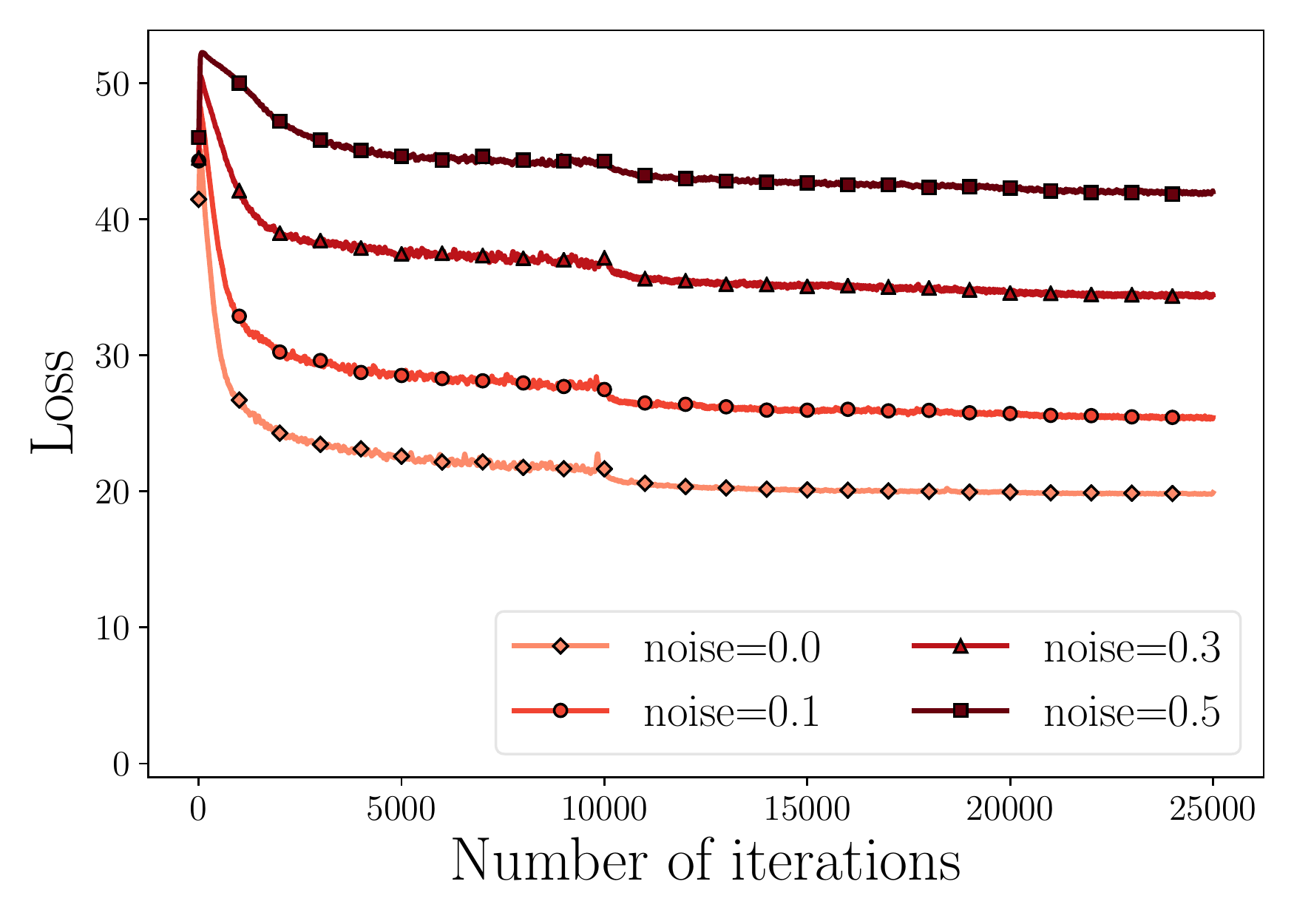}}
    \vskip -0.05in
    \caption{\small Evolution of rates $R, R^c, \Delta R$ of MCR$^2$ during training with corrupted labels.}
    \label{fig:train-label-noise}
  \end{center}
\vskip -0.15in
\end{figure*} 
% \vspace{-2mm}

\textbf{Experiments on real data.} We consider CIFAR10 dataset~\cite{krizhevsky2009learning} and ResNet-18~\cite{he2016deep} for $f(\cdot, \theta)$. We replace the last linear layer of ResNet-18 by a two-layer fully connected network with ReLU activation function such that the output dimension is 128. We set the mini-batch size as $m = 1,000$ and the precision parameter $\epsilon^2 = 0.5$. More results can be found in Appendix~\ref{sec:appendix-subsec-sup}.

Figure~\ref{fig:train-test-loss-pca-1} illustrates how the two rates and their difference (for both training and test data) evolves over epochs of training: After an initial phase, $R$ gradually increases while $R^c$ decreases, indicating that features $\bm Z$ are expanding as a whole while each class $\bm Z_j$ is being compressed.  
Figure~\ref{fig:train-test-loss-pca-3} shows the distribution of singular values per $\Z_j$ and Figure~\ref{fig:low-dim} (right) shows the angles of features sorted by class.  Compared to the geometric loss \cite{lezama2018ole}, our features are {\em not only orthogonal but also of much higher dimension}. We compare the singular values of representations, both overall data and individual classes, learned by using cross-entropy and MCR$^2$ in Figure~\ref{fig:pca-plot} and Figure~\ref{fig:heatmap-plot} in Appendix \ref{sec:subsec-pca}. We 
find that the representations learned by using MCR$^2$ loss are much more diverse than the ones learned by using cross-entropy loss. In addition, we find that we are able to select diverse images from the same class according to the ``principal'' components of the learned features (see Figure~\ref{fig:visual-class-2-8} and Figure~\ref{fig:visual-overall-data} in Appendix \ref{sec:subsec-pca}).

\textbf{Robustness to corrupted labels.} 
Because MCR$^2$ by design encourages richer representations that preserves intrinsic structures from the data $\X$, training relies less on class labels than traditional loss such as cross-entropy (CE). To verify this, we train the same network\footnote{Both CE and MCR$^2$ can have better performance by choosing larger models for our mapping.}  using both CE and MCR$^2$ with certain ratios of \textit{randomly corrupted} training labels.  Figure~\ref{fig:train-label-noise} illustrates the learning process: for different levels of corruption, while the rate for the whole set always converges to the same value, the rates for the classes are inversely proportional to the ratio of corruption, indicating our method only compress samples with valid labels. The classification results are summarized in Table~\ref{table:label-noise}. By applying \textit{exact the same} training parameters, MCR$^2$ is significantly more robust than CE, especially with  higher ratio of corrupted labels. This can be an advantage in the settings of self-supervised learning or constrastive learning when the grouping information can be very noisy. 

% \vspace{-1mm}
\begin{table}[h]
\begin{center}
\caption{\small Classification results with features learned with labels corrupted at different levels.}
\label{table:label-noise}
% \vskip -0.05in
\begin{small}
\begin{sc}
\begin{tabular}{l | c c c c c }
\toprule
 & Ratio=0.1 &  Ratio=0.2 &  Ratio=0.3 &  Ratio=0.4 &  Ratio=0.5 \\
\midrule
CE Training & 90.91\% & 86.12\% & 79.15\% & 72.45\%  & 60.37\% \\
MCR$^2$ Training  & \textbf{91.16\%} & \textbf{89.70\%} & \textbf{88.18\%} & \textbf{86.66\%} &  \textbf{84.30\%}\\
\bottomrule
\end{tabular}
\end{sc}
\end{small}
\end{center}
\vspace{-2mm}
\end{table}

\subsection{Self-supervised Learning of Invariant Features} 
% \vspace{-1mm}
\textbf{Learning invariant features via rate reduction.} Motivated by self-supervised learning algorithms~\cite{lecun2004learning,kavukcuoglu2009learning,oord2018representation,he2019momentum,wu2018unsupervised}, we use the MCR$^2$ principle to learn representations that are {\em invariant} to certain class of transformations/augmentations, say $\mathcal T$ with a distribution $P_{\mathcal T}$. Given a mini-batch of data  $\{ \bm{x}_j\}_{j=1}^{k}$, we augment each sample $\x_j$ with $n$  transformations/augmentations $\{\tau_{i}(\cdot)\}_{i=1}^n$ randomly drawn from $P_{\mathcal{T}}$.
We simply label all the augmented samples $\X_j = [\tau_{1}(\bm{x}_j), \ldots, \tau_{n}(\bm{x}_j)]$ of $\bm x_j$ as the $j$-th class, and $\Z_j$ the corresponding learned features. Using this self-labeled data, we train our feature mapping $f(\cdot, \theta)$ the same way as the supervised setting above. For every mini-batch, the total number of samples for training is $m = k n$.

\textbf{Evaluation via clustering.} To learn invariant features, our formulation itself does {\em not} require the original samples $\x_j$ come from a fixed number of classes. For evaluation, we may train on a few classes and observe how the learned features facilitate classification or clustering of the data. A common method to evaluate learned features is to train an additional linear classifier~\cite{oord2018representation,he2019momentum}, with ground truth labels. But for our purpose, because we explicitly verify whether the so-learned invariant features have good subspace structures when the samples come from $k$ classes, we use an off-the-shelf subspace clustering algorithm EnSC~\cite{you2016oracle}, which is computationally efficient and is provably correct for data with well-structured subspaces.
%which works best on data with well-structured subspaces. 
We also use K-Means on the original data $\bm X$ as our baseline for comparison.  We use normalized mutual information (NMI), clustering accuracy (ACC), and adjusted rand index (ARI) for our evaluation metrics, see Appendix~\ref{sec:appendix-subsec-clustering} for their detailed definitions.

\textbf{Controlling dynamics of expansion and compression.} By directly optimizing the rate reduction $\Delta R = R - R^c$, we achieve $0.570$ clustering accuracy on CIFAR10 dataset, which is the second best result compared with previous methods. More details can be found in Appendix~\ref{sec:appendix-subsec-selfsup}. Empirically, we observe that, without class labels, the overall \textit{coding rate} $R$ expands quickly and the MCR$^2$ loss saturates (at a local maximum), see Fig~\ref{fig:train-test-loss-selfsup-mcr}.  Our experience suggests that learning a good representation from unlabeled data might be too ambitious when directly optimizing the original $\Delta R$. Nonetheless, from the  geometric meaning of $R$ and $R^c$, one can design a different learning strategy by controlling the dynamics of expansion and compression differently during training. 
For instance, we may re-scale the rate by replacing $R(\Z, \epsilon)$ with $\widetilde R(\Z,\epsilon) \doteq \frac{1}{2 \gamma_1}\log\det(\I + \frac{\gamma_2 d}{ m\epsilon^{2}}\Z\Z^{\top})$. With $\gamma_1 = \gamma_2 = k$, the learning dynamics change from Fig~\ref{fig:train-test-loss-selfsup-mcr} to Fig~\ref{fig:train-test-loss-selfsup-mcr-ctrl}: All features are first compressed then gradually expand. We denote the controlled MCR$^2$ training by MCR$^2$-{\scriptsize CTRL}.

\begin{figure*}[t]
\subcapcentertrue
\begin{center}
    \subfigure[\label{fig:train-test-loss-selfsup-mcr}MCR$^2$]{\includegraphics[width=0.45\textwidth]{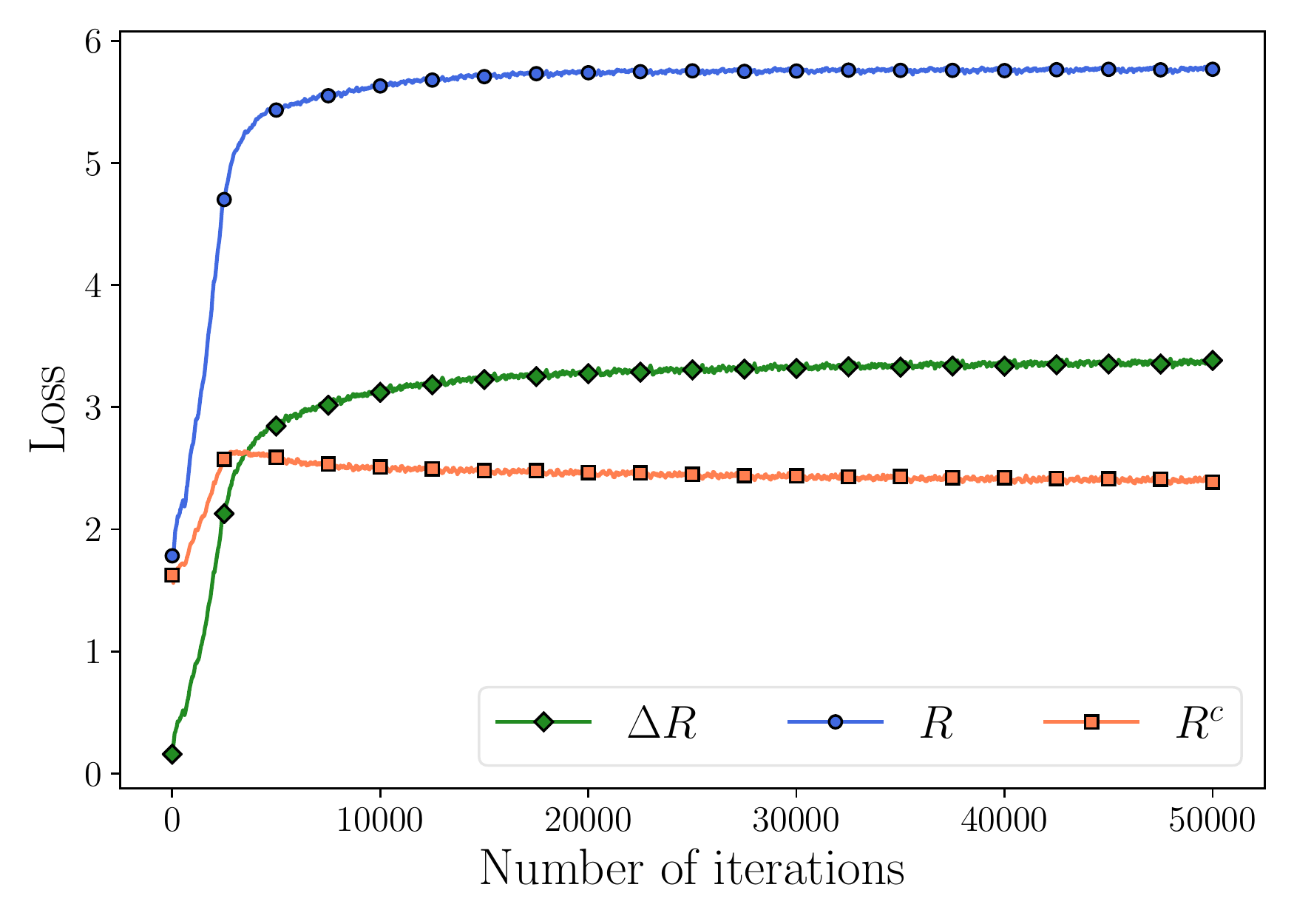}} 
    \hspace{5mm}
    \subfigure[\label{fig:train-test-loss-selfsup-mcr-ctrl}MCR$^2$-{\scriptsize CTRL}.]{\includegraphics[width=0.45\textwidth]{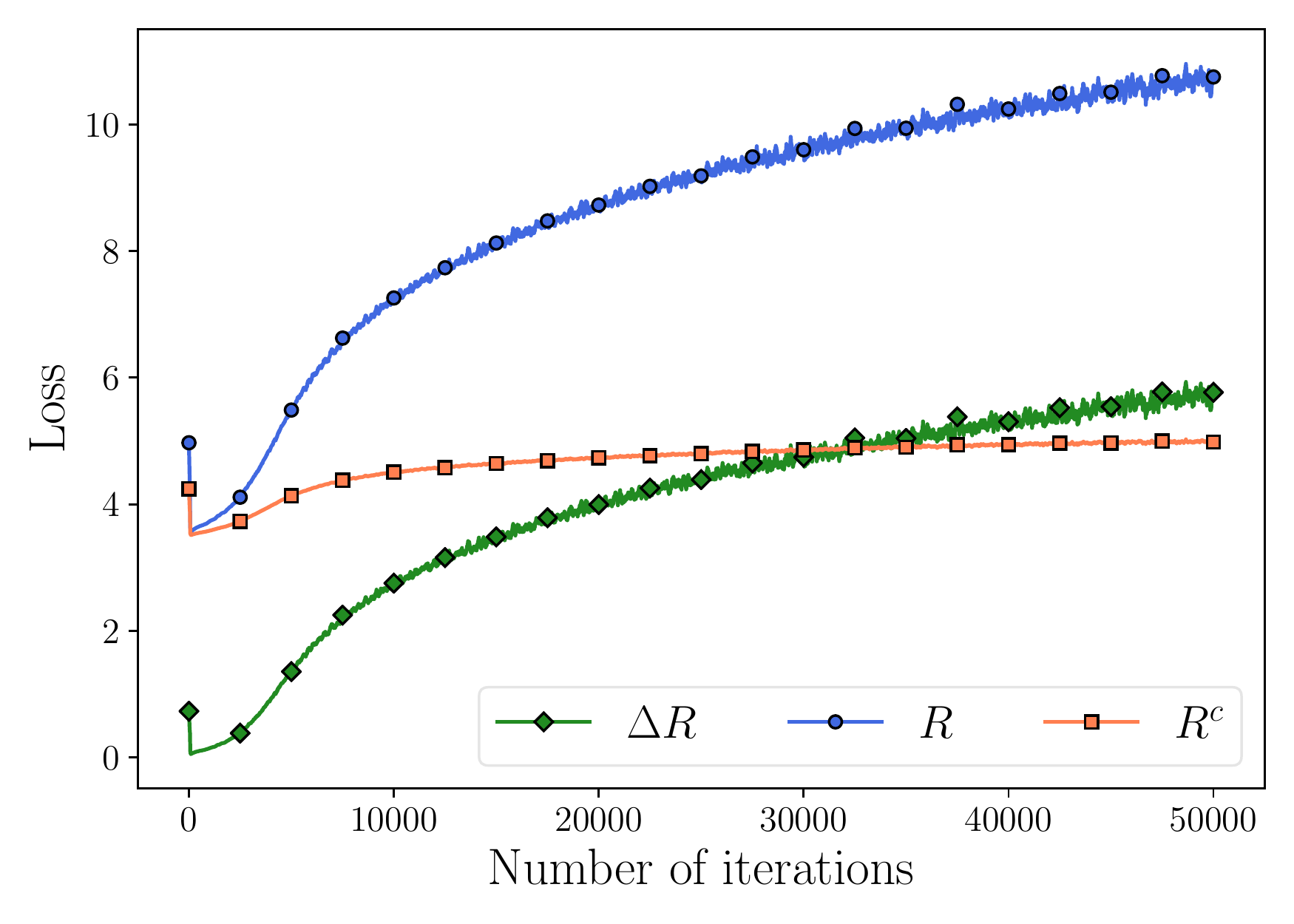}}
    \vskip -0.05in
    \caption{\small Evolution of the rates of (\textbf{left}) MCR$^2$   and (\textbf{right}) MCR$^2$-{\scriptsize CTRL} in the training process in the self-supervised setting on CIFAR10 dataset.}
\label{fig:training-dynamic-controlling-compare}
\end{center}
\vskip -0.1in
\end{figure*}

\textbf{Experiments on real data.} Similar to the supervised learning setting, we train {\em exactly the same} ResNet-18 network on the CIFAR10, CIFAR100, and STL10~\cite{coates2011analysis} datasets.
We set the mini-batch size as $k = 20$, number of augmentations for each sample as $n=50$ and the precision parameter as $\epsilon^2 = 0.5$. Table \ref{table:clustering} shows the results of the proposed MCR$^2$-{\scriptsize CTRL}  in comparison with methods JULE~\cite{yang2016joint}, RTM~\cite{nina2019decoder},  DEC~\cite{xie2016unsupervised}, DAC~\cite{chang2017deep}, and DCCM~\cite{wu2019deep} that have achieved the best results on these datasets. {Surprisingly, without utilizing any inter-class or inter-sample information and heuristics on the data, the invariant features learned by our method with augmentations alone achieves a better performance over other highly engineered clustering methods.} More ablation studies can be found in Appendix~\ref{sec:appendix-subsec-clustering}. 

Nevertheless, compared to the representations learned in the supervised setting where the optimal partition $\bm{\Pi}$ in \eqref{eqn:maximal-rate-reduction} is initialized by correct class information, the representations here learned with self-supervised  classes are far from being optimal\footnote{We find that the supervised learned representation on CIFAR10 in Section \ref{sec:supervised-experiments} can easily achieve a clustering accuracy over 99\% on the entire training data.} -- they at best correspond to local maxima of the MCR$^2$ objective \eqref{eqn:maximal-rate-reduction} when  $\theta$ and $\bm{\Pi}$ are {\em jointly optimized}. It remains wide open how to design better optimization strategies and dynamics to learn from unlabelled or partially-labelled data better representations (and the associated partitions) close to the global maxima of the MCR$^2$ objective \eqref{eqn:maximal-rate-reduction}.

\begin{table}[t]
\begin{center}
\caption{\small Clustering results on CIFAR10, CIFAR100, and STL10 datasets.}
\label{table:clustering}
% \vskip -0.07in
\begin{small}
\begin{sc}
\begin{tabular}{l l | c c c c c c c c c }
\toprule
Dataset & Metric & K-Means & JULE &  RTM    & DEC  & DAC  &  DCCM & MCR$^2$-{\scriptsize Ctrl} \\
\midrule
\multirow{3}{*}{CIFAR10}    
& NMI & 0.087 & 0.192 & 0.197 & 0.257 & 0.395  & 0.496 & \textbf{0.630}       \\
& ACC & 0.229 & 0.272 & 0.309 & 0.301 & 0.521  & 0.623 & \textbf{0.684}       \\
& ARI & 0.049 & 0.138 & 0.115 & 0.161 & 0.305  & 0.408 & \textbf{0.508}       \\
\midrule
\multirow{3}{*}{CIFAR100}    
& NMI & 0.084 & 0.103 & - & 0.136 & 0.185  & 0.285 & \textbf{0.362}       \\
& ACC & 0.130 & 0.137 & - & 0.185 & 0.237  & 0.327 & \textbf{0.347}       \\
& ARI & 0.028 & 0.033 & - & 0.050 & 0.087  & \textbf{0.173} & {0.167}       \\
\midrule
\multirow{3}{*}{STL10}    
& NMI & 0.124 & 0.182 & - & 0.276 &  0.365 & 0.376 &  \textbf{0.446}       \\
& ACC & 0.192 & 0.182 & - & 0.359 & 0.470 & 0.482  &  \textbf{0.491}       \\
& ARI & 0.061 & 0.164 & - & 0.186 & 0.256 & 0.262 &  \textbf{0.290}       \\
\bottomrule
\end{tabular}
\end{sc}
\end{small}
\end{center}
\vskip -0.2in
% \vspace{-4mm}
\end{table}

\section{Conclusion and Future Work}
This work provides rigorous theoretical justifications and clear empirical evidences for why the maximal coding rate reduction (MCR$^2$) is a fundamental principle for learning discriminative low-dim representations in almost all learning settings. It unifies and explains existing effective frameworks and heuristics widely practiced in the (deep) learning literature. It remains open {\em why} MCR$^2$ is robust to label noises in the supervised setting,  {\em why} self-learned features with MCR$^2$ alone are effective for clustering, and {\em how} in future practice instantiations of this principle can be systematically harnessed to further improve clustering or classification tasks. 

We believe that MCR$^2$ gives a principled and practical objective for (deep) learning and can potentially lead to better design operators and architectures of a deep network. A potential direction is to  monitor quantitatively the amount of rate reduction $\Delta R$ gained through every layer of the deep network. By optimizing the rate reduction through the network layers, it is no longer engineered as a ``black box.'' 

On the learning theoretical aspect, although this work has demonstrated only with mixed subspaces, this principle applies to any mixed distributions or structures, for which configurations that achieve maximal rate reduction are of independent theoretical interest. Another interesting note is that the MCR$^2$ formulation goes beyond the supervised multi-class learning setting often studied through empirical risk minimization (ERM)   \cite{PAC-multiclass-2015}. It is more related to the expectation maximization (EMX) framework \cite{bendavid2017learning}, in which the notion of ``compression'' plays a crucial role for purely theoretical analysis. We hope this work provides a good connection between machine learning theory and its practice.

\section*{Acknowledgements}
Yi would like to thank Professor Yann LeCun of New York University for having a stimulating discussion in his NYU office last November about the search for a proper ``energy'' function for features to be learned by a deep network \cite{lecun2006tutorial}, during the preparation of a joint proposal. Professor John Wright of Columbia University, who was the leading author of the first two papers on the lossy coding approach to clustering and classification \cite{ma2007segmentation,wright2008classification}, has provided valuable insights and suggestions during germination of this work. We would like to thank Professor Emmanuel Cand{\'e}s of Stanford University for having an online discussion with Yi, during the pandemic, about the rate distortion function for low-dimensional structures. Yi also likes to thank Professor Zhi Ding of UC Davis for discussing the role of rate distortion and lossy coding in communications and information theory and for providing us some pertinent references. 

Professor Shankar Sastry of UC Berkeley has always encouraged us to look into fundamental connections between low-dimensional subspaces and deep learning from the perspective of Generalized PCA \cite{GPCA}. Coincidentally, this work was partly motivated to address an inquiry from Professor Ruzena Bajcsy of UC Berkeley earlier this year on how to clarify the role of ``latent features'' learned in a network in a principled manner. We would also like to thank Professor Jiantao Jiao and Professor Jacob Steinhardt of UC Berkeley for extensive discussions about how to make deep learning robust. During the preparation of this manuscript, Dr. Harry Shum, who collaborated with Yi on lossy coding during his visit to Microsoft Research Asia in 2007 \cite{wright2008classification}, has given excellent suggestions on how to better visualize the learned features, leading to some of the interesting illustrations in this work. 

Yaodong would like to thank Zitong Yang and Xili Dai for helpful discussions on the $\log \det (\cdot)$ function. Ryan would like to thank Yuexiang Zhai for helpful discussions on learning subspace structures. Last but not the least,  we are very grateful for Xili Dai and Professor Xiaojun Yuan of UESTC and Professor Hao Chen of UC Davis who have generously provided us their GPU clusters to help us conduct the extensive experiments reported in this paper.

\bibliographystyle{alpha}
\bibliography{reference}

\onecolumn

% \begin{center}

% \Large
% \textbf{Learning Diverse and Discriminative Representations via the Principle of Maximal Coding Rate Reduction\\    
% \vspace{3mm}
% {\em (Supplementary Materials)}\vspace{2mm}}
    
% \end{center}

% \appendix

% \numberwithin{equation}{section}
% \numberwithin{figure}{section}
% \numberwithin{table}{section}

\begin{appendices}

\section{Properties of the Rate Reduction Function}\label{ap:rate-reduction}

This section is organized as follows. 
We present background and preliminary results for the $\log\det(\cdot)$ function and the coding rate function in Section~\ref{sec:theory-preliminary}. 
Then, Section~\ref{sec:theory-bounds-rate} and \ref{sec:theory-bounds-rate-reduction} provide technical lemmas for bounding the coding rate and coding rate reduction functions, respectively. Such lemmas are key results for proving our main theoretical results, which are stated informally in Theorem~\ref{thm:MCR2-properties} and formally in Section~\ref{sec:theory-main}.
Finally, proof of our main theoretical results is provided in Section~\ref{sec:theory-proof}. 

\paragraph{Notations} Throughout this section, we use $\bbS_{++}^d$, $\Re_+$ and $\mathbb Z_{++}$ to denote the set of symmetric positive definite matrices of size $d \times d$, nonnegative real numbers and positive integers, respectively.

\subsection{Preliminaries}
\label{sec:theory-preliminary}
\paragraph{Properties of the $\log\det(\cdot)$ function. }

\begin{lemma}\label{thm:logdet-strictly-concave}
The function $\log\det(\cdot): \bbS_{++}^d \to \R$ is strictly concave. That is, 
\begin{equation*}
    \log\det((1-\alpha) \Z_1 + \alpha \Z_2)) \ge (1-\alpha)\log\det(\Z_1) + \alpha\log\det(\Z_2)
\end{equation*}
for any $\alpha \in (0, 1)$ and $\{\Z_1, \Z_2\} \subseteq \bbS_{++}^d$, with equality holds if and only if $\Z_1 = \Z_2$.
\begin{proof}
Consider an arbitrary line given by $\Z = \Z_0 + t \Delta\Z$ where $\Z_0$ and $\Delta\Z \ne \0$ are symmetric matrices of size $d\times d$. 
Let $f(t) \doteq \log\det(\Z_0 + t \Delta\Z)$ be a function defined on an interval of values of $t$ for which $\Z_0 + t\Delta\Z \in \bbS_{++}^d$.
Following the same argument as in \cite{boyd2004convex}, we may assume $\Z_0 \in \bbS_{++}^d$ and get
\begin{equation*}
    f(t) = \log\det \Z_0 + \sum_{i=1}^d \log(1+ t\lambda_i),
\end{equation*}
where $\{\lambda_i\}_{i=1}^d$ are eigenvalues of $\Z_0^{-\frac{1}{2}}\Delta\Z \Z_0^{-\frac{1}{2}}$. 
The second order derivative of $f(t)$ is given by
\begin{equation*}
    f''(t) = -\sum_{i=1}^d \frac{\lambda_i^2}{(1+t\lambda_i)^2} < 0.
\end{equation*}
Therefore, $f(t)$ is strictly concave along the line $\Z = \Z_0 + t\Delta\Z$. 
By definition, we conclude that $\log\det(\cdot)$ is strictly concave.
\end{proof}
\end{lemma}

\paragraph{Properties of the coding rate function. }
The following properties, also known as the Sylvester's determinant theorem, for the coding rate function are known in the paper \cite{ma2007segmentation}.
\begin{lemma}[Commutative property \cite{ma2007segmentation}]\label{thm:coding-rate-commute}
For any $\Z \in \R^{d\times m}$ we have 
\begin{equation*}
    R(\Z,\epsilon) \doteq \frac{1}{2}\log \det\left(\I + \frac{d}{m\epsilon^{2}}\Z\Z^{\top}\right) =\frac{1}{2}\log\det\left(\I + \frac{d}{m\epsilon^{2}}\Z^{\top}\Z\right).
\end{equation*}
\end{lemma}

\begin{lemma}[Invariant property \cite{ma2007segmentation}]\label{thm:coding-rate-invariant}
For any $\Z \in \R^{d\times m}$ and any orthogonal matrices $\U\in \R^{d\times d}$ and $\V \in \R^{m\times m}$ we have 
\begin{equation*}
    R(\Z,\epsilon) = R(\U\Z\V^\top,\epsilon).
\end{equation*}
\end{lemma}

\subsection{Lower and Upper Bounds for Coding Rate}
\label{sec:theory-bounds-rate}

The following result provides an upper and a lower bound on the coding rate of $\Z$ as a function of the coding rate for its components $\{\Z_j\}_{j=1}^k$. 
The lower bound is tight when all the components $\{\Z_j\}_{j=1}^k$ have the same covariance (assuming that they have zero mean).
The upper bound is tight when the components $\{\Z_j\}_{j=1}^k$ are pair-wise orthogonal.

\begin{lemma}\label{thm:coding-rate-bounds}
For any $\{\Z_j \in \Re^{d\times m_j}\}_{j=1}^k$ and any $\epsilon > 0$, let $\Z = [\Z_1, \cdots, \Z_k] \in \Re^{d\times m}$ with $m=\sum_{j=1}^k m_j$. We have
\begin{equation}\label{eq:coding-rate-bounds}
\begin{split}
\sum_{j=1}^k \frac{m_j}{2}  \log \det\left(\I + \frac{d}{m_j\epsilon^2}\Z_j \Z_j^\top\right)
&\le  
\frac{m}{2} \log\det\left(\I + \frac{d}{m\epsilon^2}\Z \Z^\top\right)  \\
&\le 
\sum_{j=1}^k \frac{m}{2}  \log \det\left(\I + \frac{d}{m\epsilon^2}\Z_j \Z_j^\top\right),
\end{split}
\end{equation}
where the first equality holds if and only if
$$\frac{\Z_1 \Z_1^\top}{m_1} = \frac{\Z_2 \Z_2^\top}{m_2}=\cdots= \frac{\Z_k \Z_k^\top}{m_k},$$
and the second equality holds if and only if $\Z_{j_1}^\top \Z_{j_2} = \0$ for all $1 \le j_1 < j_2 \le k$.
\end{lemma}
\begin{proof}
By Lemma~\ref{thm:logdet-strictly-concave}, $\log \det(\cdot)$ is strictly concave. Therefore,
\begin{align*}
    \log\det\Big(\sum_{j=1}^k \alpha_j \S_j\Big) \ge \sum_{j=1}^k \alpha_j \log\det(\S_j), ~\text{for all}~ \{\alpha_j > 0\}_{j=1}^k, \sum_{j=1}^k \alpha_j = 1 ~\text{and}~\{\S_j \in \bbS_{++}^d\}_{j=1}^k,
\end{align*}
where equality holds if and only if $\S_1 = \S_2 = \cdots = \S_k$. Take $\alpha_j = \frac{m_j}{m}$ and $\S_j = \I + \frac{d}{m_j \epsilon^2} \Z_j\Z_j^\top$, we get
\begin{equation*}
     \frac{m}{2} \log\det\left(\I + \frac{d}{m\epsilon^2}\Z \Z^\top\right) 
     \ge
     \sum_{j=1}^k \frac{m_j}{2}  \log \det\left(\I + \frac{d}{m_j\epsilon^2}\Z_j \Z_j^\top\right),
\end{equation*}
with equality holds if and only if $\frac{\Z_1\Z_1^\top}{m_1} = \cdots = \frac{\Z_k\Z_k^\top}{m_k}$. 
This proves the lower bound in \eqref{eq:coding-rate-bounds}. 

We now prove the upper bound. 
By the strict concavity of $\log\det(\cdot)$, we have
\begin{equation*}
    \log\det(\Q) \le \log\det(\S) + \langle \nabla \log\det(\S), \,\Q - \S\rangle, ~\text{for all}~\{\Q, \S\}\subseteq \bbS_{++}^{m},
\end{equation*}
where equality holds if and only if $\Q = \S$. 
Plugging in $\nabla \log\det(\S) = \S^{-1}$ (see e.g., \cite{boyd2004convex}) and $\S^{-1} = (\S^{-1})^{\top}$ gives
\begin{equation}\label{eq:prf-logdet-gradient-inequality}
    \log\det(\Q) \le \log\det(\S) + \tr(\S^{-1} \Q) - m.
\end{equation}

We now take
\begin{gather}\label{eq:prf-logdet-gradient-inequality-QS}
    \Q = \I + \frac{d}{m\epsilon^2}\Z^\top \Z = \I + \frac{d}{m\epsilon^2}
    \begin{bmatrix}
    \Z_1^\top \Z_1 & \Z_1^\top \Z_2  & \cdots & \Z_1^\top \Z_k \\
    \Z_2^\top \Z_1 & \Z_2^\top  \Z_2 & \cdots & \Z_2^\top \Z_2 \\
    \vdots         & \vdots          & \ddots & \vdots \\
    \Z_k^\top \Z_1 & \Z_k^\top \Z_2  & \cdots & \Z_k^\top \Z_k \\
    \end{bmatrix}, ~\text{and}~ 
    \\
    \S = \I + \frac{d}{m\epsilon^2}
    \begin{bmatrix}
    \Z_1^\top \Z_1 & \0              & \cdots & \0 \\
    \0             & \Z_2^\top  \Z_2 & \cdots & \0 \\
    \vdots         & \vdots          & \ddots & \vdots \\
    \0             & \0              & \cdots & \Z_k^\top \Z_k \\
    \end{bmatrix}. \nonumber
\end{gather}
From the property of determinant for block diagonal matrix, we have
\begin{equation}\label{eq:prf-logdet-gradient-inequality-term1}
    \log\det(\S) = \sum_{j=1}^k \log\det \left(\I + \frac{d}{m\epsilon^2} \Z_j^\top \Z_j\right).
\end{equation}
Also, note that
\begin{align}\label{eq:prf-logdet-gradient-inequality-term2}
&\tr(\S^{-1} \Q) 
\nonumber\\
= \ 
&\tr
\begin{bmatrix}
(\I +\frac{d}{m\epsilon^2}\Z_1^\top \Z_1)^{-1}(\I +\frac{d}{m\epsilon^2}\Z_1^\top \Z_1)  & \cdots & (\I +\frac{d}{m\epsilon^2}\Z_1^\top \Z_1)^{-1}(\I +\frac{d}{m\epsilon^2}\Z_1^\top \Z_k) \\
\vdots         & \ddots & \vdots \\
(\I +\frac{d}{m\epsilon^2}\Z_k^\top \Z_k)^{-1}(\I +\frac{d}{m\epsilon^2}\Z_k^\top \Z_1)            & \cdots & (\I +\frac{d}{m\epsilon^2}\Z_k^\top \Z_k)^{-1}(\I +\frac{d}{m\epsilon^2}\Z_k^\top \Z_k) \\
\end{bmatrix}
\nonumber\\
= \
&\tr     \begin{bmatrix}
\I            & \cdots   & * \\
\vdots        & \ddots   & \vdots \\
*             & \cdots   & \I \\
\end{bmatrix}
= m,
\end{align}
where ``*'' denotes nonzero quantities that are irrelevant for the purpose of computing the trace. 
Plugging \eqref{eq:prf-logdet-gradient-inequality-term1} and \eqref{eq:prf-logdet-gradient-inequality-term2} back in \eqref{eq:prf-logdet-gradient-inequality}
gives 
\begin{equation*}
    \frac{m}{2} \log\det\left(\I + \frac{d}{m\epsilon^2}\Z^\top \Z\right) \le \sum_{j=1}^k \frac{m}{2}  \log \det\left(\I + \frac{d}{m\epsilon^2}\Z_j^\top \Z_j\right),
\end{equation*}
where the equality holds if and only if $\Q = \S$, which by the formulation in \eqref{eq:prf-logdet-gradient-inequality-QS}, holds if and only if $\Z_{j_1}^\top \Z_{j_2} = \0$ for all $1 \le j_1 < j_2 \le k$. 
Further using the result in Lemma~\ref{thm:coding-rate-commute} gives 
\begin{equation*}
    \frac{m}{2} \log\det\left(\I + \frac{d}{m\epsilon^2}\Z \Z^\top\right) \le \sum_{j=1}^k \frac{m}{2}  \log \det\left(\I + \frac{d}{m\epsilon^2}\Z_j \Z_j^\top\right),
\end{equation*}
which produces the upper bound in \eqref{eq:coding-rate-bounds}. 
\end{proof}

\subsection{An Upper Bound on Coding Rate Reduction}
\label{sec:theory-bounds-rate-reduction}

We may now provide an upper bound on the coding rate reduction $\Delta R(\Z, \bm{\Pi}, \epsilon)$ (defined in \eqref{eqn:maximal-rate-reduction}) in terms of its individual components $\{\Z_j\}_{j=1}^k$.
\begin{lemma}\label{thm:rate-reduction-bound}
For any $\Z \in \Re^{d\times m}, \bm{\Pi} \in \Omega$ and $\epsilon > 0$, let $\Z_j \in \Re^{d\times m_j}$ be $\Z \bm{\Pi}_j$ with zero columns removed. We have
\begin{equation}\label{eq:rate-reduction-bound}
    \Delta R(\Z, \bm{\Pi}, \epsilon) \le 
    \sum_{j=1}^k \frac{1}{2m}\log\left( \frac{\det^m\left(\I + \frac{d}{m\epsilon^2}\Z_j \Z_j^\top\right)}{\det^{m_j}\left(\I + \frac{d}{m_j\epsilon^2}\Z_j \Z_j^\top\right)}\right),
\end{equation}
with equality holds if and only if $\Z_{j_1}^\top \Z_{j_2} = \0$ for all $1 \le j_1 < j_2 \le k$.
\end{lemma}
\begin{proof}
From \eqref{eqn:coding-length-eval}, \eqref{eqn:compress-loss-eval} and \eqref{eqn:coding-length-reduction}, we have
\begin{equation*}
\begin{split}
&\quad \,\, \Delta R(\Z, \bm{\Pi}, \epsilon) \\
&= R(\Z, \epsilon) - R^c(\Z, \epsilon \mid  \bm{\Pi})\\
&= \frac{1}{2}\log \left(\det\left(\I + \frac{d}{m\epsilon^{2}}\Z\Z^{\top}\right)\right) - \sum_{j=1}^{k}\left\{\frac{\tr(\bm{\Pi}_j)}{2m}\log \left(\det\left(\I + d\frac{\Z\bm{\Pi}_j\Z^{\top}}{\tr(\bm{\Pi}_j)\epsilon^{2}}\right)\right)\right\}\\
&= \frac{1}{2}\log \left(\det\left(\I + \frac{d}{m\epsilon^{2}}\Z\Z^{\top}\right)\right) - \sum_{j=1}^{k}\left\{\frac{m_j}{2m}\log \left(\det\left(\I + d\frac{\Z_j\Z_j^{\top}}{m_j\epsilon^{2}}\right)\right)\right\}\\
&\le \sum_{j=1}^k \frac{1}{2}  \log\left( \det\left(\I + \frac{d}{m\epsilon^2}\Z_j \Z_j^\top\right)\right) - \sum_{j=1}^{k}\left\{\frac{m_j}{2m}\log \left(\det\left(\I + d\frac{\Z_j \Z_j^{\top}}{m_j\epsilon^{2}}\right)\right)\right\}\\
&= \sum_{j=1}^k \frac{1}{2m}  \log\left(\det {\!}^m\left(\I + \frac{d}{m\epsilon^2}\Z_j \Z_j^\top\right)\right) - \sum_{j=1}^{k}\left\{\frac{1}{2m}\log \left(\det{\!}^{m_j}\left(\I + d\frac{\Z_j \Z_j^{\top}}{m_j\epsilon^{2}}\right)\right)\right\}\\
&= \sum_{j=1}^k \frac{1}{2m}\log\left( \frac{\det^m\left(\I + \frac{d}{m\epsilon^2}\Z_j \Z_j^\top\right)}{\det^{m_j}\left(\I + \frac{d}{m_j\epsilon^2}\Z_j \Z_j^\top\right)}\right),
\end{split}
\end{equation*}
where the inequality follows from the upper bound in Lemma~\ref{thm:coding-rate-bounds}, and that the equality holds if and only if $\Z^{\top}_{j_1} \Z_{j_2} = \0$ for all $1 \le j_1 < j_2 \le k$.
\end{proof}

\subsection{Main Results: Properties of Maximal Coding Rate Reduction}
\label{sec:theory-main}

We now present our main theoretical results.  
The following theorem states that for any fixed encoding of the partition $\bm{\Pi}$, the coding rate reduction is maximized by data $\Z$ that is maximally discriminative between different classes and is diverse within each of the classes. 
This result holds provided that the sum of rank for different classes is small relative to the ambient dimension, and that $\epsilon$ is small. 

\begin{theorem}\label{thm:maximal-rate-reduction}
Let $\bm{\Pi} = \{\bm{\Pi}_j \in \Re^{m \times m}\}_{j=1}^{k}$ with $\{\bm{\Pi}_j \ge \mathbf{0}\}_{j=1}^k$ and \, $\bm{\Pi}_1 + \cdots + \bm{\Pi}_k = \I$ be a given set of diagonal matrices whose diagonal entries encode the membership of the $m$ samples in the $k$ classes.
Given any $\epsilon > 0$, $d > 0$ and  $\{d \ge d_j>0\}_{j=1}^k$, consider the optimization problem
\begin{equation}\label{eq:maximal-rate-reduction-thm}
\begin{split}
    \Z^* \in &\argmax_{\Z\in \Re^{d \times m}} \Delta R(\Z, \bm{\Pi}, \epsilon) \\ 
    & \ \text{s.t.}~\|\Z\bm{\Pi}_j\|_F^2 = \tr({\bm{\Pi}_j}), \ \rank(\Z\bm{\Pi}_j) \le d_j, \ \forall j \in \{1, \ldots, k\}.
\end{split}
\end{equation}
Under the conditions 
\begin{itemize}
\item \emph{(Large ambient dimension)} $d \ge \sum_{j=1}^k d_j$, and
\item \emph{(High coding precision)} $\epsilon ^4 < \min_{j \in \{1, \ldots, k\}}\left\{\frac{\tr({\bm{\Pi}}_j)}{m}\frac{d^2}{d_j^2}\right\}$,
\end{itemize}
the optimal solution $\Z^*$ satisfies
\begin{itemize}
    \item \emph{(Between-class discriminative)} $(\Z_{j_1}^*)^\top \Z_{j_2}^* = \0$ for all $1 \le j_1 < j_2 \le k$, i.e., $\Z_{j_1}^*$ and $\Z_{j_2}^*$ lie in orthogonal subspaces, and
    \item \emph{(Within-class diverse)} For each $j \in \{1, \ldots, k\}$, the rank of $\Z_j^*$ is equal to $d_j$ and either all singular values of $\Z_j^*$ are equal to $\frac{\tr({\bm{\Pi}_j})}{d_j}$, or the $d_j -1$ largest singular values of $\Z_j^*$ are equal and have value larger than $\frac{\tr({\bm{\Pi}_j})}{d_j}$,
\end{itemize}
where $\Z_j^* \in \Re^{d\times \tr{(\bm{\Pi}}_j)}$ denotes $\Z^* \bm{\Pi}_j$ with zero columns removed. 
\end{theorem}

\subsection{Proof of Main Results}
\label{sec:theory-proof}

We start with presenting a lemma that will be used in the proof to Theorem~\ref{thm:maximal-rate-reduction}.
\begin{lemma}\label{thm:generic-simplex-optimization}
Given any twice differentiable $f: \Re_+ \to \Re$, integer $r \in \mathbb{Z}_{++}$ and $c \in \Re_+$, consider the optimization problem
\begin{equation}
\label{eq:generic-simplex-optimization}
\begin{split}
    &\max_{\x} \ \sum_{p=1}^r f(x_p)  \\
    & \ \ \text{s.t.} \ \x=[x_1, \ldots, x_r] \in \Re_+^r, \ x_1 \ge x_2 \ge \cdots \ge x_r, \ \text{and} \ \sum_{p=1}^r x_p = c. 
\end{split}
\end{equation}
Let $\x^*$ be an arbitrary global solution to \eqref{eq:generic-simplex-optimization}.
If the conditions
\begin{itemize}
    \item $f'(0) < f'(x)$ for all $x > 0$,
    \item There exists $x_T > 0$ such that $f'(x)$ is strictly increasing in $[0, x_T]$ and strictly decreasing in $[x_T, \infty)$,
    \item $f''(\frac{c}{r}) < 0$ (equivalently, $\frac{c}{r} > x_T$),
\end{itemize}
are satisfied, then we have either
\begin{itemize}
    \item $\x^* = [\frac{c}{r}, \ldots, \frac{c}{r}]$, or
    \item $\x^* = [x_H, \ldots, x_H, x_L]$ for some $x_H \in (\frac{c}{r}, \frac{c}{r-1})$ and $x_L > 0$.
\end{itemize} 
\end{lemma}
\begin{proof}
The result holds trivially if $r = 1$. Throughout the proof we consider the case where $r > 1$.

We consider the optimization problem with the inequality constraint $x_1 \ge \cdots \ge x_r$ in \eqref{eq:generic-simplex-optimization} removed:
\begin{equation}\label{eq:prf-generic-simplex-optimization}
\max_{\x=[x_1, \ldots, x_r] \in \Re_+^r} \ \sum_{p=1}^r f(x_p)  ~~~~\text{s.t.}~ \sum_{p=1}^r x_p = c.
\end{equation}
We need to show that any global solution $\x^* = [x_1^*, \ldots, x_r^*]$ to \eqref{eq:prf-generic-simplex-optimization} is either $\x^* = [\frac{c}{r}, \ldots, \frac{c}{r}]$ or $\x^* = [x_H, \ldots, x_H, x_L]\cdot \bm{P}$ for some $x_H > \frac{c}{r}$, $x_L > 0$ and permutation matrix $\bm{P} \in \Re^{r \times r}$.
Let 
\begin{equation*}
    \cL(\x, \blambda) = \sum_{p=1}^r f(x_p) - \lambda_0 \cdot \left(\sum_{p=1}^r x_p - c\right) - \sum_{p=1}^r \lambda_p x_p
\end{equation*}
be the Lagragian function for \eqref{eq:prf-generic-simplex-optimization} where $\blambda = [\lambda_0, \lambda_1, \ldots, \lambda_r]$ is the Lagragian multiplier. 
By the first order optimality conditions (i.e., the Karush–Kuhn–Tucker (KKT) conditions, see, e.g., \cite[Theorem 12.1]{nocedal2006numerical}), there exists $\blambda^*= [\lambda_0^*, \lambda_1^*, \ldots, \lambda_r^*]$ such that
\begin{align}
    \sum_{p=1}^r x_q^* &= c,\label{eq:kkt1}\\
    x_q^* &\ge 0, ~\forall q\in \{1, \ldots, r\},\label{eq:kkt2}\\
    \lambda_q^* &\ge 0, ~\forall q\in \{1, \ldots, r\},\label{eq:kkt3}\\
    \lambda_q^* \cdot x_q^* &= 0, ~\forall q\in \{1, \ldots, r\}, ~~\text{and}~\label{eq:kkt4}\\
    [f'(x_1^*), \ldots, f'(x_r^*)] &= [\lambda_0^*, \ldots, \lambda_0^*] + [\lambda_1^*, \ldots, \lambda_r^*].\label{eq:kkt5}
\end{align}
By using the KKT conditions, we first show that all entries of $\x^*$ are strictly positive. 
To prove by contradiction, suppose that $\x^*$ has $r_0$ nonzero entries and $r - r_0$ zero entries for some $1 \le r_0 < r$. 
Note that $r_0 \ge 1$ since an all zero vector $\x^*$ does not satisfy the equality constraint \eqref{eq:kkt1}.

Without loss of generality, we may assume that $x_p^* > 0$ for $p \le r_0$ and $x_p^* = 0$ otherwise. 
By \eqref{eq:kkt4}, we have
\begin{equation*}
    \lambda_1^* = \cdots = \lambda_{r_0}^* = 0.
\end{equation*}
Plugging it into \eqref{eq:kkt5}, we get
\begin{equation*}
    f'(x_1^*) = \cdots = f'(x_{r_0}^*) = \lambda_0^*.
\end{equation*}
From \eqref{eq:kkt5} and noting that $x_{r_0+1}=0$ we get
\begin{equation*}
    f'(0) = f'(x_{r_0+1}) = \lambda_0^* + \lambda_{r_0 + 1}^*.
\end{equation*}
Finally, from \eqref{eq:kkt3}, we have
\begin{equation*}
    \lambda_{r_0+1}^* \ge 0.
\end{equation*}
Combining the last three equations above gives $f'(0) - f'(x_1^*) \ge 0$, contradicting the assumption that $f'(0) < f'(x)$ for all $x > 0$. 
This shows that $r_0 = r$, i.e., all entries of $\x^*$ are strictly positive.
Using this fact and \eqref{eq:kkt4} gives 
\begin{equation*}
    \lambda_p^* = 0 ~~\text{for all}~p \in \{1, \ldots, r\}.
\end{equation*}
Combining this with \eqref{eq:kkt5} gives
\begin{equation}\label{eq:prf-equal-first-order}
    f'(x_1^*) = \cdots = f'(x_{r}^*) = \lambda_0^*.
\end{equation}
It follows from the fact that $f'(x)$ is strictly unimodal that
\begin{equation}\label{eq:prf-two-values}
    \exists \ x_H \ge x_L > 0 ~~\text{s.t.}~~\{x_p^*\}_{p=1}^r \subseteq \{x_L, x_H\}.
\end{equation}
That is, the set $\{x_p^*\}_{p=1}^r$ may contain no more than two values. 
To see why this is true, suppose that there exists three distinct values for $\{x_p^*\}_{p=1}^r$. 
Without loss of generality we may assume that $0 < x_1^* < x_2^* < x_3^*$. 
If $x_2^* \le x_T$ (recall $x_T := \arg\max_{x\ge 0} f'(x)$), then by using the fact that $f'(x)$ is strictly increasing in $[0, x_T]$, we must have $f'(x_1^*) < f'(x_2^*)$ which contradicts \eqref{eq:prf-equal-first-order}. 
A similar contradiction is arrived by considering $f'(x_2^*)$ and $f'(x_3^*)$ for the case where $x_2^* > x_T$. 

There are two possible cases as a consequence of \eqref{eq:prf-two-values}. 
First, if $x_L = x_H$, then we have $x_1^* = \cdots = x_r^*$. 
By further using \eqref{eq:kkt1} we get 
\begin{equation*}
     x_1^* = \cdots = x_r^* = \frac{c}{r}.
\end{equation*}

It remains to consider the case where $x_L < x_H$. 
First, by the unimodality of $f'(x)$, we must have $x_L < x_T < x_H$, therefore 
\begin{equation}\label{eq:prf-second-order-sign}
    f''(x_L) > 0 ~\text{and}~f''(x_H) < 0.
\end{equation} 
Let $\ell := |\{p: x_p = x_L\}|$ be the number of entries of $\x^*$ that are equal to $x_L$ and $h:= r - \ell$. 
We show that it is necessary to have $\ell = 1$ and $h = r-1$. 
To prove by contradiction, assume that $\ell > 1$ and $h < r-1$. 
Without loss of generality we may assume $\{x_p^* = x_H\}_{p=1}^{h}$ and $\{x_p^* = x_L\}_{p=h+1}^{r}$. 
By \eqref{eq:prf-second-order-sign}, we have
\begin{equation*}
    f''(x_p^*) > 0 ~\text{for all}~p > h.
\end{equation*}
In particular, by using $h < r-1$ we have
\begin{equation}\label{eq:prf-last-two-positive}
    f''(x_{r-1}^*) > 0 ~\text{and}~f''(x_{r}^*) > 0.
\end{equation}
On the other hand, by using the second order necessary conditions for constraint optimization (see, e.g., \cite[Theorem 12.5]{nocedal2006numerical}), the following result holds
\begin{equation}\label{eq:prf-second-order}
\begin{split}
    \v^\top \nabla_{\x\x}\mathcal{L}(\x^*, \blambda^*) \v &\le 0, ~~\text{for all}~ \left\{\v: \left\langle \nabla_{\x}\left(\sum_{p=1}^r x_p^* - c\right), \v \right\rangle = 0\right\}\\
    \iff\quad \sum_{p=1}^r f''(x_p^*) \cdot v_p^2 &\le 0, ~~\text{for all}~ \left\{\v=[v_1, \ldots, v_r]: \sum_{p=1}^r v_p = 0\right\}.
\end{split}
\end{equation}
Take $\v$ to be such that $v_1 = \cdots = v_{r-2} = 0$ and $v_{r-1} = - v_r \ne 0$. Plugging it into \eqref{eq:prf-second-order} gives
\begin{equation*}
    f''(x_{r-1}^*) + f''(x_{r}^*) \le 0, 
\end{equation*}
which contradicts \eqref{eq:prf-last-two-positive}. 
Therefore, we may conclude that $\ell = 1$. 
That is, $\x^*$ is given by
\begin{equation*}
\x^* = [x_H, \ldots, x_H, x_L], \ \text{where} \ x_H > x_L > 0.
\end{equation*}
By using the condition in \eqref{eq:kkt1}, we may further show that
\begin{align*}
&(r-1) x_H + x_L = c \implies x_H = \frac{c}{r-1} - \frac{c}{x_L} < \frac{x_L}{r-1}, \\
&(r-1) x_H + x_L = c \implies (r-1) x_H + x_H > c \implies x_H > \frac{c}{r},
\end{align*}
which completes our proof.
\end{proof}

\begin{proof}[Proof of Theorem~\ref{thm:maximal-rate-reduction}]
Without loss of generality, let $\Z^* = [\Z_1^*, \ldots, \Z_k^*]$ be the optimal solution of problem~\eqref{eq:maximal-rate-reduction-thm}. 

To show that $\Z_j^*, j \in \left\{ 1, \dots, k\right\}$ are pairwise orthogonal, suppose for the purpose of arriving at a contradiction that $(\Z_{j_1}^*)^\top \Z_{j_2}^* \ne \0$ for some $1 \le j_1 < j_2 \le k$. 
By using Lemma~\ref{thm:rate-reduction-bound}, the strict inequality in \eqref{eq:rate-reduction-bound} holds for the optimal solution $\Z^*$. That is, 
\begin{equation}\label{eq:prf-optimal-strict-inequality}
    \Delta R(\Z^*, \bm{\Pi}, \epsilon) < 
    \sum_{j=1}^k \frac{1}{2m}\log\left( \frac{\det^m \left(\I + \frac{d}{m\epsilon^2}\Z_j^* (\Z_j^*)^\top\right)}{\det^{m_j}\left(\I + \frac{d}{m_j\epsilon^2}\Z_j^* (\Z_j^*)^\top\right)}\right).
\end{equation}
On the other hand, since $\sum_{j=1}^k d_j \le d$, there exists $\{\U_j' \in \R^{d \times d_j}\}_{j=1}^k$ such that the columns of the matrix $[\U_1', \ldots, \U_k']$ are orthonormal. 
Denote $\Z_j^* = \U_j^* \bfSigma_j^* (\V_j^*)^\top$ the compact SVD of $\Z_j^*$, and let
\begin{equation*}
    \Z' = [\Z_1', \ldots, \Z_k'], ~~\text{where}~ \Z_j' = \U_j' \bfSigma_j^* (\V_j^*)^\top.
\end{equation*}
It follows that 
\begin{equation*}
    (\Z_{j_1}')^\top \Z_{j_2}' = \V_{j_1}^*\bfSigma_{j_1}^* (\U_{j_1}')^\top \U_{j_2}' \bfSigma_{j_2}^* (\V_{j_2}^*)^\top= \V_{j_1}^*\bfSigma_{j_1}^*  \0  \bfSigma_{j_2}^* (\V_{j_2}^*)^\top = \0 ~~\text{for all}~ 1\le j_1 < j_2 \le k. 
\end{equation*}
That is, the matrices $\Z_1', \ldots, \Z_k'$ are pairwise orthogonal. 
Applying Lemma~\ref{thm:rate-reduction-bound} for $\Z'$ gives
\begin{equation}\label{eq:prf-constructed-optimal}
\begin{split}
    \Delta R(\Z', \bm{\Pi}, \epsilon) &= 
    \sum_{j=1}^k \frac{1}{2m}\log\left( \frac{\det^m\left(\I + \frac{d}{m\epsilon^2}\Z_j' (\Z_j')^\top\right)}{\det^{m_j}\left(\I + \frac{d}{m_j\epsilon^2}\Z_j' (\Z_j')^\top\right)}\right)\\
    &= \sum_{j=1}^k \frac{1}{2m}\log\left( \frac{\det^m\left(\I + \frac{d}{m\epsilon^2}\Z_j^* (\Z_j^*)^\top\right)}{\det^{m_j}\left(\I + \frac{d}{m_j\epsilon^2}\Z_j^* (\Z_j^*)^\top\right)}\right),
\end{split}
\end{equation}
where the second equality follows from Lemma~\ref{thm:coding-rate-invariant}.
Comparing \eqref{eq:prf-optimal-strict-inequality} and \eqref{eq:prf-constructed-optimal} gives $\Delta R(\Z', \bm{\Pi}, \epsilon) > \Delta R(\Z^*, \bm{\Pi}, \epsilon)$, which contradicts the optimality of $\Z^*$.
Therefore, we must have 
\begin{equation*}
    (\Z_{j_1}^*)^\top \Z_{j_2}^* = \0 ~\text{for all}~1 \le j_1 < j_2 \le k.
\end{equation*} 
Moreover, from Lemma~\ref{thm:coding-rate-invariant} we have
\begin{equation}\label{eq:prf-coding-rate-decomposition}
    \Delta R(\Z^*, \bm{\Pi}, \epsilon) = 
    \sum_{j=1}^k \frac{1}{2m}\log\left( \frac{\det^m \left(\I + \frac{d}{m\epsilon^2}\Z_j^* (\Z_j^*)^\top\right)}{\det^{m_j}\left(\I + \frac{d}{m_j\epsilon^2}\Z_j^* (\Z_j^*)^\top\right)}\right).
\end{equation}
We now prove the result concerning the singular values of $\Z_j^*$. 
To start with, we claim that the following result holds:
\begin{equation}\label{eq:prf-rich-key}
     \Z_j^* \in \arg\max_{\Z_j} \  \log\left(\frac{\det^m\left(\I + \frac{d}{m\epsilon^2}\Z_j \Z_j^\top\right)}{\det^{m_j}\left(\I + \frac{d}{m_j\epsilon^2}\Z_j \Z_j^\top
     \right)}\right) ~~\text{s.t.}~\|\Z_j\|_F^2 = m_j,\, \rank(\Z_j) \le d_j.
\end{equation}
To see why \eqref{eq:prf-rich-key} holds, suppose that there exists $\widetilde{\Z}_j$ such that $\|\widetilde{\Z}_j\|_F^2 = m_j$, $\rank(\widetilde{\Z}_j) \le d_j$ and
\begin{equation}\label{eq:prf-tildeZ-inequality}
    \log\left(\frac{\det^m\left(\I + \frac{d}{m\epsilon^2}\widetilde{\Z}_j \widetilde{\Z}_j^\top\right)}{\det^{m_j}\left(\I + \frac{d}{m_j\epsilon^2}\widetilde{\Z}_j \widetilde{\Z}_j^\top\right)}\right) > \log\left(\frac{\det^m\left(\I + \frac{d}{m\epsilon^2}\Z_j^* (\Z_j^*)^\top\right)}{\det^{m_j}\left(\I + \frac{d}{m_j\epsilon^2}\Z_j^* (\Z_j^*)^\top\right)}\right).
\end{equation}

Denote $\widetilde{\Z}_j = \widetilde{\U}_j \widetilde{\bfSigma}_j\widetilde{\V}_j^\top$ the compact SVD of $\widetilde{\Z}_j$ and let
\begin{equation*}
    \Z' = [\Z_1^*, \ldots, \Z_{j-1}^*, \Z_j', \Z_{j+1}^*, \ldots, \Z_k^*], ~~\text{where}~\Z_j' := \U_j^* \widetilde{\bfSigma}_j\widetilde{\V}_j^\top.
\end{equation*}
Note that $\|\Z_j'\|_F^2 = m_j$, $\rank(\Z_j') \le d_j$ and $(\Z_j')^\top \Z_{j'}^* = \0$ for all $j' \ne j$. 
It follows that $\Z'$ is a feasible solution to \eqref{eq:maximal-rate-reduction-thm} and that the components of $\Z'$ are pairwise orthogonal.  
By using Lemma~\ref{thm:rate-reduction-bound}, Lemma~\ref{thm:coding-rate-invariant} and \eqref{eq:prf-tildeZ-inequality} we have
\begin{equation*}
\begin{split}
    &\Delta R(\Z', \bm{\Pi}, \epsilon) \\
    =\ & \frac{1}{2m}\log\left( \frac{\det^m\left(\I + \frac{d}{m\epsilon^2}\Z_j' (\Z_j')^\top\right)}{\det^{m_j}\left(\I + \frac{d}{m_j\epsilon^2}\Z_j' (\Z_j')^\top\right)}\right) + 
    \sum_{j' \ne j} \frac{1}{2m}\log\left( \frac{\det^m\left(\I + \frac{d}{m\epsilon^2}\Z_{j'}^* (\Z_{j'}^*)^\top\right)}{\det^{m_{j'}}\left(\I + \frac{d}{m_{j'}\epsilon^2}\Z_{j'}^* (\Z_{j'}^*)^\top\right)}\right)\\
    =\ & \frac{1}{2m}\log\left( \frac{\det^m\left(\I + \frac{d}{m\epsilon^2}\widetilde{\Z}_j (\widetilde{\Z}_j)^\top\right)}{\det^{m_j}\left(\I + \frac{d}{m_j\epsilon^2}\widetilde{\Z}_j (\widetilde{\Z}_j)^\top\right)}\right) + 
    \sum_{j' \ne j} \frac{1}{2m}\log\left( \frac{\det^m\left(\I + \frac{d}{m\epsilon^2}\Z_{j'}^* (\Z_{j'}^*)^\top\right)}{\det^{m_{j'}}\left(\I + \frac{d}{m_{j'}\epsilon^2}\Z_{j'}^* (\Z_{j'}^*)^\top\right)}\right)\\
    >\ & \frac{1}{2m}\log\left(\frac{\det^m\left(\I + \frac{d}{m\epsilon^2}\Z_j^* (\Z_j^*)^\top\right)}{\det^{m_j}\left(\I + \frac{d}{m_j\epsilon^2}\Z_j^* (\Z_j^*)^\top\right)}\right) + 
    \sum_{j' \ne j} \frac{1}{2m}\log\left( \frac{\det^m\left(\I + \frac{d}{m\epsilon^2}\Z_{j'}^* (\Z_{j'}^*)^\top\right)}{\det^{m_{j'}}\left(\I + \frac{d}{m_{j'}\epsilon^2}\Z_{j'}^* (\Z_{j'}^*)^\top\right)}\right)\\
    =\ & \sum_{j=1}^k \frac{1}{2m}\log \left(\frac{\det^m \left(\I + \frac{d}{m\epsilon^2}\Z_j^* (\Z_j^*)^\top\right)}{\det^{m_j}\left(\I + \frac{d}{m_j\epsilon^2}\Z_j^* (\Z_j^*)^\top\right)}\right).
\end{split}
\end{equation*}
Combining it with \eqref{eq:prf-coding-rate-decomposition} shows $\Delta R(\Z', \bm{\Pi}, \epsilon) > \Delta R(\Z^*, \bm{\Pi}, \epsilon)$, contradicting the optimality of $\Z^*$.
Therefore, the result in \eqref{eq:prf-rich-key} holds.

Observe that the optimization problem in \eqref{eq:prf-rich-key} depends on $\Z_j$ only through its singular values. 
That is, by letting $\bfsigma_j:=[\sigma_{1,j}, \ldots, \sigma_{\min(m_j, d),j}]$ be the singular values of $\Z_j$, we have
\begin{equation*}
% \label{eq:prf-spectral-functions}
     \log\left(\frac{\det^m\left(\I + \frac{d}{m\epsilon^2}\Z_j \Z_j^\top\right)}{\det^{m_j}\left(\I + \frac{d}{m_j\epsilon^2}\Z_j \Z_j^\top
     \right)} \right)
     = \sum_{p=1}^{\min\{m_j, d\}} \log\left( \frac{(1+\frac{d}{m \epsilon^2} \sigma_{p,j}^2)^m}{(1+\frac{d}{m_j \epsilon^2} \sigma_{p,j}^2)^{m_j}}\right),
\end{equation*}
also, we have
\begin{equation*}
  \|\Z_j\|_F^2 = \sum_{p=1}^{\min\{m_j, d\}} \sigma_{p,j}^2  ~~\text{and}~~ \rank(\Z_j) = \|\bfsigma_j\|_0. 
\end{equation*}
Using these relations, \eqref{eq:prf-rich-key} is equivalent to
\begin{equation}\label{eq:prf-sigma-optimization-all}
\begin{split}
    &\max_{\bfsigma_j \in \Re_+^{\min\{m_j, d\}}} \sum_{p=1}^{\min\{m_j, d\}} \log\left( \frac{(1+\frac{d}{m \epsilon^2} \sigma_{p,j}^2)^m}{(1+\frac{d}{m_j \epsilon^2} \sigma_{p,j}^2)^{m_j}}\right) \\
    &\ ~~\text{s.t.}~\sum_{p=1}^{\min\{m_j, d\}} \sigma_{p,j}^2 = m_j, ~\text{and}~ \ \rank(\Z_j) = \|\bfsigma_j\|_0 
\end{split}
\end{equation}
Let $\bfsigma_j^* = [\sigma_{1,j}^*, \ldots, \sigma_{\min\{m_j, d\},j}^*]$ be an optimal solution to \eqref{eq:prf-sigma-optimization-all}. 
Without loss of generality we assume that the entries of $\bfsigma_j^*$ are sorted in descending order. 
It follows that 
\begin{equation*}
\sigma_{p, j}^* = 0 \ ~\text{for all}~ \ p > d_j,
\end{equation*} 
and
\begin{equation}\label{eq:prf-sigma-optimization}
    [\sigma_{1,j}^*, \ldots, \sigma_{d_j, j}^*] = \argmax_{\substack{[\sigma_{1,j}, \ldots, \sigma_{d_j, j}] \in \Re^{d_j}_{+}\\ \sigma_{1,j} \ge \cdots \ge \sigma_{d_j,j}}} \  \sum_{p=1}^{d_j} \log\left( \frac{(1+\frac{d}{m \epsilon^2} \sigma_{p,j}^2)^m}{(1+\frac{d}{m_j \epsilon^2} \sigma_{p,j}^2)^{m_j}}\right)~~~~\text{s.t.}~\sum_{p=1}^{d_j} \sigma_{p,j}^2 = m_j.
\end{equation}

Then we define 
\begin{equation*}
    f(x; d, \epsilon, m_j, m) = \log\left( \frac{(1+\frac{d}{m \epsilon^2} x)^m}{(1+\frac{d}{m_j \epsilon^2} x)^{m_j}}\right),
\end{equation*}
and rewrite \eqref{eq:prf-sigma-optimization} as
\begin{equation}\label{eq:prf-sigma-optimization_f}
    \max_{\substack{[x_1, \ldots, x_{d_j}]  \in \Re_+^{d_j}\\x_1 \ge \cdots \ge x_{d_j}}} \ \sum_{p=1}^{d_j}  f(x_p; d, \epsilon, m_j, m) \ ~~\text{s.t.}~ \sum_{p=1}^{d_j} x_p = m_j.
\end{equation}
We compute the first and second derivative for $f$ with respect to $x$, which are given by
\begin{align*}
    f'(x; d, \epsilon, m_j, m) &= \frac{d^2 x (m-m_j)}{(dx+m\epsilon^2)(dx+m_j\epsilon^2)}, \\
    f''(x; d, \epsilon, m_j, m) &= \frac{d^2(m-m_j)(m m_j\epsilon^4 - d^2 x^2)}{(dx+m\epsilon^2)^2(dx+m_j\epsilon^2)^2}.
\end{align*}
Note that
\begin{itemize}
    \item $0 = f'(0) < f'(x)$ for all $x > 0$,
    \item $f'(x)$ is strictly increasing in $[0, x_T]$ and strictly decreasing in $[x_T, \infty)$, where $x_T =\epsilon^2\sqrt{\frac{m}{d}\frac{m_j}{d}}$, and
    \item by using the condition $\epsilon ^4 < \frac{m_j}{m}\frac{d^2}{d_j^2}$, we have $f''(\frac{m_j}{d_j}) < 0$.
\end{itemize}
Therefore, we may apply Lemma~\ref{thm:generic-simplex-optimization} and conclude that the unique optimal solution to \eqref{eq:prf-sigma-optimization_f} is either
\begin{itemize}
    \item $\x^* = [\frac{m_j}{d_j}, \ldots, \frac{m_j}{d_j}]$, or
    \item $\x^* = [x_H, \ldots, x_H, x_L]$ for some $x_H \in (\frac{m_j}{d_j}, \frac{m_j}{d_j -1})$ and $x_L > 0$.
\end{itemize} 
Equivalently, we have either
\begin{itemize}
    \item $[\sigma_{1, j}^*, \ldots, \sigma_{d_j, j}^*] = \left[\sqrt{\frac{m_j}{d_j}}, \ldots, \sqrt{\frac{m_j}{d_j}}\right]$, or
    \item $[\sigma_{1, j}^*, \ldots, \sigma_{d_j, j}^*] = [\sigma_H, \ldots, \sigma_H, \sigma_L]$ for some $\sigma_H \in \left(\sqrt{\frac{m_j}{d_j}}, \sqrt{\frac{m_j}{d_j -1}}\right)$ and $\sigma_L > 0$,
\end{itemize} 
as claimed.
\end{proof}

\newpage
\section{Additional Simulations and Experiments}\label{ap:additional-exp}
\subsection{Simulations - Verifying Diversity Promoting Properties of MCR$^2$}
As proved in Theorem~\ref{thm:maximal-rate-reduction}, the proposed MCR$^2$ objective promotes within-class diversity. In this section, we use simulated data to verify the diversity promoting property of MCR$^2$. As shown in Table~\ref{table:simulations}, we calculate our proposed MCR$^2$ objective on simulated data. We observe that orthogonal subspaces with \textit{higher} dimension achieve higher MCR$^2$ value, which is consistent with our theoretical analysis in Theorem~\ref{thm:maximal-rate-reduction}.

\begin{table}[h]
\begin{center}
\caption{\small \textbf{MCR$^2$ objective on simulated data.} We evaluate the proposed MCR$^2$ objective defined in \eqref{eqn:maximal-rate-reduction}, including ${R}$,  ${R}^{c}$, and  $\Delta{R}$, on simulated data. The output dimension $d$ is set as 512, 256, and 128. We set the batch size as $m=1000$ and random assign the label of each sample from $0$ to $9$, i.e., 10 classes.  We generate two types of data: 1) (\textsc{Random Gaussian}) For comparison with data without structures, for each class we generate random vectors sampled from Gaussian distribution (the dimension is set as the output dimension $d$) and normalize each vector to be on the unit sphere. 2) (\textsc{Subspace}) For each class, we generate vectors sampled from its corresponding subspace with  dimension $d_j$ and normalize each vector to be on the unit sphere. We consider the subspaces from different classes are orthogonal/nonorthogonal to each other.}
\label{table:simulations}
\begin{small}
\begin{sc}
\begin{tabular}{l | c c c c c}
\toprule
& ${R}$ &  ${R}^{c}$ &  $\Delta{R}$ &   Orthogonal? & Output Dimension\\
\midrule
Random Gaussian & 552.70 & 193.29 & 360.41 & {\cmark} & 512 \\
Subspace  ($d_j = 50$) & 545.63 & 108.46 & \textbf{437.17} & {\cmark} & 512 \\
Subspace  ($d_j = 40$)  & 487.07 & 92.71 & 394.36 & {\cmark} & 512 \\
Subspace  ($d_j = 30$) & 413.08 & 74.84 & 338.24 & {\cmark} & 512 \\
Subspace  ($d_j = 20$) & 318.52 &	54.48 &	264.04 & {\cmark} & 512 \\
Subspace  ($d_j = 10$) & 195.46 &	30.97 &	164.49 & {\cmark} & 512 \\
Subspace  ($d_j = 1$) & 31.18 &	4.27 &	26.91 & {\cmark} & 512 \\
\midrule
Random Gaussian & 292.71 &	154.13 &	138.57 & {\cmark} & 256 \\
Subspace  ($d_j = 25$) & 288.65 &	56.34 &	\textbf{232.31} & {\cmark} & 256 \\
Subspace  ($d_j = 20$)  & 253.51 &	47.58 &	205.92 & {\cmark} & 256 \\
Subspace  ($d_j = 15$) & 211.97 &	38.04 &	173.93 & {\cmark} & 256 \\
Subspace  ($d_j = 10$) & 161.87 &	27.52 &	134.35 & {\cmark} & 256 \\
Subspace  ($d_j = 5$) & 98.35 &	15.55 &	82.79 & {\cmark} & 256 \\
Subspace  ($d_j = 1$) & 27.73 &	3.92 &	23.80 & {\cmark} & 256 \\
\midrule
Random Gaussian & 150.05 &	110.85 &	39.19 & {\cmark} & 128 \\
Subspace  ($d_j = 12$) & 144.36 &	27.72 &	\textbf{116.63} & {\cmark} & 128 \\
Subspace  ($d_j = 10$) & 129.12 &	24.06 &	105.05 & {\cmark} & 128 \\
Subspace  ($d_j = 8$) & 112.01 &	20.18 &	91.83 & {\cmark} & 128 \\
Subspace  ($d_j = 6$) & 92.55 &	16.04 &	76.51 & {\cmark} & 128 \\
Subspace  ($d_j = 4$) & 69.57 &	11.51 &	58.06 & {\cmark} & 128 \\
Subspace  ($d_j = 2$) & 41.68 &	6.45 &	35.23 & {\cmark} & 128 \\
Subspace  ($d_j = 1$) & 24.28 &	3.57 &	20.70 & {\cmark} & 128 \\
\midrule
Subspace  ($d_j = 50$) & 145.60 &	75.31 &	70.29 & {\xmark} & 128 \\
Subspace  ($d_j = 40$) & 142.69 &	65.68 &	77.01 & {\xmark} & 128 \\
Subspace  ($d_j = 30$) & 135.42 &	54.27 &	81.15 & {\xmark} & 128 \\
Subspace  ($d_j = 20$) & 120.98 &	40.71 &	80.27 & {\xmark} & 128 \\
Subspace  ($d_j = 15$) & 111.10 &	32.89 &	78.21 & {\xmark} & 128 \\
Subspace  ($d_j = 12$) & 101.94 &	27.73 &	74.21 & {\xmark} & 128 \\
\bottomrule
\end{tabular}
\end{sc}
\end{small}
\end{center}
\vskip -0.1in
\end{table}

\subsection{Implementation Details}\label{sec:appendix-exp}
\paragraph{Training Setting.} We mainly use ResNet-18~\cite{he2016deep} in our experiments, where we use 4 residual blocks with layer widths $[64, 128, 256, 512]$.  The implementation of network architectures used in this paper are mainly based on this github repo.\footnote{\url{https://github.com/kuangliu/pytorch-cifar}} For data augmentation in the supervised setting, we apply the \texttt{RandomCrop} and \texttt{RandomHorizontalFlip}. For the supervised setting, we train the models for 500 epochs and use stage-wise learning rate decay every 200 epochs (decay by a factor of 10). For the supervised setting, we train the models for 100 epochs and use stage-wise learning rate decay at 20-th epoch and 40-th epoch (decay by a factor of 10).

\paragraph{Evaluation Details.} For the supervised setting, we set the number of principal components for nearest subspace classifier $r_j = 30$. We also study the effect of $r_j$ in Section~\ref{sec:appendix-subsec-sup}. For the CIFAR100 dataset, we consider 20 superclasses and set the cluster number as 20, which is the same setting as in \cite{chang2017deep, wu2018unsupervised}.

\paragraph{Datasets.} We apply the default datasets in PyTorch, including CIFAR10, CIFAR100, and STL10.

\paragraph{Augmentations $\mathcal{T}$ used for the self-supervised setting.} We apply the same data augmentation for CIFAR10 dataset and CIFAR100 dataset and the pseudo-code is as follows.

\begin{tcolorbox}
\begin{footnotesize}
\begin{verbatim}
import torchvision.transforms as transforms
TRANSFORM = transforms.Compose([
    transforms.RandomResizedCrop(32),
    transforms.RandomHorizontalFlip(),
    transforms.RandomApply([transforms.ColorJitter(0.4, 0.4, 0.4, 0.1)], p=0.8),
    transforms.RandomGrayscale(p=0.2),
    transforms.ToTensor()])
\end{verbatim}
\end{footnotesize}
\end{tcolorbox}

The augmentations we use for STL10 dataset and the pseudo-code is as follows.

\begin{tcolorbox}
\begin{footnotesize}
\begin{verbatim}
import torchvision.transforms as transforms
TRANSFORM = transforms.Compose([
    transforms.RandomResizedCrop(96),
    transforms.RandomHorizontalFlip(),
    transforms.RandomApply([transforms.ColorJitter(0.8, 0.8, 0.8, 0.2)], p=0.8),
    transforms.RandomGrayscale(p=0.2),
    GaussianBlur(kernel_size=9),
    transforms.ToTensor()])
\end{verbatim}
\end{footnotesize}
\end{tcolorbox}

\paragraph{Cross-entropy training details.} For CE models presented in Table~\ref{table:label-noise}, Figure \ref{fig:pca-ce-1}-\ref{fig:pca-ce-3}, and Figure~\ref{fig:heatmap-plot}, we use the same network architecture, ResNet-18~\cite{he2016deep}, for cross-entropy training on CIFAR10, and set the output dimension as 10 for the last layer.  We apply SGD, and set learning rate $\texttt{lr=0.1}$, momentum $\texttt{momentum=0.9}$, and weight decay $\texttt{wd= 5e{-4}}$. We set the total number of training epoch as 400, and use stage-wise learning rate decay every 150 epochs (decay by a factor of 10).

\subsection{Additional Experimental Results}

\subsubsection{PCA Results of MCR$^2$ Training versus Cross-Entropy Training}\label{sec:subsec-pca}

\begin{figure*}[h]
\subcapcentertrue
  \begin{center}
    \subfigure[\label{fig:pca-mcr-1}PCA: MCR$^2$ training learned features for overall data (first 30 components).]{\includegraphics[width=0.31\textwidth]{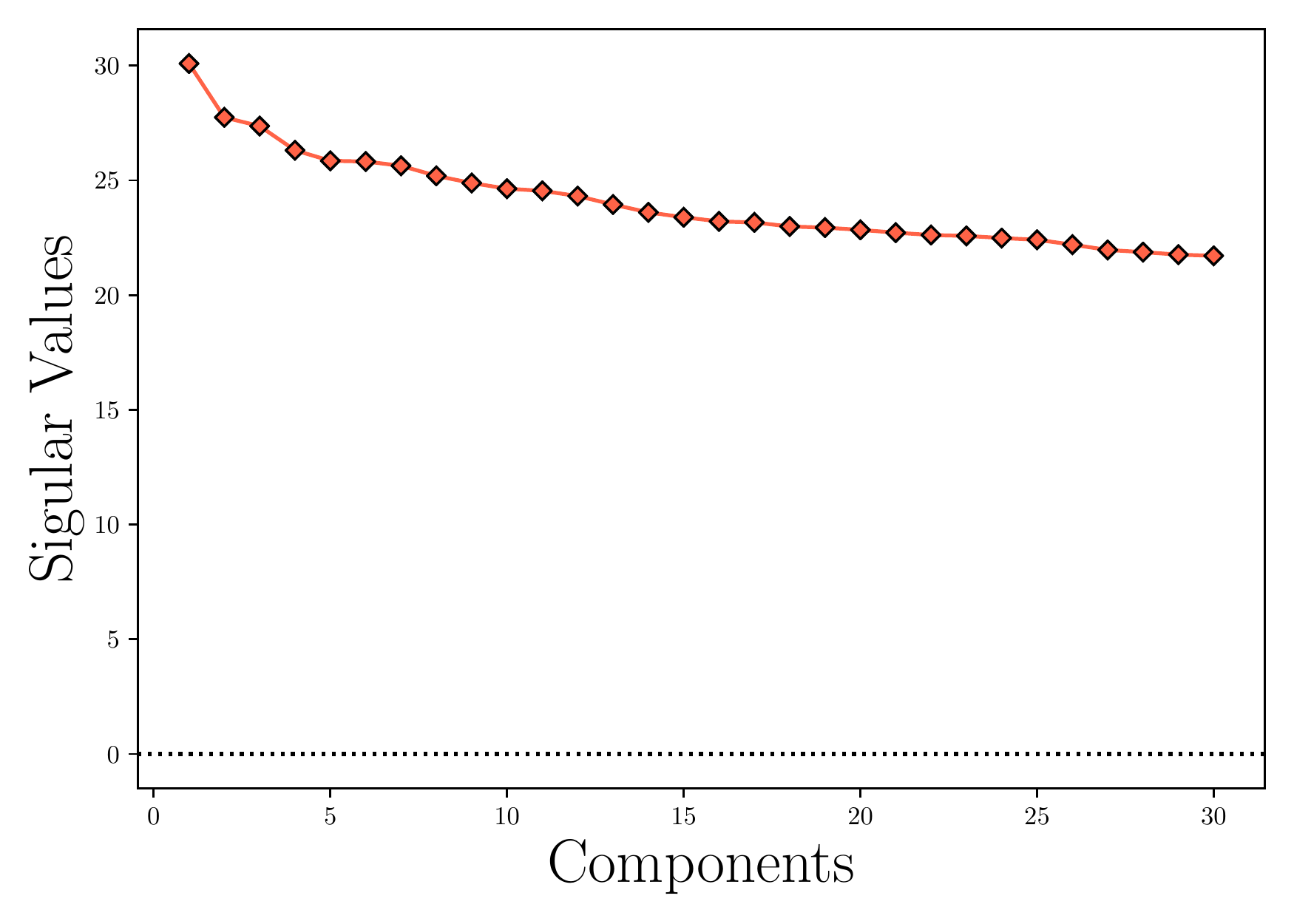}}
    \subfigure[\label{fig:pca-mcr-2}PCA: MCR$^2$ training learned features for overall data.]{\includegraphics[width=0.31\textwidth]{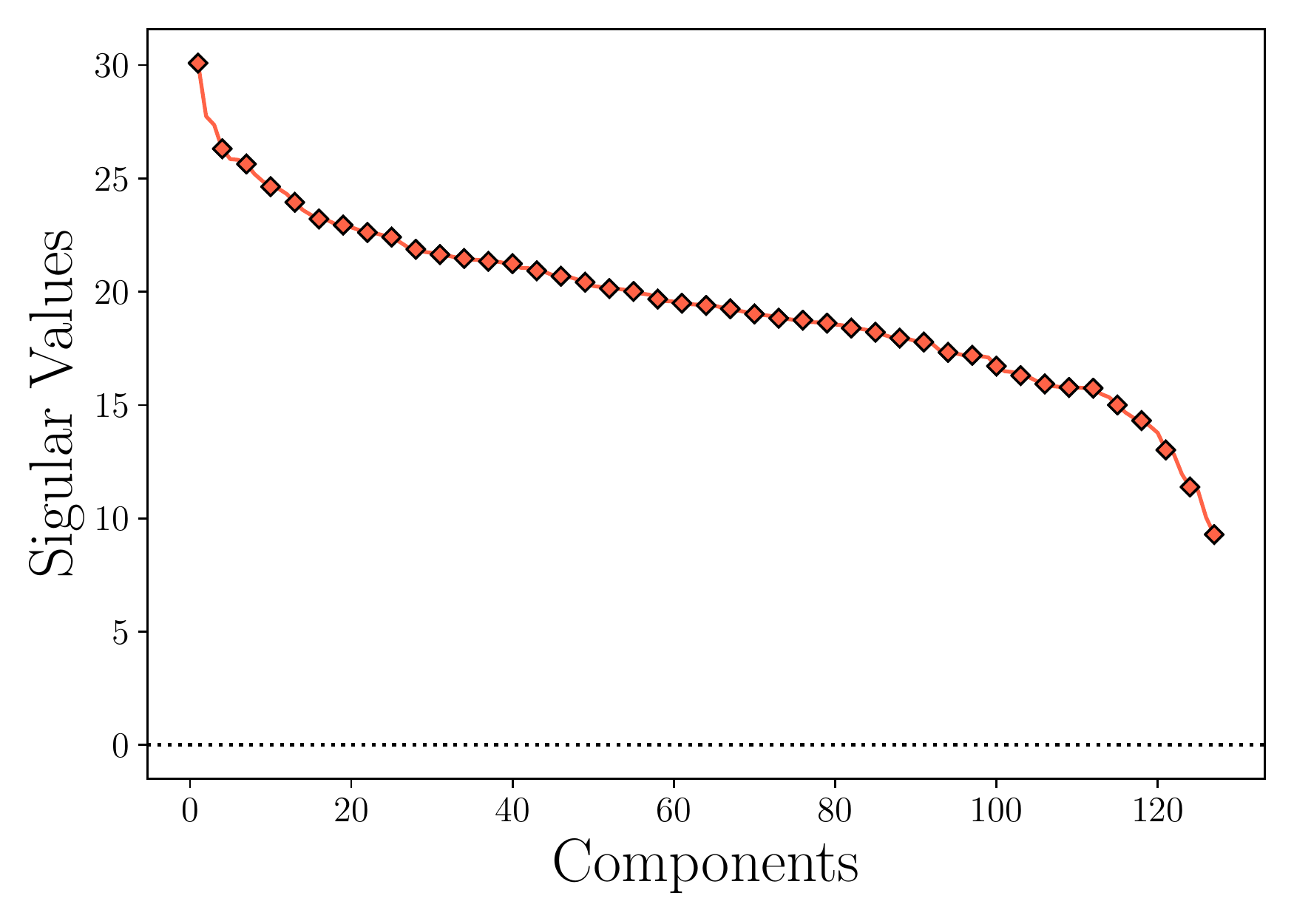}}
    \subfigure[\label{fig:pca-mcr-3}PCA: MCR$^2$ training learned features for every class.]{\includegraphics[width=0.31\textwidth]{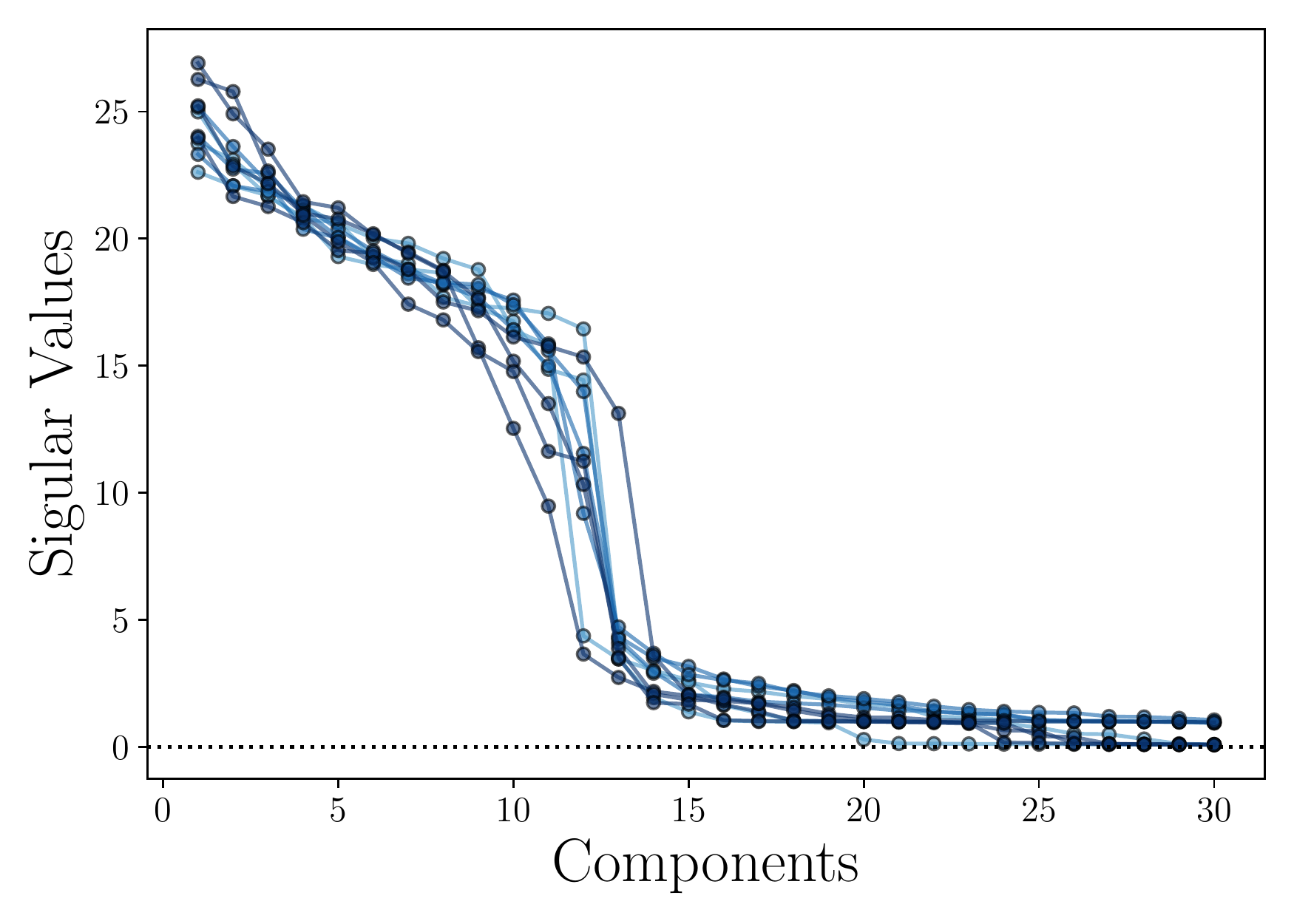}}
    \subfigure[\label{fig:pca-ce-1}PCA: cross-entropy training learned features for overall data (first 30 components).]{\includegraphics[width=0.31\textwidth]{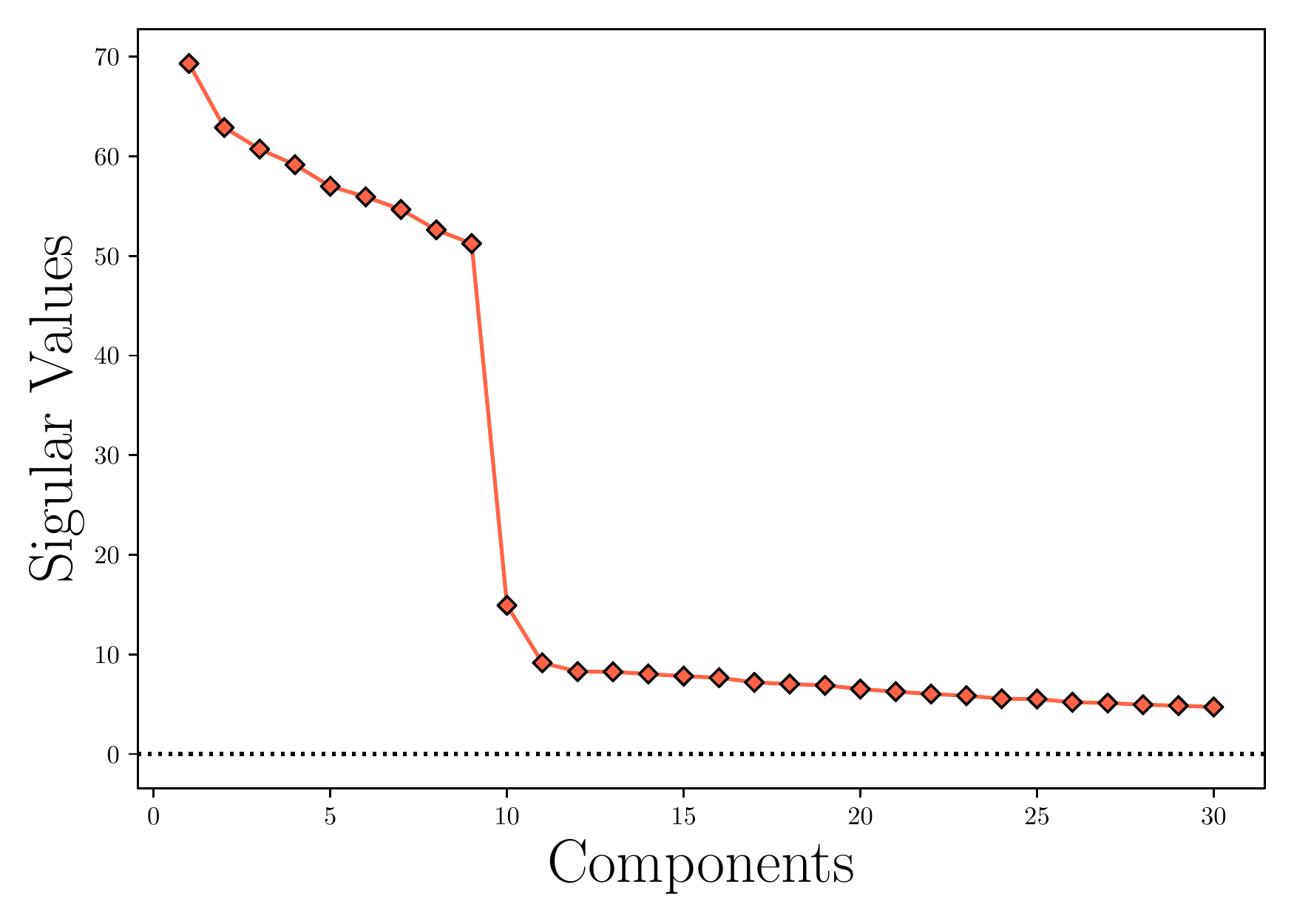}}
    \subfigure[\label{fig:pca-ce-2}PCA: cross-entropy training learned features for overall data.]{\includegraphics[width=0.31\textwidth]{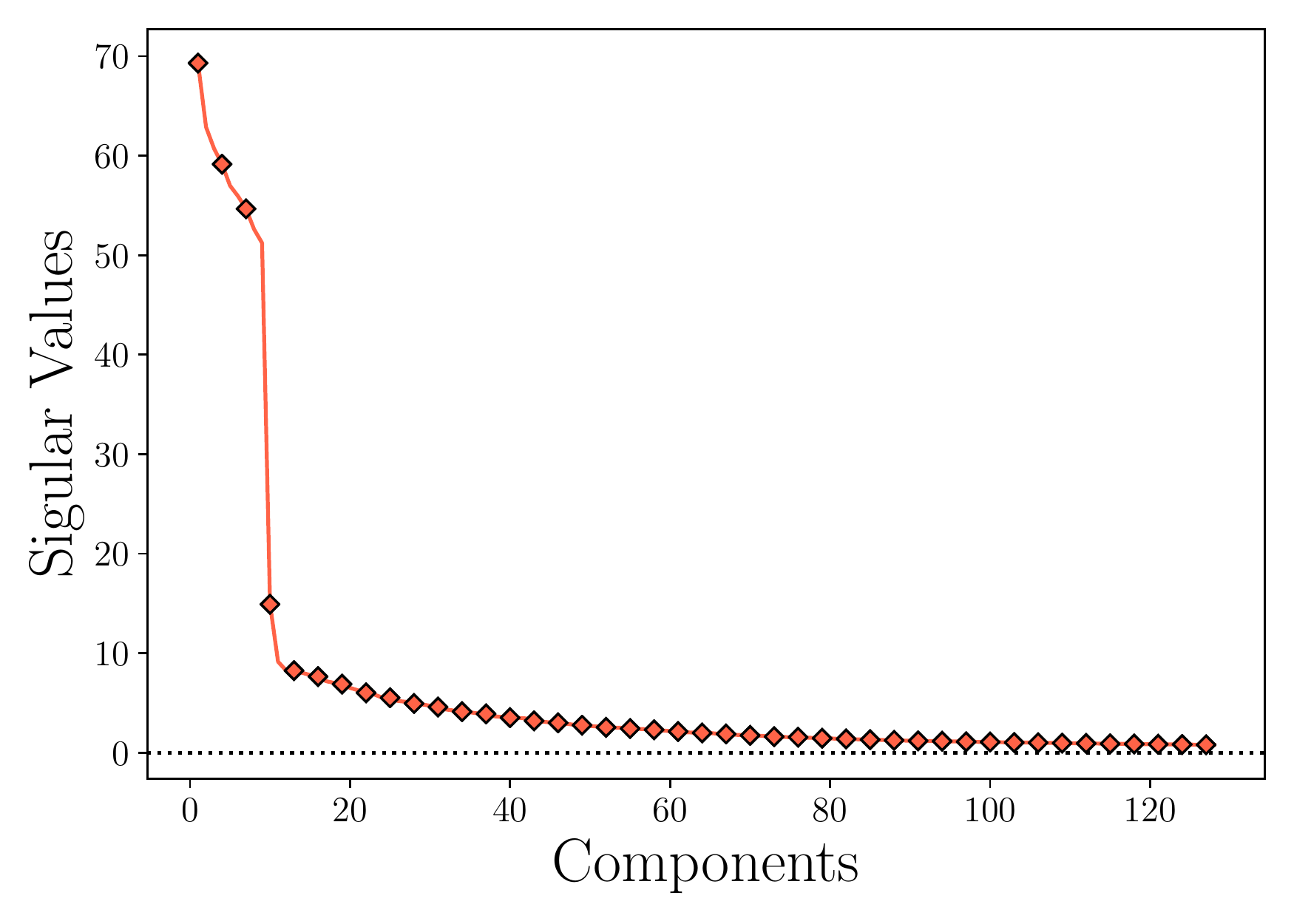}}
    \subfigure[\label{fig:pca-ce-3}PCA: cross-entropy training learned features for every class.]{\includegraphics[width=0.31\textwidth]{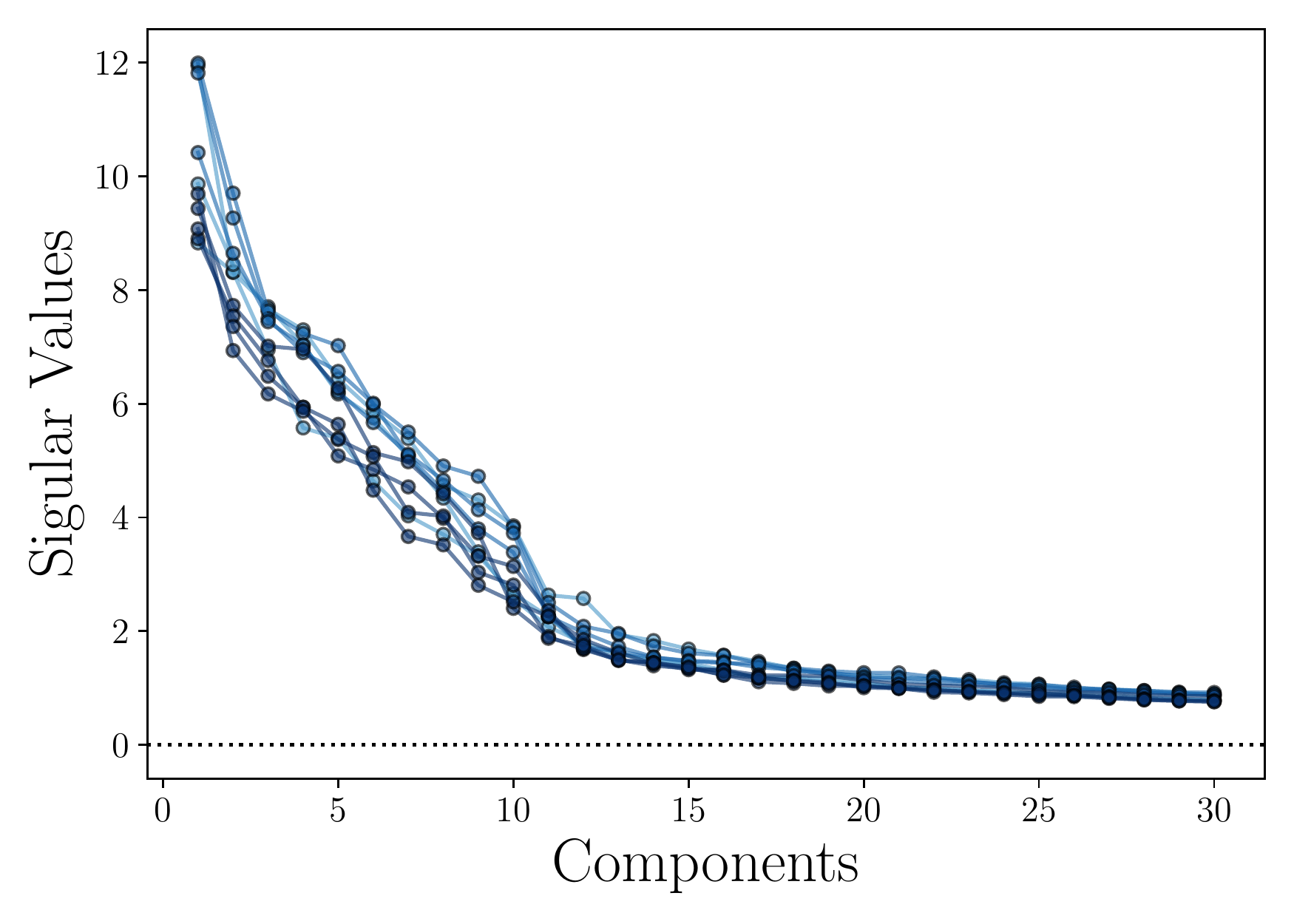}}
    \vskip -0.05in
    \caption{\small Principal component analysis (PCA) of learned representations for the  MCR$^2$ trained model (\textbf{first row}) and the cross-entropy trained model (\textbf{second row}).}
    \label{fig:pca-plot}
  \end{center}
%   \vskip -0.1in
\end{figure*}

\begin{figure*}[h]
  \begin{center}
    \subfigure{\includegraphics[width=0.47\textwidth]{Figures/heatmap.pdf}}
    \hspace{0.25cm}
    \subfigure{\includegraphics[width=0.47\textwidth]{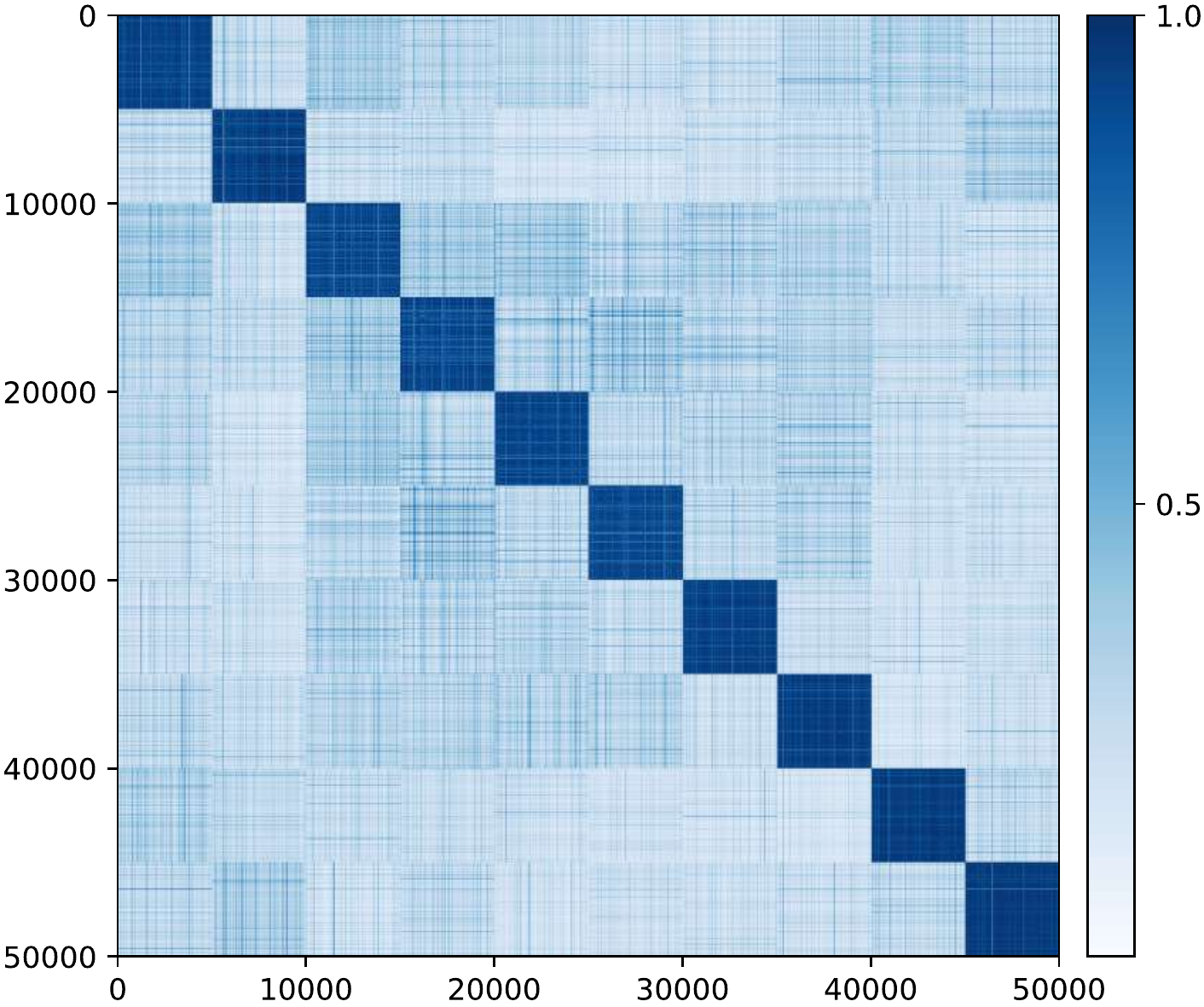}}
    \vskip -0.1in
    \caption{\small Cosine similarity between learned features by using the MCR$^2$ objective  (\textbf{left}) and CE loss (\textbf{right}).}
    \label{fig:heatmap-plot}
  \end{center}
  \vskip -0.1in
\end{figure*}

\begin{figure*}[h]
\subcapcentertrue
\begin{center}
    \subfigure[\label{fig:visual-bird}Bird]{\includegraphics[width=0.47\textwidth]{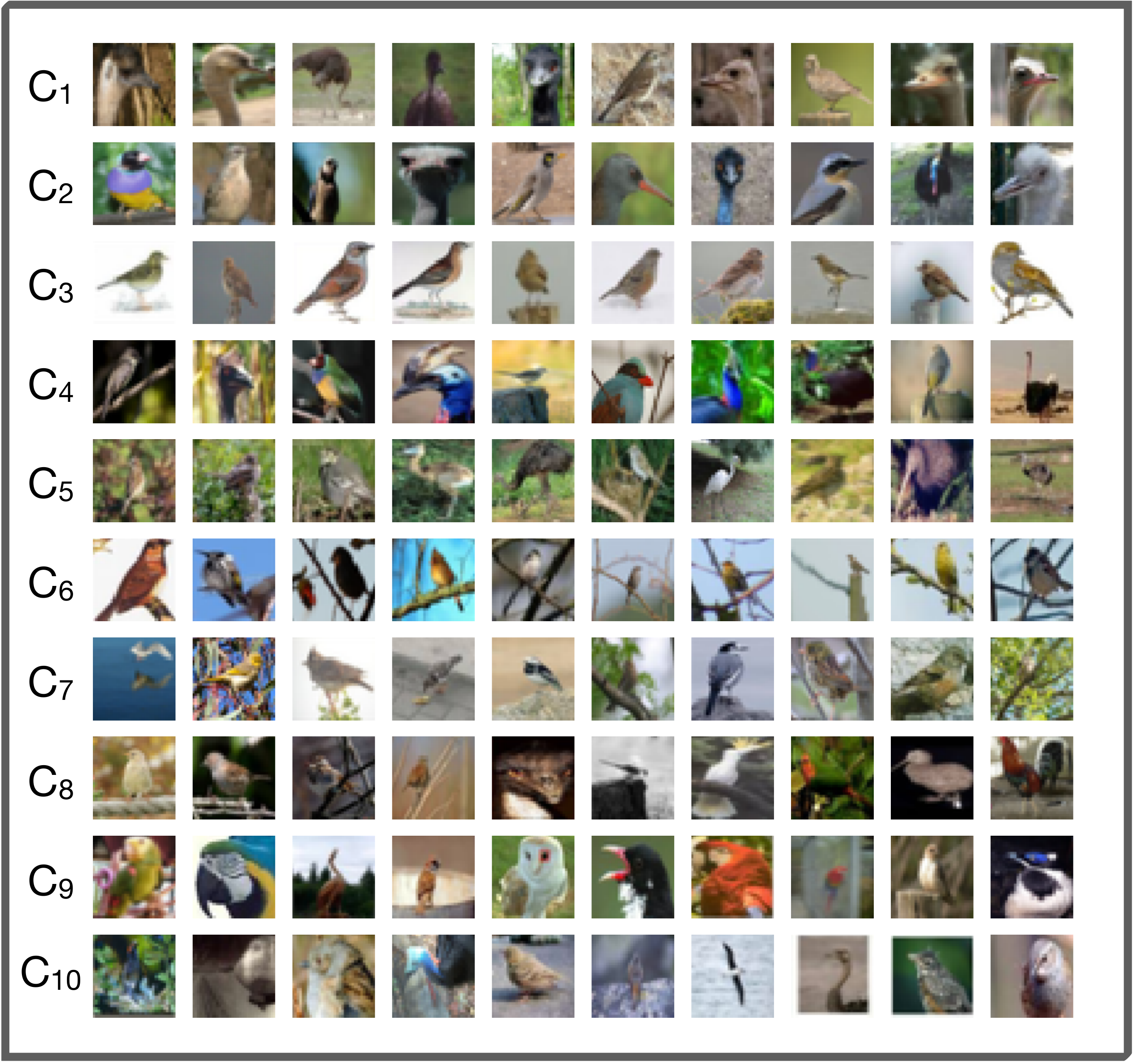}}
    % \hspace{0.05cm}
    \subfigure[\label{fig:visual-ship}Ship]{\includegraphics[width=0.47\textwidth]{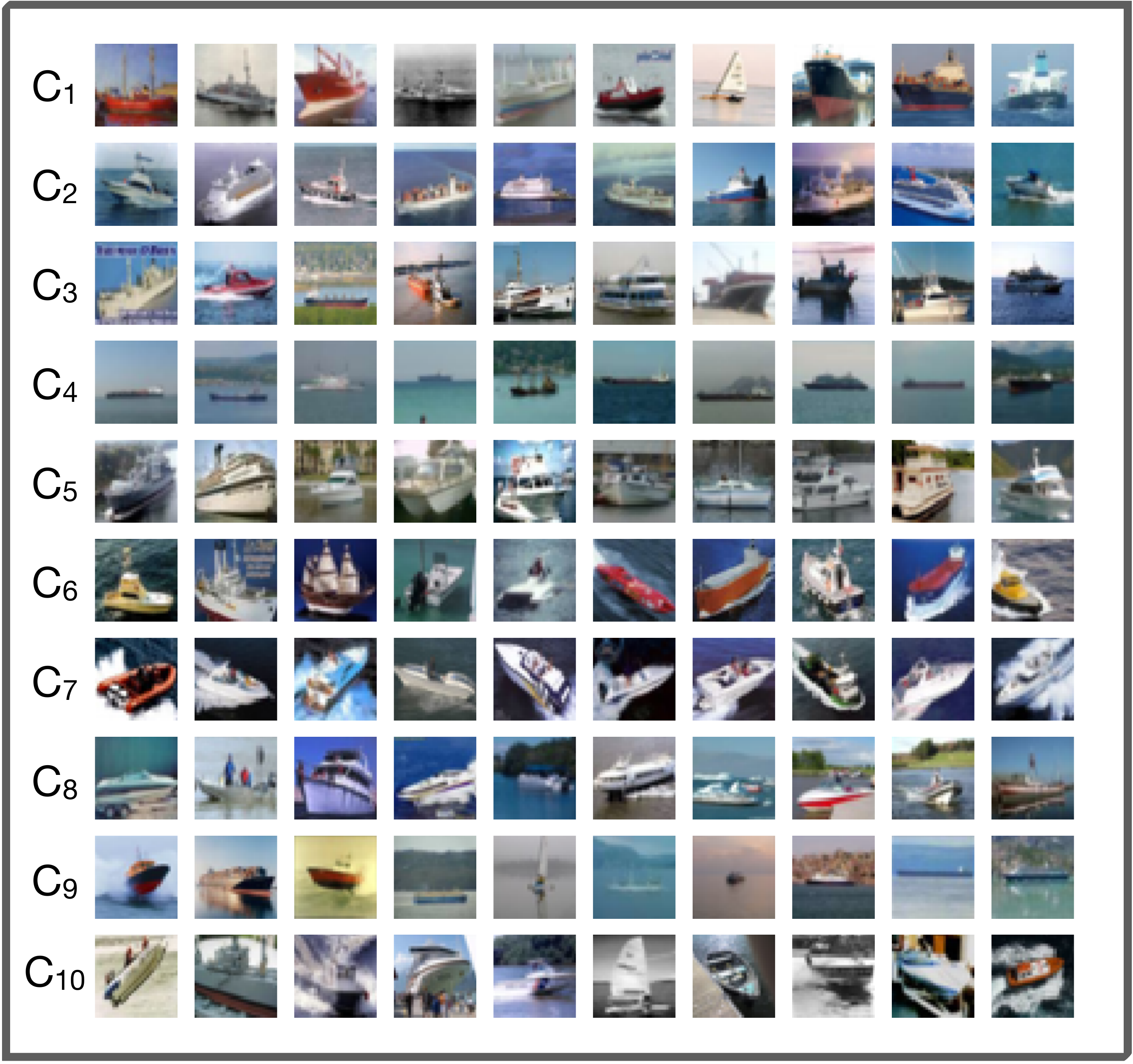}}
    \caption{\small Visualization of principal components learned for class 2-`Bird' and class 8-`Ship'. For each class $j$, we first compute the top-10 singular vectors of the SVD of the learned features $\Z_j$. Then for the $l$-th singular vector of class $j$, $\u_{j}^{l}$, and for the feature of the $i$-th image of class $j$, $\z_{j}^{i}$, we calculate the absolute value of inner product, $| \langle  \z_{j}^{i}, \u_{j}^{l} \rangle|$, then we select the top-10 images according to  $| \langle  \z_{j}^{i}, \u_{j}^{l} \rangle|$ for each singular vector. 
    In the above two figures, each row corresponds to one singular vector (component $C_l$). The rows are sorted based on the magnitude of the associated singular values, from large to small.}
\label{fig:visual-class-2-8}
\end{center}
\vskip -0.1in
\end{figure*}

\begin{figure*}[!h]
\subcapcentertrue
\begin{center}
    \subfigure[\label{fig:visual-overall-mcr}10  representative images from each class based on top-10 principal components of the SVD of learned representations by MCR$^2$.]{\includegraphics[width=0.46\textwidth]{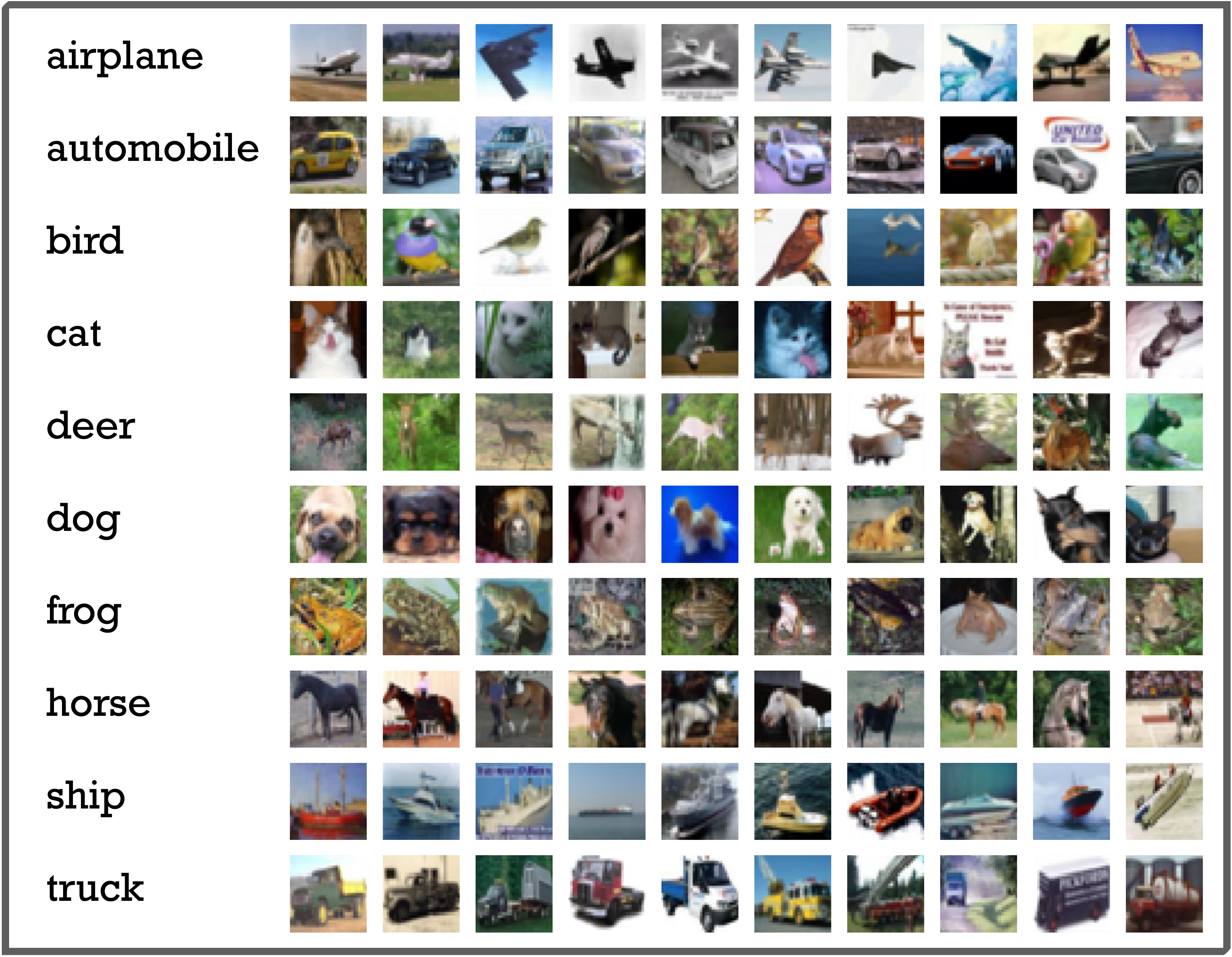}}
    % \hspace{0.05cm}
    \subfigure[\label{fig:visual-overall-cifar-website} Randomly selected 10 images from each class.]{\includegraphics[width=0.46\textwidth]{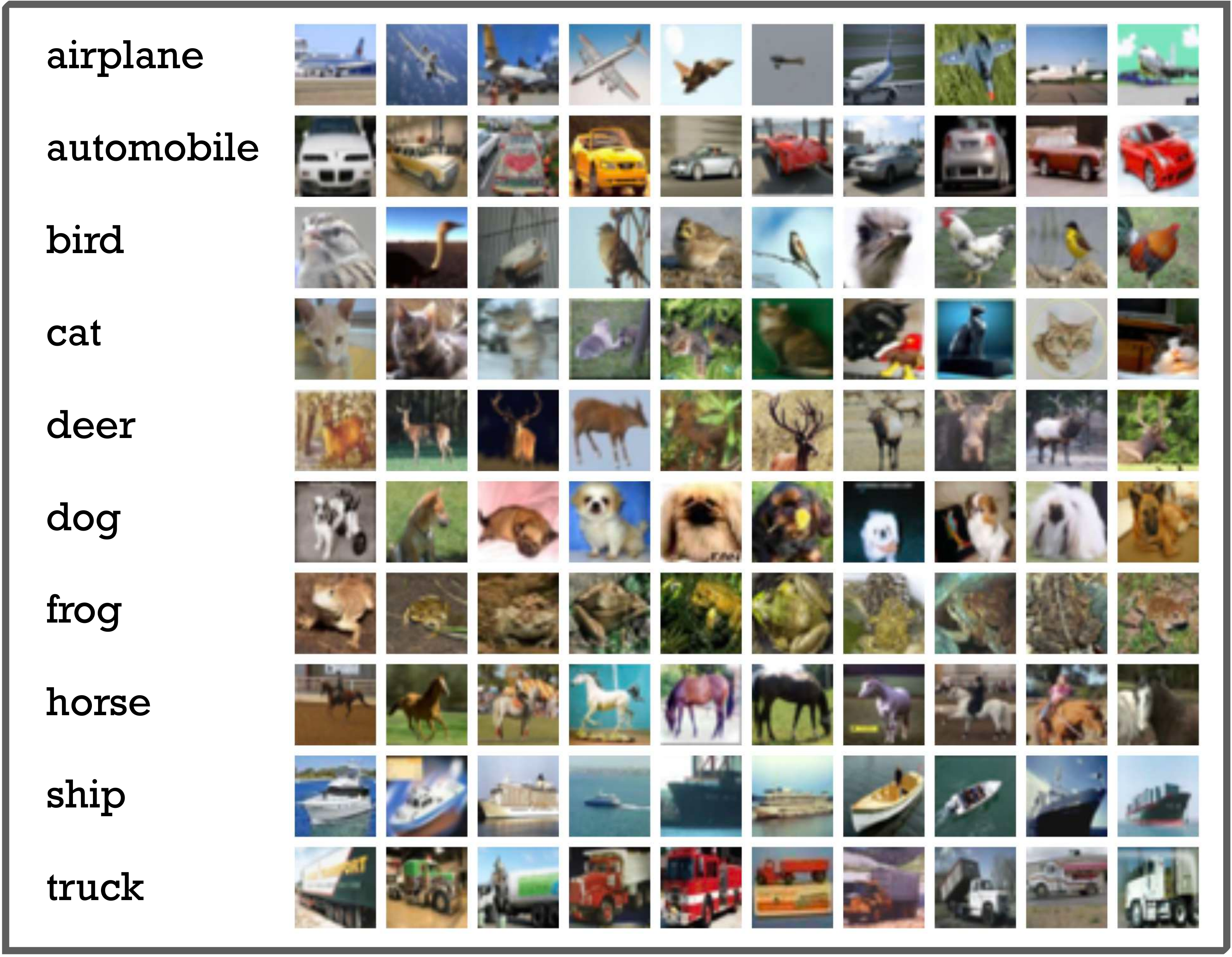}}
    \caption{\small Visualization of top-10 ``principal'' images for each class in the CIFAR10 dataset. \textbf{(a)} For each class-$j$, we first compute the top-10 singular vectors of the SVD of the learned features $\Z_j$. Then for the $l$-th singular vector of class $j$, $\u_{j}^{l}$, and for the feature of the $i$-th image of class $j$, $\z_{j}^{i}$, we calculate the absolute value of inner product, $| \langle  \z_{j}^{i}, \u_{j}^{l} \rangle|$, then we select the largest one for each singular vector within class $j$. Each row corresponds to one class, and each image corresponds to one singular vector, ordered by the value of the associated singular value. \textbf{(b)} For each class, 10 images are randomly selected in the dataset. These images are the ones displayed in the CIFAR dataset website~\cite{krizhevsky2009learning}.}
\label{fig:visual-overall-data}
\end{center}
\vskip -0.1in
\end{figure*}

For comparison, similar to Figure~\ref{fig:train-test-loss-pca-3}, we calculate the principle components of representations learned by MCR$^2$ training and cross-entropy training. For cross-entropy training, we take the output of the second last layer as the learned representation. The results are summarized in Figure~\ref{fig:pca-plot}. We also compare the cosine similarity between learned representations for both MCR$^2$ training and cross-entropy training, and the results are presented in  Figure~\ref{fig:heatmap-plot}. 

As shown in Figure~\ref{fig:pca-plot}, we observe that representations learned by MCR$^2$ are much more diverse, the dimension of learned features (each class) is  around a dozen, and the dimension of the overall features is nearly 120, and the output dimension is 128. In contrast, the dimension of the overall features learned using entropy is slightly greater than 10, which is much smaller than that learned by MCR$^2$. From Figure~\ref{fig:heatmap-plot}, for MCR$^2$ training,  we find that the  features of different class are almost orthogonal.

\paragraph{Visualize representative images selected from CIFAR10 dataset by using MCR$^2$.} As mentioned in Section \ref{sec:motivation}, obtaining the properties of desired representation in the proposed MCR$^2$ principle is equivalent to performing {\em nonlinear generalized principle components} on the given dataset. As shown in Figure~\ref{fig:pca-mcr-1}-\ref{fig:pca-mcr-3}, MCR$^2$ can indeed learn such diverse and discriminative  representations. In order to better interpret the representations learned by MCR$^2$, we select images according to their ``principal'' components (singular vectors using SVD) of the learned features.  In Figure~\ref{fig:visual-class-2-8}, we visualize images selected from class-`Bird' and class-`Ship'. For each class, we first compute top-10 singular vectors of the SVD of the learned features and then for each of the top singular vectors, we display in each row the top-10 images whose
corresponding features are closest to the singular vector. As shown in Figure~ \ref{fig:visual-class-2-8}, we observe that images in the same row share many common characteristics such as shapes, textures, patterns, and styles, whereas images in different rows are significantly different from each other -- suggesting our method captures all the different ``modes'' of the data even within the same class. Notice that top rows are associated with components with larger singular values, hence they are images that show up more frequently in the dataset.

In Figure~\ref{fig:visual-overall-mcr}, we visualize the 10 ``principal'' images selected from CIFAR10 for each of the 10 classes. That is, for each class, we display the 10 images whose corresponding features are most coherent with the top-10 singular vectors. We observe that the selected images are much more diverse and representative than those selected randomly from the dataset (displayed on the CIFAR official website), indicating such principal images can be used as a good ``summary'' of the dataset.

%\clearpage
\subsubsection{Experimental Results of MCR$^2$ in the Supervised Learning Setting.}\label{sec:appendix-subsec-sup}

\paragraph{Training details for mainline experiment.} For the model presented in Figure~\ref{fig:low-dim} (\textbf{Right}) and Figure~\ref{fig:train-test-loss-pca},  we use ResNet-18 to parameterize $f(\cdot, \theta)$, and we set the output dimension $d=128$, precision $\epsilon^2=0.5$, mini-batch size $m=1,000$. We use SGD in Pytorch~\cite{paszke2019pytorch} as the optimizer, and set the  learning rate \texttt{lr=0.01}, weight decay \texttt{wd=5e-4}, and \texttt{momentum=0.9}.

\paragraph{Experiments for studying the effect of hyperparameters and architectures.} We present the experimental results of MCR$^2$ training in the supervised setting by using various training hyperparameters and different network architectures. The results are summarized in Table~\ref{table:ablation-supervise}. Besides the ResNet architecture, we also consider VGG architecture~\cite{simonyan2014very} and ResNext achitecture~\cite{xie2017aggregated}. From Table~\ref{table:ablation-supervise}, we find that larger batch size $m$ can lead to  better performance. Also, models with higher output dimension $d$ require larger training batch size $m$.

\begin{table}[h]
\vskip -0.05in
\centering
\caption{\small Experiments of MCR$^2$ in the supervised setting on the CIFAR10 dataset.}
\label{table:ablation-supervise}
\vskip -0.05in
\begin{small}
\begin{sc}
\begin{tabular}{ lcccccl}
\toprule
Arch & Dim $d$ &  Precision $\epsilon^2$ & BatchSize $m$ & {\texttt{lr}} & ACC & Comment \\
\midrule
\multirow{1}{*}{ResNet-18} & 128  & 0.5 & 1,000 & 0.01 &  92.20\% & Mainline, Fig~\ref{fig:train-test-loss-pca} \\
\midrule
ResNext-29  & 128  & 0.5 & 1,000 & 0.01  & 92.55\% & \multirow{2}{6em}{Different Architecture}  \\
VGG-11      & 128  & 0.5 & 1,000 & 0.01 & 90.76\% &        \\
% ResNet-34   & 128  & 0.5 & 1000 & 0.01 & -\% &        \\
\midrule
ResNet-18   & 512  & 0.5 & 1,000 & 0.01 & 88.60\% & \multirow{3}{6em}{Effect of Output Dimension}  \\
ResNet-18   & 256  & 0.5 & 1,000 & 0.01 & 92.10\% &        \\
ResNet-18   & 64  & 0.5 & 1,000 & 0.01 & 92.21\% &        \\
\midrule
ResNet-18   & 128  & 1.0 & 1,000 & 0.01 & 93.06\% & \multirow{3}{6em}{Effect of precision}  \\
ResNet-18   & 128  & 0.4 & 1,000 & 0.01 & 91.93\% &        \\
ResNet-18   & 128  & 0.2 & 1,000 & 0.01 & 90.06\% &        \\
\midrule
ResNet-18   & 128  & 0.5 & 500 & 0.01 & 82.33\% & \multirow{5}{6em}{Effect of Batch Size}  \\
ResNet-18   & 128  & 0.5 & 2,000 & 0.01 & 93.02\% &        \\
ResNet-18   & 128  & 0.5 & 4,000 & 0.01 & 92.59\% &        \\
ResNet-18   & 512  & 0.5 & 2,000 & 0.01 & 92.47\% &        \\
ResNet-18   & 512  & 0.5 & 4,000 & 0.01 & 92.17\% &        \\
\midrule
ResNet-18   & 128  & 0.5 & 1,000 & 0.05 & 86.02\% & \multirow{3}{6em}{Effect of \texttt{lr}}  \\
ResNet-18   & 128  & 0.5 & 1,000 & 0.005 & 92.39\% &        \\
ResNet-18   & 128  & 0.5 & 1,000 & 0.001 & 92.23\% &        \\
\bottomrule
\end{tabular}
\end{sc}
\end{small}
\vskip -0.1in
\end{table}

\paragraph{Effect of  $r_j$ on classification.} Unless otherwise stated, we set the number of components $r_j=30$ for nearest subspace classification. We study the effect of $r_j$ when used for classification, and the results are summarized in Table~\ref{table:effect-rj}. We observe that the nearest subspace classification works for a wide range of $r_j$.

\begin{table}[h]
\vskip -0.05in
\begin{center}
\caption{\small Effect of number of components $r_j$ for nearest subspace classification in the supervised setting.}
\vskip -0.05in
\label{table:effect-rj}
\begin{small}
\begin{sc}
\begin{tabular}{l|ccccc}
\toprule
Number of components & $r_j=10$ & $r_j=20$ & $r_j=30$ & $r_j=40$ & $r_j=50$  \\
\midrule
Mainline (Label Noise Ratio=0.0) & 92.68\%  & 92.53\% & 92.20\% & 92.32\% & 92.17\%  \\
\midrule
Label Noise Ratio=0.1 & 91.71\%  & 91.73\% & 91.16\% & 91.83\% & 91.78\%  \\
Label Noise Ratio=0.2 & 90.68\%  & 90.61\% & 89.70\% & 90.62\% & 90.54\%  \\
Label Noise Ratio=0.3 & 88.24\%  & 87.97\% & 88.18\% & 88.15\% & 88.10\%  \\
Label Noise Ratio=0.4 & 86.49\%  & 86.67\% & 86.66\% & 86.71\% & 86.44\%  \\
Label Noise Ratio=0.5 & 83.90\%  & 84.18\% & 84.30\% & 84.18\% & 83.76\%  \\
\bottomrule
\end{tabular}
\end{sc}
\end{small}
\end{center}
\vskip -0.15in
\end{table}

\paragraph{Effect of  $\epsilon^2$ on learning from corrupted labels.} To further study the proposed MCR$^2$ on learning from corrupted labels, we use different precision parameters, $\epsilon^2 = 0.75, 1.0$, in addition to the one shown in Table~\ref{table:label-noise}. Except for the precision parameter $\epsilon^2$, all the other parameters are the same as the mainline experiment (the first row in Table~\ref{table:ablation-supervise}). The first row ($\epsilon^2=0.5$) in Table~\ref{table:label-noise-appendix-precision} is identical to the \textsc{MCR$^2$ training} in Table~\ref{table:clustering}. Notice that with slightly different choices in $
\epsilon^2$, one might even see slightly improved performance over the ones reported in the main body. 

\begin{table}[h]
\vspace{-2mm}
\begin{center}
\caption{\small Effect of Precision $\epsilon^2$ on classification results with features learned with labels corrupted at different levels by using MCR$^2$ training.}
\label{table:label-noise-appendix-precision}
% \vskip -0.07in
\begin{small}
\begin{sc}
\begin{tabular}{l | c c c c c }
\toprule
Precision  & Ratio=0.1 &  Ratio=0.2 &  Ratio=0.3 &  Ratio=0.4 &  Ratio=0.5 \\
\midrule
$\epsilon^2=0.5$ & {91.16\%} & {89.70\%} & {88.18\%} & {86.66\%} &  {84.30\%}\\
$\epsilon^2=0.75$ & \textbf{92.37\%} & {90.82\%} & \textbf{89.91\%} & \textbf{87.67\%} &  {83.69\%}\\
$\epsilon^2=1.0$ & {91.93\%} & \textbf{91.11\%} & {89.60\%} & {87.09\%} &  \textbf{84.53\%}\\
\bottomrule
\end{tabular}
\end{sc}
\end{small}
\end{center}
% \vspace{-6mm}
\end{table}

\subsubsection{Experimental Results of MCR$^2$ in the Self-supervised Learning Setting}\label{sec:appendix-subsec-selfsup}

\paragraph{Training details of MCR$^2$-{\scriptsize CTRL}.} For three datasets (CIFAR10, CIFAR100, and STL10), we use ResNet-18 as in the supervised setting, and we set the output dimension $d=128$, precision $\epsilon^2=0.5$, mini-batch size $k=20$, number of augmentations $n=50$, $\gamma_1 = \gamma_2 =20$. We observe that MCR$^2$-{\scriptsize CTRL} can achieve better clustering performance by using smaller $\gamma_2$, i.e., $\gamma_2=15$, on CIFAR10 and CIFAR100 datasets. We use SGD in Pytorch~\cite{paszke2019pytorch} as the optimizer, and set the  learning rate \texttt{lr=0.1}, weight decay \texttt{wd=5e-4}, and \texttt{momentum=0.9}.

\paragraph{Training dynamic comparison between MCR$^2$ and MCR$^2$-{\scriptsize CTRL}}.
In the self-supervised setting, we compare the training process for MCR$^2$ and MCR$^2$-\text{\scriptsize CTRL} in terms of $R, \widetilde{R}, R^c$, and $\Delta R$. For MCR$^2$ training, the features first expand (for both $R$ and $R^c$) then compress (for ). For MCR$^2$-\text{\scriptsize CTRL}, both $\widetilde{R}$ and $R^c$ first compress then $\widetilde{R}$ expands quickly and $R^c$ remains small, as we have seen in Figure \ref{fig:training-dynamic-controlling-compare} in the main body.

\paragraph{Clustering results comparison.} 
We compare the clustering performance between MCR$^2$ and MCR$^2$-{\scriptsize CTRL} in terms of NMI, ACC, and ARI. The clustering results are summarized in Table~\ref{table:mcr-mcrctrl-compare-appendix}. We find  that  MCR$^2$-{\scriptsize CTRL}  can achieve better performance for clustering.

\begin{table}[th]
\begin{center}
\caption{\small Clustering comparison between MCR$^2$ and MCR$^2$-{\scriptsize CTRL} on CIFAR10 dataset.}
\label{table:mcr-mcrctrl-compare-appendix}
\begin{small}
\begin{sc}
\begin{tabular}{l | c c c  }
\toprule
 & NMI & ACC & ARI  \\
\midrule
MCR$^2$ & 0.544 & 0.570 & 0.399\\
MCR$^2$-{\scriptsize Ctrl} & 0.630 & 0.684 & 0.508 \\
\bottomrule
\end{tabular}
\end{sc}
\end{small}
\end{center}
\end{table}

\subsubsection{Clustering Metrics and More  Results}\label{sec:appendix-subsec-clustering}
We first introduce the definitions of  normalized mutual information (NMI)~\cite{strehl2002cluster}, clustering accuracy (ACC), and adjusted rand index (ARI)~\cite{hubert1985comparing}.

\textbf{Normalized mutual information (NMI).} Suppose $Y$ is the ground truth partition and $C$ is the prediction partition. The NMI metric is defined as 
\begin{equation*}
    \text{NMI}(Y, C) = \frac{\sum_{i=1}^{k}\sum_{j=1}^{s}|Y_{i} \cap C_{j}|\log\left(\frac{m |Y_{i} \cap C_{j}| }{|Y_{i}| |C_{j}|}\right)}{\sqrt{\left(\sum_{i=1}^{k}|Y_i|\log\left(\frac{|Y_i|}{m}\right)\right) \left(\sum_{j=1}^{s}|C_j|\log\left(\frac{|C_j|}{m}\right)\right)}},
\end{equation*}
where $Y_i$ is the $i$-th cluster in $Y$ and $C_j$ is the $j$-th cluster in $C$, and $m$ is the total number of samples.

\textbf{Clustering accuracy (ACC).} Given $m$ samples, $\{(\x_i, \y_i)\}_{i=1}^m$.  For the $i$-th sample $\x_i$, let $\y_i$ be its ground truth label, and let $\bm{c}_i$ be its cluster label. The ACC metric is defined as 
\begin{equation*}
    \text{ACC}(\Y, \bm{C})= \max_{\sigma\in S}\frac{\sum_{i=1}^{m}\mathbf{1}\{\y_i = \sigma(\bm{c}_i)\}}{m},
\end{equation*}
where $S$ is the set includes all the one-to-one mappings from cluster to label, and $\Y = [\y_1, \dots, \y_m]$, $\bm{C} = [\bm{c}_1, \dots, \bm{c}_{m}]$.

\textbf{Adjusted rand index (ARI).} Suppose there are $m$ samples, and let $Y$ and $C$ be two clustering of these samples, where $Y = \{Y_1, \dots, Y_r\}$ and $C = \{C_1, \dots, C_{s}\}$. Let $m_{ij}$ denote the number of the intersection between $Y_i$ and $C_{j}$, i.e., $m_{ij} = |Y_i \cap C_j|$. The ARI metric is defined as 
\begin{equation*}
    \text{ARI} = \frac{\sum_{ij}\binom{m_{ij}}{2} - \left(\sum_{i}\binom{a_{i}}{2} \sum_{j}\binom{b_{j}}{2} \right)\big/ \binom{m}{2} }{\frac{1}{2}\left(\sum_{i}\binom{a_{i}}{2} +\sum_{j}\binom{b_{j}}{2} \right) - \left(\sum_{i}\binom{a_{i}}{2} \sum_{j}\binom{b_{j}}{2} \right)\big/ \binom{m}{2}},
\end{equation*}
where $a_{i} = \sum_{j}m_{ij}$ and $b_{j} = \sum_{i}m_{ij}$.

\paragraph{More experiments on the effect of hyperparameters of MCR$^2$-{\scriptsize CTRL}.}We provide more experimental results of MCR$^2$-{\scriptsize CTRL} training in the self-supervised setting by varying training hyperparameters on the STL10 dataset. The results are summarized in Table~\ref{table:ablation-self-supervise}. Notice that the choice of hyperparameters only has small effect on the performance with the MCR$^2$-{\scriptsize CTRL} objective. We may hypothesize that, in order to further improve the performance, one has to seek other, potentially better, control of optimization dynamics or strategies. We leave those for future investigation. 

\begin{table}[htp]
\centering
\caption{\small Experiments of MCR$^2$-{\scriptsize CTRL} in the self-supervised setting on STL10 dataset.}
\label{table:ablation-self-supervise}
\begin{small}
\begin{sc}
\begin{tabular}{ lccccc}
\toprule
Arch &   Precision $\epsilon^2$  & Learning Rate \texttt{lr} & NMI & ACC & ARI \\
\midrule
ResNet-18 &  0.5 & 0.1 & 0.446 & 0.491 & 0.290  \\
\midrule
ResNet-18  &  0.75 &  0.1 &   0.450 & 0.484 & 0.288   \\
ResNet-18  &  0.25 &  0.1 &   0.447 & 0.489 & 0.293   \\
ResNet-18  &  0.5 &  0.2  &   0.477 & 0.473 & 0.295   \\
ResNet-18  &  0.5 &  0.05  &   0.444 & 0.496 & 0.293  \\
ResNet-18  &  0.25 &  0.05 &   0.454 & 0.489 & 0.294  \\
\bottomrule
\end{tabular}
\end{sc}
\end{small}
\end{table}

\end{appendices}

\end{document}